%% file: lukasphd.tex
\documentclass[pdftex,final]{pittetd}
\usewithpatch{graphicx}
\usepackage{amsmath,amsthm}
\InputIfFileExists{amsmatch.pit}{%
            \ClassInfo{pittetd}{Patch for amsmatch loaded}}{%
            \ClassInfo{pittetd}{No patch found for amsmatch}}
\InputIfFileExists{amsthm.pit}{%
            \ClassInfo{pittetd}{Patch for amsthm loaded}}{%
            \ClassInfo{pittetd}{No patch found for amsthm}}
\usepackage{booktabs}
\usepackage{loe}
\usepackage{cite}
\usepackage{algorithm}
\usepackage{dsfont}
\usepackage{amsfonts}%
\usepackage{amssymb}%

\usepackage[T1]{fontenc}

\makeatletter 
\let\my@@font@warning\@font@warning 
\let\@font@warning\@font@info 
\makeatother 
\usepackage{fouriernc} 
\makeatletter 
\let\@font@warning\my@@font@warning 
\makeatother

\usepackage{breakcites}
\title[Michal Valko: PhD. Thesis]{Adaptive Graph-Based Algorithms for Conditional Anomaly Detection and Semi-Supervised Learning}
\author{\textbf{Michal Valko}}
\degree{MSc. Comenius University, Bratislava, 2005}

\date{August 1$^{\rm st}$ 2011}
\year{2011}                         
\keywords{Machine Learning, Anomaly Detection, Semi-Supervised Learning, Conditional Anomaly Detection, Online Learning, Graph-Based Learning, Adaptive Learning, Manifold Learning, Healthcare Informatics, Harmonic Solution, Backbone Graph, Random Walks}
\subject{PhD.  Thesis}
\school{Arts and Sciences}
\input{macros}

\begin{document}

\maketitle
%
\committeemember{\textbf{Milos Hauskrecht, PhD}, Associate Professor, Computer Science }
\committeemember{\textbf{G. Elisabeta Marai, PhD},  Assistant Professor, Computer Science }
\committeemember{\textbf{Diane Litman, PhD}, Professor, Computer Science }
\committeemember{\textbf{John Lafferty, PhD}, Professor, Machine Learning  (Carnegie Mellon University)}
\school{Computer Science Department}
\makecommittee
\copyrightpage       

\begin{abstract}
\input{abstract}

\end{abstract}

\tableofcontents
\listoftables                      
\listoffigures                     
\listofalgorithmsMICHEL
\listofequations

\preface
\input{ack}

%

\include{chap_intro}
\include{chap_bck_related}

\include{chapter_ssl}
\include{chapter_cad}

\include{chap_theory}

\include{chap_experiments}
\include{chap_future}

%
%

\bibliographystyle{apalike}
\bibliography{miki} 

\end{document}

%% file: macros.tex
\newcommand{\Lbd}{L^{\rm bd}}
\newcommand{\dmax}{{d_{\max}}}

\renewcommand{\qed}{\rule{5pt}{5pt}}
\newcommand{\commentout}[1]{}

\newtheorem{lemma}{Lemma}
\newtheorem{proposition}{Proposition}

\newtheorem{definition}{Definition}

\renewcommand{\qed}{\rule{5pt}{5pt}}

\newcommand{\bc}{{\bf c}}

\newcommand{\bh}{{\bf h}}

\newcommand{\bv}{{\bf v}}

\newcommand{\bw}{{\bf w}}

\newcommand{\bx}{{\bf x}}

\newcommand{\by}{{\bf y}}

\newcommand{\bell}{\ell}

\newcommand{\cH}{\mathcal{H}}
\newcommand{\cL}{\mathcal{L}}

\newcommand{\eps}{\varepsilon}

\newcommand{\realset}{\mathbb{R}}

\newcommand{\abs}[1]{\left|#1\right|}

\newcommand{\E}[2]{\mathbb{E}_{#1} \! \left[#2\right]}

\newcommand{\I}[1]{\mathds{1} \! \left\{#1\right\}}

\newcommand{\normw}[2]{\left\|#1\right\|_{#2}}

\newcommand{\set}[1]{\left\{#1\right\}}
\newcommand{\sgn}{\mathrm{sgn}}

\newcommand{\transpose}{^\mathsf{\scriptscriptstyle T}}

\newcommand{\hfs}{\ell^\ast}
\newcommand{\hfsoq}{\ell^{\rm q}}
\newcommand{\hfso}{\ell^{\rm o}}

\newcommand{\W}{W}
\newcommand{\Woq}{W^{\rm q}}
\newcommand{\compactWoq}{\tilde{W}^{\rm q}}

\newcommand{\Wo}{W^{\rm o}}
\newcommand{\Lo}{L^{\rm o}}
\newcommand{\Loq}{L^{\rm q}}
\newcommand{\Do}{D^{\rm o}}
\newcommand{\Doq}{D^{\rm q}}
\newcommand{\Ko}{K^{\rm o}}
\newcommand{\Koq}{K^{\rm q}}
\newcommand{\Zo}{Z^{\rm o}}
\newcommand{\Zoq}{Z^{\rm q}}
\newcommand{\co}{c^{\rm o}}
\newcommand{\coq}{c^{\rm q}}
\newcommand{\ello}{\ell^{\rm o}}
\newcommand{\elloq}{\ell^{\rm q}}
\newcommand{\Rmultiplier}{m}

\newcommand{\eat}[1]{}

\newcommand{\vecy}{\mathbf{y}}

\def\argmin{{\rm argmin}}

\def\E{{\rm E}}

\long\def\comment#1{}

\renewcommand{\O}{\ensuremath{\mathcal{O}}}

\newcommand{\M}{\mathcal{M}}

\newcommand{\pertK}{\ensuremath{\tilde{K}}}

\renewcommand{\ng}{k}
\renewcommand{\pertK}{\Koq}

\newcommand{\naive}{na\"{\i}ve{ }}

%% file: abstract.tex
\noindent
We develop graph-based methods for semi-supervised learning based on label propagation on a data similarity graph. When data is abundant or arrive in a stream, the problems of computation and data storage arise for any graph-based method. We propose a fast approximate online algorithm that solves for the harmonic solution on an approximate graph. We show, both empirically and theoretically, that good behavior can be achieved by collapsing nearby points into a set of local representative points that minimize distortion. Moreover, we regularize the harmonic solution to achieve better stability properties.   

We also present graph-based methods for detecting conditional anomalies and apply them to the identification of unusual clinical actions in hospitals.  Our hypothesis is that patient-management actions that are unusual with respect to the past patients may be due to errors and that it is worthwhile to raise an alert if such a condition is encountered. Conditional anomaly detection extends standard unconditional anomaly framework but also faces new problems known as fringe and isolated points. We devise novel nonparametric graph-based methods to tackle these problems. Our methods rely on graph connectivity analysis and soft harmonic solution. Finally, we conduct an extensive human evaluation study of our conditional anomaly methods by 15 experts in critical care.

%% file: ack.tex

I spent 6 splendid years in Pittsburgh. I want to thank the people who made my journey to PhD possible and enjoyable. First of all, I would like to sincerely thank my advisor Milos Hauskrecht, who made me excited about Machine Learning. With Milos, all my work had a purpose and I was happy  to work on important problems that people care about.  Therefore, the thought of quitting my PhD (so common among my peers) never crossed my mind. Milos stood by me literally from day one, when he came to pick me up at the airport. 

This thesis would not be about graph-based algorithms if not for Brano Kveton, who made me interested in them. He was my mentor during my internships at Intel Research. I am grateful for his friendship and encouragement that made me surpass the expectations I had for myself. Brano is an amazing researcher and he collaborated on many of the semi-supervised learning methods presented here. I would also like to thank my thesis committee, Liz Marai, Diane Litman  and in particular John Lafferty, whose research with Jerry Zhu on harmonic functions  has laid the foundation for this dissertation.

I was privileged to work with Greg Cooper, who should be a role model for any scientist. He has always impressed me with his professionalism by constantly having a big picture in mind. Greg was a part of a larger interdisciplinary group that I had a great opportunity to work in, ranging from biomedical informaticians to clinicians and pharmacists: Roger Day, who recruited me to the Bayesian club while he played tuba in 7 different bands;  Shyam Visweswaran; Amy Seybert; Gilles Clermont; Wendy Chapman; and many others.  I would like to thank our post-doc Hamed Valizadegan and Saeed Amizadeh, with whom I worked during my last year, and also other members of Milos' machine learning lab: Rich Pelikan, Iyad Batal, Shuguang Wang, Quang Nguyen, and Dave Krebs.

During my studies, I had the chance to do two internships with Intel Research, which taught me a lot about research in the industry. Besides Brano, I would like to thank my research collaborators Ling Huang, Ali Rahimi, Daniel Ting from Berkeley, my colleges Georgios Theocharous, Kent Lyons, Jennifer Healey, John Mark Agosta, Nick Carter, and many other researches and interns.

I also want to thank my teachers from the Machine Learning Department at Carnegie Mellon University, where I learned so much about the field, namely John Lafferty, Larry Wasserman, Carlos Guestrin, Geoff Gordon, Eric Xing, and Gary Miller. I am also grateful to my teachers from Pitt and from Comenius University in Bratislava. But my love for math started with my high school teacher, Martin Macko, who taught me mathematics for 8 years in the best way I can imagine. It is because of him that I decided to do science and research. My mom had a role in that decision, too, and it is probably not a coincidence that she went to grad school for cybernetics. 

I am grateful to our administrative staff, who helped to smooth my life at Pitt: Kathy O'Connor, Kathleen Allport, Wendy Bergstein, Loretta Shabatura, Nancy Kreuzer, and Keena Walker. These ladies are running this place and create a feeling like home. Also, I want to thank Phillipa Carter from the Graduate Studies who will remember me as somebody that would stalk her everywhere and would always want to skip the chain of command. 

Besides my studies, my life in Pittsburgh would not be so great if it was not my friends who made my stay here so enjoyable.  I need to mention Tomas Singliar and Zuzana Jurigova, Christopher Hughes, Mihai and Diana Rotaru, Roxana and Vlad Gheorgiou for marvelous times we spent together. I made great friends at Pitt, with whom I would like to keep in touch including Panickos Nephytou, Kiyeon Lee, Ryan Moore, Alexandre Ferreira, Yaw Gyamfi, Samah Mazraani, Weijia Li, Peter Djalaliev, Lory al Moakar, Shenoda Guirguis, Tatiana Ilina,  and also many friends at CMU such as Miro Dudik, Mike Gamalinda, Polo Chau, Kiki Levanti, Michael Papamichael, Jose Gonzales-Brenes, Lucia Castellanos, Mladen Kolar, Stano Funiak, and others.

My time in Pittsburgh was fabulous also due to Pitt Men's Glee Club, where I sang Tenor 2  at many concerts, including the ones on our tours to Italy, Croatia, and Texas. I need to start with Richard Earl Teaster, the director of the group,  who was one of the first people I met upon my arrival to Pittsburgh and over 6 years became one of my best friends. Besides that he has been my private voice teacher for many of those years and raised me from somebody with almost no experience in singing to a singer in this exquisite ensemble. The glee club became a huge part of me, and my life often consisted of research, friends and glee club. Moreover, within the group I made great friends, such as  Mike Pollock, Matt (and Kelly) Keeny, Dexter Gulick, Chad Slyman, Geoff Arnold, Paul Schillenger, Evan Pavloff, Joel Arter, Adam Sloan, Greg Hill, Phil Slane, Tyler Kirland, Ryan McGinnis, Ben Dichter, Josh Niznik, Matt Recker, Matt Cahalan, James Montgomery, Nick Czarnek, Erick Markley, Jared Wilson, Victor Bench, John Moriarty, Mark Ellenberger, Charlie Eichman, Evan McCullough, Evan Williams,  and dozens of others with whom I shared the world of music, including my music teachers: Don Fellows, John Goldsmith, Claudia Pinza, Jim Ferla, Sister Maria Ozah, and Ben Harris.

For the most of my time at Pitt, I lived in a big house a few minutes away from my department. Since we often housed visiting students and scholars, during the five years I probably had more than 50 housemates.  I will always remember the parental caring of Kevin and Elizabeth Leichbach, dancing robotic toys of David Palmer, delicious food of Azzurra Missinato and Kikumi Ozaki, Sera Wang, Sonia Wu, Michael Wasilewski, and many others who lived at 307 Halket Street. Through Kevin and Elizabeth to have I met many nice people, such as  Kirsten Ream, Kate Shaughnessy, Kelly Reed, George and Becky Mazagieros, and Bob Fratto. I made many friends at the Pittsburgh branch of Campus Crusade for Christ. The ones that stand out are Matt Budavich, TJ Laird, Zach Fogel, Joe Rathbun, Justin Cooper, and Evelyn Yarzebinski, who made sure that this dissertation is actually written in English.

I would also like to thank my friends in Slovakia, who were often welcoming me at home during my summer and winter breaks and spent time with me on hikes, travels and music festivals. Notably, I would to thank Michal Rjasko for the long years of great friendship.  
Finally, I am indebted to my family, in particular to my mom Maria, my dad Michal, and my sister Zuzka, who provided me with everything I needed and supported me throughout all my studies. I also thank my extended family, especially my only living grandparent, Anna, who has been praying for me every single day.  Her prayers were definitely heard as God gave me the strength and grace I needed.

%% file: chap_intro.tex
\chapter{Introduction}

\section{Motivation}
If we want people to enjoy the benefits of machine learning,  we should provide them with algorithms that do not require
much training time before they can be useful.  Therefore, we will investigate the algorithms that need only minimal feedback from the users. For example, in semi-supervised learning we assume that only very few examples from the data are labeled
and we try to use the unlabeled examples to learn something about the structure of the data. 
In the area of conditional anomaly detection and in particular in medicine, a traditional
approach is to ask experts to create a set of rules that would raise an alert if an adverse event is encountered. 
Since a manual creation of rules is very time consuming, we would rather like to 
learn what the adverse event might be from the collection of the past data.

In this dissertation, we will take advantage of using a similarity graph as the data representation.
Similarity graphs help us model the relationship between the examples.   
However, graph-based algorithms, such as label propagation, do not scale well beyond several thousand examples.  
We will address this problem by data quantization, where unlike other approaches 
($k$-means, subsampling) we consider the quality of the inference. 
Moreover, we investigate an online learning formulation of semi-supervised learning, which is suitable for adaptive machine
learning systems when the data arrive in a stream. 

Furthermore, we extend graph-based learning to conditional anomaly detection problem and apply it to clinical scenarios.
Traditionally, anomaly detection techniques identify unusual patterns in data. In clinical settings, these may concern identification of unusual patients, unusual patient--state outcomes, or unusual patient-management decisions.  

Our ability to detect unusual events in clinical data may have a tremendous impact on  health care and its quality.  First, the identification of an action that differs from an expected or usual  pattern of care can aid in detection and prevention of the potential medical errors. According to the HealthGrades study (Wall Street Journal on July 27, 2004), medical errors account for 200,000 preventable deaths a year. Second, the identification of anomalous patient responses can help us to identify new promising treatments.  

\begin{figure}
  \centering
  \includegraphics[width=0.7\columnwidth, viewport=50 183 700 500]{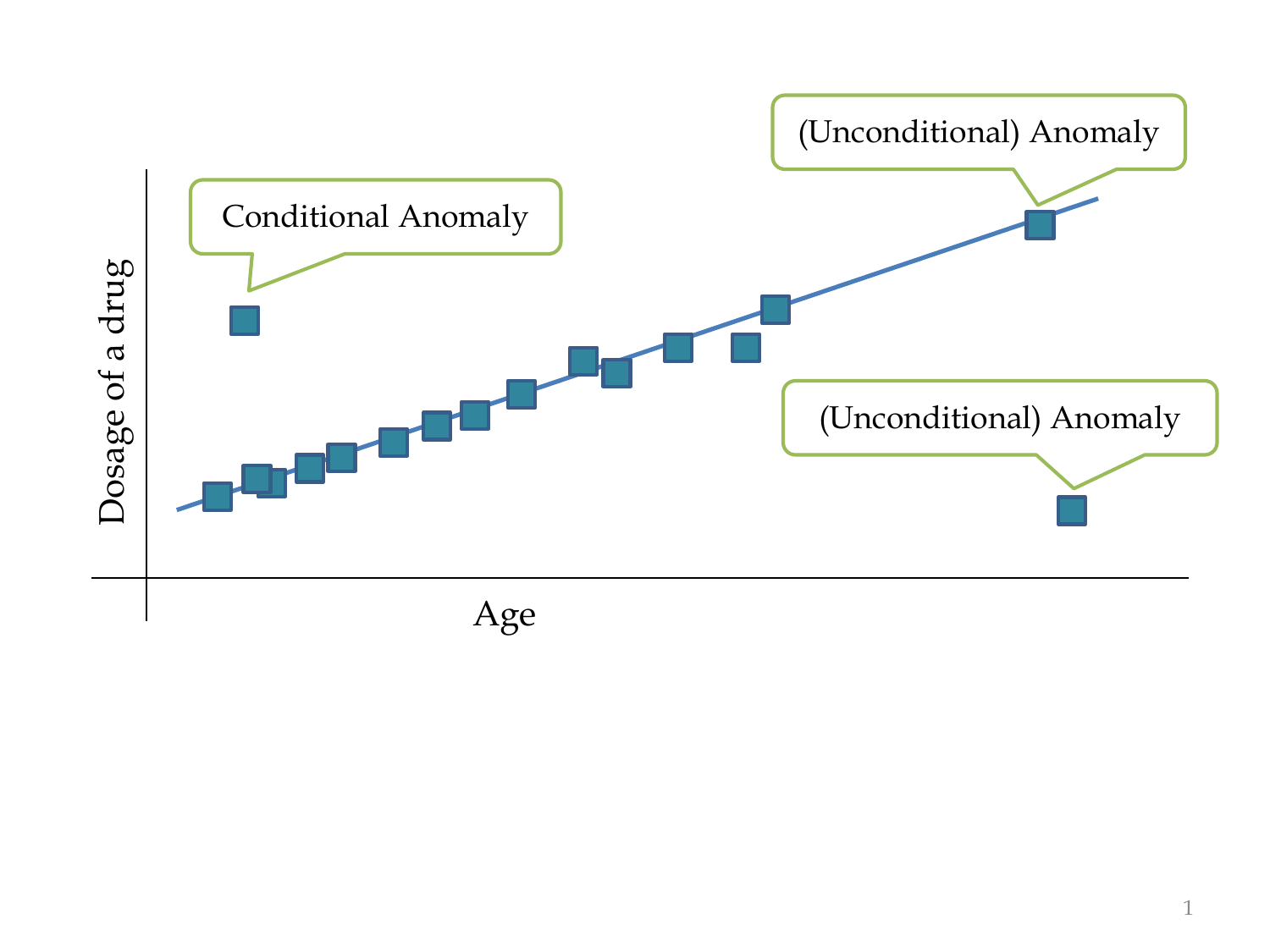}  
  \caption{Conditional vs. unconditional anomalies}
	\label{fig:conditional_unconditional}
\end{figure}

Typical anomaly detection methods used in data analysis are unconditional (with respect to the context) and look for outliers with respect to all data attributes. In the medical domain these methods would identify unusual patients, that is, patients suffering from a less frequent disease or patients with unusual collection of symptoms.  Unfortunately, this does not fit the nature of the problem we want to solve in error detection: the identification of unusual patient management decisions with respect to past patients who suffer from the same or similar condition.  To address this, we are developing a \emph{qualitatively new conditional anomaly detection framework} where the decision event is judged anomalous with respect to the patient's symptoms, state, and demographics.  

The conditional anomaly detection is the problem of  detecting unusual values for a subset of variables given the values of the remaining variables.
Figure \ref{fig:conditional_unconditional} illustrates the concept of conditional anomaly:
Assume that the dosage of a drug is a linear function of the age. 
Now imagine that we have a young patient that was given a higher dosage of a drug (Figure~\ref{fig:conditional_unconditional}, top left).
The amount of dosage is not unusual at all. Indeed, we have other patients with the same or similar dosage.
What is unusual is the dosage with respect to his age;
the patients that have similar ages were given lower dosages. 
We can say that this dosage was \emph{conditionally anomalous} given a patient's age.


\begin{figure}
		\begin{center}
		\includegraphics[width=0.8\columnwidth, clip, viewport=30 300 622 536]{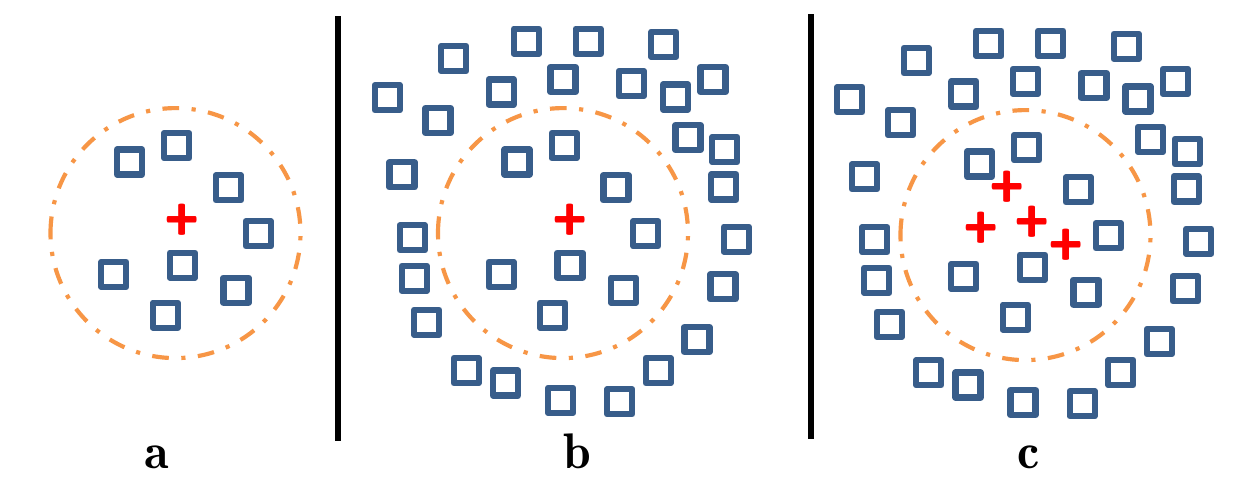}
		\caption{Disadvantages of nearest neighbor approach for conditional anomaly detection}
				\label{fig:wknn_had}
				\end{center}
				\end{figure}

Throughout this dissertation, we build on label propagation on a data similarity graph, 
which exploits the manifold assumption \cite{chapelle2006semi-supervised}. 
Unlike local neighborhood methods based on the nearest neighbors, it respects the structure of the manifold and lets us account for more complex interactions in the data. In other words, while the metric may provide a reasonable local similarity measure, it is frequently inadequate as a measure of global similarity \cite{szummer2001partially}. 
Figure~\ref{fig:wknn_had} illustrates a potential benefit of label propagation, where the goal 
is to detect that the positive (+) example has an anomalous label conditioned on its placement. 
The positive (+) label in \ref{fig:wknn_had}\textbf{b} is more anomalous than the one in \ref{fig:wknn_had}\textbf{a}, but 
nearest neighbor (NN) would consider them equal, because in only considers the points within the displayed circle. 
Moreover, the NN approach would find clustered (+) anomalies in \ref{fig:wknn_had}\textbf{c} normal because it ignores the data beyond the nearest neighbors.

\section{Thesis Statement and Main Contributions}

Although very popular, label propagation on a data similarity graph does not scale well beyond several thousand of examples,
due to the following reasons:

\begin{enumerate}

\item The computation of the similarity matrix and the label propagation are $\Omega(n^2)$ where $n$ is the number of examples. 
Label propagation itself requires the computation of the $n\times n$ matrix inverse or the solution of the system of $n$ linear equations.
\item Current methods that reduce the size of the graph to form an approximate back-bone graph do not link
the construction of this graph to the final inference task.
\item  Despite the usefulness of the online semi-supervised learning paradigm for practical
adaptive algorithms, there is not much success in applying this paradigm to realistic problems,
especially when data arrive at a high rate.
\end{enumerate}

\noindent Next, the problem of conditional anomaly detection could be approached  by 

\begin{enumerate}
	\item extending one-class (unconditional) anomaly methods (Section~\ref{sec:ClassOutlierApproach})
	\item classification and claiming misclassified examples as conditionally  anomalous (Section~\ref{sec:DiscriminativeApproach})
\end{enumerate}

\noindent
Both of these approaches suffer from the problems of isolated and fringe points described in Section~\ref{sec:ConditonalAnomalyDetection}.
%
In this dissertation  we develop the methodology to address these problems.
We take a graph-based approach, because it is non-parametric, incorporates the manifold assumption, and
can also easily take advantage of unlabeled data. We present the following main contributions:

\begin{itemize}
\item We show how to combine max-margin and semi-supervised learning to \textbf{max-margin graph cuts} semi-supervised learning (Section~		\ref{sec:MMGC}).
	\item We show how to compute label propagation on a graph and the centroids of a backbone graph \textbf{jointly}. (Section~\ref{sec:LargeScaleSemiSupervisedLearningWith})
	\item We propose the \textbf{online harmonic function} solution and show how to compute its approximation efficiently (Section~\ref{sec:OnlineLearningWithQuantizedGraphs}).	
	\item We prove \textbf{performance bounds} for our online algorithm in a semi-supervised setting on quantized graphs (Section~\ref{sec:AnalysisOfOnlineSSLOnQuantizedGraphs}).
	  \item We introduce  non-parametric \textbf{graph-based} methods and show how they can \emph{handle unconditional outliers} (Section~\ref{sec:Regularization}).
	\item We show how a \textbf{soft harmonic solution} on data similarity graphs can be used for conditional anomaly detection (Section~\ref{sec:ConditionalAnomalyDetectionWithSoftHarmonicFunctions}).

\end{itemize}

\noindent	In addition, we test the conditional anomaly detection methods by comparing them to the evaluations conducted with 
a panel of physicians and show the benefits of our methods (Section~\ref{sec:PilotStudyIn2009}).
Based on the aforementioned contributions, we claim the following:

\begin{center}
\textbf{\large \sffamily	Our graph-based methods can perform online semi-supervised
learning with a constant per-step update and provable performance guarantees. 
Moreover, they can detect conditional anomalies and filter unconditional anomalies.}
\end{center}

\section{Organization of the Dissertation}

\begin{itemize}
	\item In Chapter \ref{sec:BackgroundAndRelatedWork}, we outline the related work in 
anomaly detection (Section \ref{sec:RelatedWorkInAnomalyDetection}),
semi-supervised learning (Section \ref{sec:RelSemiSupervisedLearning}),
and graph quantization (Section \ref{sec:RelatedWorkInOnlineQuantization}).
\item Chapter~\ref{sec:SemiSupervisedLearning} presents new approaches for semi-supervised learning  and the online semi-supervised learning
	(Section~\ref{sec:OnlineLearningWithQuantizedGraphs}).
  \item Chapter \ref{sec:ConditonalAnomalyDetection} presents novel methods for 
	conditional anomaly detection. 
	\item Chapter~\ref{sec:TheoreticalAnalysis} presents the theoretical analysis of the methods from Chapter~\ref{sec:SemiSupervisedLearning}
	and Chapter~\ref{sec:ConditonalAnomalyDetection}.	
	In particular, it presents the analysis of max-margin graph cuts (Section~\ref{sec:TheoryMaxMarginGraphCuts}) and 
the analysis of the online semi-supervised learning on quantized graphs (Section~\ref{sec:AnalysisOfOnlineSSLOnQuantizedGraphs}).
\item Chapter \ref{sec:Experiments} presents the experimental results 
on various synthetic and real-world datasets,
notably the face recognition video datasets and the medical datasets from University of Pittsburgh Medical Center.
\end{itemize}

\noindent Parts of this dissertation have previously appeared in 
\cite{hauskrecht2007evidence-based,hauskrecht2010conditional,valko2008conditional,valko2008distance,valko2010feature,valko2010online,valko2011conditional,kveton2010semi-supervised,kveton2010online}.

\bigskip

%% file: chap_bck_related.tex
\chapter{Related Work}
\label{sec:BackgroundAndRelatedWork}

In this chapter we review the relevant work on graph quantization, semi-supervised learning, and anomaly detection.

\input{sec_related_quantization.tex}
\input{sec_related_ssl.tex}

\input{sec_related_anomaly.tex}

%% file: sec_related_quantization.tex
\section{Related work in Graph Quantization}
\label{sec:RelatedWorkInOnlineQuantization}

Given $n$ data points and a typical graph construction method, 
the exact computation of the harmonic solution has 
space and time complexity of $\Omega(n^2)$ in general 
due to the construction of an $n \times n$ similarity matrix.
Furthermore, exact computation requires 
an inverse operation on an $n \times n$ similarity matrix
which takes $O(n^3)$ in most practical implementations\footnote{The complexity can be further improved to $\O(n_u^{2.376})$ by using 
the Coppersmith-Winograd algorithm.}.
For applications with large data size (e.g., exceeding thousands), 
the exact computation or even storage of the harmonic solution becomes infeasible, 
and problems with $n$ in the millions are entirely out of reach.

An influential line of work in the related area of graph partitioning 
approaches the computation problem by reducing the size of the graph, collapsing 
vertices and edges, partitioning the smaller graph, and then uncoarsening 
to construct a partition for the original graph~\cite{hendrickson1995multilevel,karypis1999fast}. 
Our work is similar in spirit but provides a theoretical analysis for a particular kind of coarsening and
uncoarsening methodology.

Our aim is to find an effective \emph{data preprocessing} technique 
that reduces the size of the data and coarsens the graph
\cite{madigan2002likelihood-based,mitra2002density-based}.
\noindent 
There are two types of approaches widely used 
in practice for data preprocessing: 

\begin{enumerate}
	\item 
 data quantization based methods, which aim to replace the 
original data set with a small number of high quality 
`representative' points that capture relevant 
structure~\cite{goldberg2008online,yan2009fast};
\item
 Nystr\"{o}m method based methods, which aim to explore 
low-rank matrix approximations to speed up the matrix 
operations~\cite{fowlkes2004spectral}).
\end{enumerate}
While it is useful to define such preprocessors,
it is not satisfactory to simply reduce the size of similarity 
matrix to speed up the matrix calculations.
so that the related matrix operation can be performed in a 
desired time frame.

 What is  needed is an explicit connection between the amount of data 
reduction that is achieved by a preprocessor and the subsequent 
effect on the classification error. 
Some widely used data preprocessing 
approaches are based on data quantization, which replaces the 
original data set with a small number of high quality 
centroids that capture relevant 
structure~\cite{goldberg2008online,yan2009fast}.

Such approaches are often heuristic 
and do not quantify the relationship between the 
noise induced by the quantization and the final prediction risk.
An alternative approach to the computation problem is the 
\emph{Nystr\"{o}m method}, a low rank matrix approximation method that allows faster 
computation of the inverse. This method has been widely adopted, 
particularly in the context of approximations for SVMs \cite{drineas2005nystr$o$m,williams2001using,fine2001efficient} 
and spectral clustering \cite{fowlkes2004spectral}.

However, since the Nystr\"{o}m method uses interactions between
subsampled points and \emph{all} other data points, storage of all
points is required and thus, it becomes unsuitable for 
infinitely streamed data.
To our best knowledge, we are not aware of any 
online version of Nystr\"{o}m method
that could process an unbounded amount of streamed data.
Additionally, in an offline setting, the approaches based on the Nystr\"{o}m method have inferior performance to the quantization-based methods, if both of them are given the same time budget for computation. This was shown in an early work on the spectral clustering \cite{yan2009fast}. 

Using incremental $k$-centers \cite{charikar1997incremental}
which has provable worst case bound on the distortion, we quantify
the error introduced by quantization. Moreover, using regularization we 
show that the solution is stable, which gives the desired generalization
bounds. 

An interesting method is introduced in \cite{aggarwal2003framework}, which 
addresses context drift, or \emph{evolution} in the data streams. Clusters can 
emerge and die based on approximated recency.
But again this method is a heuristic and comes with no guarantees on the quality of the quantization. 

%% file: sec_related_ssl.tex
\section{Related Work in Semi-Supervised Learning (SSL)}
\label{sec:RelSemiSupervisedLearning}

\cite{zhu2003semi-supervised} extend their previous work \cite{zhu2003semi-supervised} to Gaussian processes
by no longer assuming that soft labels are fixed to the observed data. Instead they assume 
the data generation process $\bf x \to \bf y \to \bf t$, where  $\bf y \to \bf t$ is a noisy label generation 
with process modeled by a sigmoid. The posterior is not Gaussian and the authors use Laplace 
approximation to compute $p(\bf y_L, \bf y_U | \bf t_L)$. 
They discuss using different kernels for the learning of graph weights, such as the $\tanh$-weighted graph,
and optimize it either by maximizing the likelihood of labeled data or maximizing the alignment to labeled data. 

\cite{fergus2009semi-supervised} use the convergence of the eigenvectors of the normalized Laplacian
to eigenfunctions of weighted Laplace-Beltrami operators to scale graph-based SSL
to millions of examples. Assuming that the underlying distribution has a product form
(which is a reasonable assumption after a PCA projection), they estimated the density
using histograms for each dimension independently.  Therefore, they only needed to solve
$d$ generalized eigenvector problems on the backbone graph, where $d$ is 
the dimension of the data. Moreover, they only used 
the $k$ smallest eigenvectors and subsequently needed to solve only one $k \times k$
least squares problem.

\subsection{Semi-Supervised Max-Margin Learning}
\label{sec:SemiSupervisedMMLearning}

Most of the existing work on semi-supervised max-margin learning can be viewed as manifold regularization of SVMs \cite{belkin2006manifold} or semi-supervised SVMs with the hat loss on unlabeled data \cite{bennett1999semi-supervised}. The two approaches are reviewed in the rest of this section. Let $l$ and $u$ be the sets of labeled and unlabeled data respectively. Assume that  $f$ is a function from some \emph{reproducing kernel Hilbert space (RKHS)} $\cH_K$, and $\normw{\cdot}{K}$ is the norm that measures the complexity of $f$.

\subsubsection{Semi-supervised SVMs}
\label{sec:S3VM}
Semi-supervised support vector machines with the \emph{hat loss} $\widehat{V}(f, \bx) = \max\{1 - \abs{f(\bx)}, 0\}$ on unlabeled data \cite{bennett1999semi-supervised}:
\begin{align}
  \min_f \ \sum_{i \in l} V(f, \bx_i, y_i) +
  \gamma \normw{f}{K}^2 +
  \gamma_u \sum_{i \in u} \widehat{V}(f, \bx_i)
  \label{eq:S3VM}
\end{align}
compute max-margin decision boundaries that avoid dense regions of data. The hat loss makes the optimization problem non-convex. As a result, it is hard to solve the problem optimally and most of the work in this field has focused on approximations. A comprehensive review of these methods was done by \cite{zhu2008semi-supervised}.

In comparison to semi-supervised SVMs, learning of max-margin graph cuts (\ref{eq:MMGC}) is a convex problem. The convexity is achieved by having a two-stage learning algorithm. First, we infer labels of unlabeled examples using the regularized harmonic function solution, and then we minimize the corresponding convex losses.

\subsubsection{Manifold regularization of SVMs}
\label{sec:MR}

Manifold regularization of SVMs \cite{belkin2006manifold}:
\begin{align}
  \min_{f \in \cH_K} \ \sum_{i \in l} V(f, \bx_i, y_i) +
  \gamma \normw{f}{K}^2 +
  \gamma_u {\bf f}\transpose L {\bf f},
  \label{eq:MR}
\end{align}
where ${\bf f} = (f(\bx_1), \dots, f(\bx_n))$, computes max-margin decision boundaries that are smooth in the feature space. The smoothness is achieved by the minimization of the regularization term ${\bf f}\transpose L {\bf f}$. Intuitively, when two examples are close on a manifold, the minimization of ${\bf f}\transpose L {\bf f}$ leads to assigning the same label to both examples.

\subsection{Online Semi-Supervised Learning}
\label{sec:OnlineSemiSupervisedLearning}

The online learning formulation of SSL is suitable for 
\emph{adaptive} machine learning systems. In this setting, a few labeled examples are provided in advance 
and set the initial bias of the system while unlabeled examples are gathered
online and update the bias continuously.
In the online setting, learning is viewed as a repeated game against a potentially
adversarial nature. At each step $t$ of this game, we observe an example
$\bx_t$, and then \mbox{predict its} label $\hat{y}_t$. The challenge of
the game is that after it started we do not observe the true label $y_t$. Thus, if we want
to adapt to changes in the environment, we have to rely on indirect forms
of feedback, such as the structure of data.

Despite the usefulness of this paradigm for practical 
adaptive algorithms~\cite{grabner2008semi-supervised,goldberg2008online}, 
there is not much success in applying 
this paradigm to realistic problems,
especially when data arrive at a high rate such as in video applications.
\cite{grabner2008semi-supervised} applies online 
semi-supervised boosting to object tracking,
but uses a heuristic method to greedily label the unlabeled examples.
This method learns a binary classifier, where one of the
classes explicitly models outliers. In comparison, our approach is
multi-class and allows for implicit modeling of outliers. 
The two algorithms are compared empirically in Section~\ref{sec:OnlineQuantizedSSLExperiments}. 
%
\cite{goldberg2008online} develop an online version of manifold regularization of SVMs.
 Their method learns max-margin
decision boundaries, which are additionally regularized by the manifold.
Unfortunately, the approach was never
applied to a naturally online learning problem, such as adaptive face recognition.
Moreover, while the method is sound in principle, no theoretical guarantees are provided.

\cite{goldberg2011oasis:} combine semi-supervised learning and active learning
in a unified framework. Unlike our work which builds on manifold assumption, they
exploit cluster (or gap) assumption, \cite{chapelle2006semi-supervised}. 
The authors present a Bayesian model for this learning setting and
use a sequential Monte Carlo approximation for efficient online Bayesian update.

%% file: sec_related_anomaly.tex
\section{Related Work in Anomaly Detection (AD)}
\label{sec:RelatedWorkInAnomalyDetection}

\subsection{Unconditional Anomaly Detection}
\label{sec:UncoditionalAnomalyDetection}

In this section we review previous approaches for traditional anomaly detection.
Traditional anomaly detection looks for examples that deviate from the rest of the data 
if they are not expected from some underlying model.
A comprehensive review of many anomaly detection approaches can be found in \cite{markou2003novelty} and \cite{markou2003noveltya}.

\cite{scholkopf1999estimating} proposed the one-class SVM, that only needs positive (or non-anomalous) examples to learn the margin.
The idea is that the space origin (zero) is treated as the only example of  the `negative' class. 
In that way, the learning essentially estimates the support of the distribution.
The data that do not fall into this support have negative projections and can be considered anomalous.

\cite{eskin2000anomaly} assumes that the  number of anomalies is significantly lower than the number of normal cases. The author 
defines a distribution for the data as a mixture of majority ($M$) and anomalous distribution($A$): $D = (1-\lambda)M  + \lambda A$. 
He then iteratively partitions the dataset into the majority set $M_t$ and the anomalous set $A_t$. At the beginning $A_0 = \emptyset, M_0 = D$.
At each step $t$, it is determined whether the case $x_t$ is an anomaly. $x_t$ is considered anomalous if its displacement to the 
anomaly set ($M_t = M_{t-1} \setminus \{x\}$ and $A_t = A_{t-1} \cup \{x\}$) increases the log-likelihood $LL_{t-1}$ of the dataset by a predefined threshold $c$. 
If $LL_t - LL_{t-1} \le c$, $x_t$ remains marked as a normal case ($M_t = M_{t-1}$ and $A_t = A_{t-1}$). 
At the end, we get the final partition of $D$ into a normal set and an anomalous set.
	
The curse of high dimensionality is of concern in \cite{aggarwal2001outlier}. The authors search for
the abnormal lower dimensional projections by dividing  each attribute into the equi-depth (the same range of $f$ cases) ranges. 
Assuming statistical independence, each $k$-dimensional sub--cube in this
grid should contain the fraction of $f^k$ of total cases. The authors then search for $k$-dimensional
sub-cubes, where the presence of points is significantly lower than expected. As the brute force search
for projections is computationally infeasible, the authors use genetic algorithms to perform the search. 

In \cite{breunig2000lof:}, the authors expand $k$-distance (distance to the $k$ nearest neighbor) to get the so--called 
reachability distance for the object $O$ with respect to $p$ as $\mathrm{reach\_dist}(O,p) = \max(k\_\mathrm{distance}(p), \mathrm{dist}(O,p))$.
Using this smoothed distance, they define the \emph{local outlier factor} (LOF), which expresses the degree 
of the considered object being an outlier with respect to its neighborhood. LOF depends on $MinPts$, the number
of nearest points to define a local neighborhood.  Although this is data-dependent, the authors propose to calculate
the maximum LOF for $MinPts$ within a reasonable range (which was $30$--$50$ in their experiments) and threshold.
The bigger the LOF the more anomalous the object is. The authors give bounds for LOF and prove they are tight for important cases. For example, LOF is close to one for objects within the clusters. A useful property of LOFs is that it works well with cluster of different densities.

\cite{lazarevic2005feature} applies a bagging approach to improve the performance of local 
(nearest neighbor) anomaly detectors.
In every round of the algorithm a subset of features is selected and a local anomaly detector 
(such as LOF \cite{breunig2000lof:}) is applied.  Every round produces a scoring of all data, which is at the end
merged to get a final score using either breadth-first or cumulative-sum approach.

\cite{syed2010unsupervised} use a minimum enclosing ball approach to detect anomalies
in  clinical data similar to the data that we use in this work. The authors learn a minimum volume 
hypersphere that encloses the data for all patients.  The anomaly score is defined
as the distance from the center. They showed that this unsupervised approach 
performed similarly to the supervised approaches with prelabeled examples (Section~\ref{sec:ApproachesWithPrelabeledAnomalies}).


\cite{akoglu2010oddball:} performs anomaly detection on weighted graphs when
nodes do not follow discovered power laws between the number of neighbors and
the properties of the local neighborhood subgraph (total number of edges, total weight, and the 
principal eigenvalue of the weighted adjacency graph). The outlier score is 
defined as a distance to the fitting line. To account for the points
that fit the line but are far away from all other examples, the authors
combine their methods with a density based method, such as LOF \cite{breunig2000lof:}.

\cite{he2007graph-based} is a semi-supervised method that propagates the labels until a heuristic stopping criterion is reached.
Moreover, it uses unlabeled data to better estimate the prior in the case that the empirical distribution 
is skewed from the true distribution. 

\cite{moonesignhe2006outlier} use random walks to detect outliers.
They build their similarity matrix either by cosine similarity or by a number
of shared neighbors after thresholded cosine similarity. Anomalous 
nodes are identified as those with low connectivity. Connectivity 
is calculated using the Markov chain with the similarity as a 
transition matrix. Starting from the uniform connectivity assigned at
step 0, connectivity is spread according to the similarity matrix 
until convergence.

\subsubsection{Approaches with prelabeled anomalies} 
\label{sec:ApproachesWithPrelabeledAnomalies}

\cite{chawla2003smoteboost:} combine a boosting scheme with SMOTE (Synthetic Minority Over-sampling TEchnique). 
They do that in every iteration of smoothing. For continuous data, SMOTE generates a new sample by sampling a data point 
and one of its $k$ nearest neighbors and taking a random point on segment between them in the space. For discrete data, a new
point is created as a majority vote of the $k$ nearest neighbors for each feature. The authors show improvement with this method over just 
smoothing, just SMOTE and applying SMOTE once before the boosting for a minority class. The SMOTEboost approach generally improves recall but does not cause significant degradation in precision, thus improving the F-measure.

\cite{ma2003online} use support vector regression to learn the underlying temporal model (time event is modeled as a linear regression function of the previous events). A surprise is defined as the value outside the tolerance range.  Given the fixed length of the event, a probability of 
number of surprises actually happing is calculated.  When that is too small, an anomaly is declared. 

\subsubsection{Rare category detection}

\cite{pelleg2005active} aim to detect rare category which presumably correspond to the interesting anomalies in a pool-based
active learning framework. After a human expert labels some examples, the Gaussian mixture is fit to the data. Different hinting 
heuristics are then used to propose the new examples to be labeled by the expert. The authors propose \emph{interleave} heuristics which takes 
one example per mixture a time with low fit probability, not taking to account any mixture weight. This heuristic appears to be superior to the low-likelihood one (suggesting examples with the overall low fit probability) and ambiguous one (suggesting examples with uncertain class membership). 

\cite{he2008nearest-neighbor-based} attempt to detect rare categories in the data, assuming that examples from the 
rare category are self-similar, tightly grouped, and we have some knowledge about the class priors. 
The nearest neighbor based statistic is used to actively sample points corresponding to points with the maximum change in the local density. 

\subsection{Conditional Anomaly Detection (CAD)}
\label{sec:ConditionalAnomalyDetection}

We start with a short summary of our work. In \cite{hauskrecht2007evidence-based},
we introduce the concept of the conditional anomaly detection (CAD) and show its potential
for the medical records. For each case, we take its nearest neighbors and learn a Bayesian 
belief network (BBN) or a \naive Bayes model (NB) from them. The cases with low class-conditional probabilities
were deemed anomalous.  We discovered that while for BBN it was better 
to use all the cases for learning, for a more restricted NB a small neighborhood was beneficial.
The main problem with learning the structure of BBN is that it does not scale beyond a couple dozen features.
In \cite{valko2008distance}, we show the benefit of distance metric learning for the selection of closest cases.
We also use the softmax model \cite{mccullagh1989generalized} to calculate the class-conditional probability
of a probabilistic one nearest neighbor (similar to \cite{goldberger2004neighbourhood}) for this purpose.
In \cite{valko2008conditional}, we introduce a new anomaly measure based on the distance from 
the hyperplane learned by SVM \cite{vapnik1995nature} and perform the initial experiments
on the PCP (Section~\ref{sec:MARS}) dataset. 
We later conduct an extensive human evaluation study with a panel of 15 physicians in
 \cite{hauskrecht2010conditional}. Aside from our work which will be reviewed in more detail in later chapters, we also describe other early work along these lines.

\cite{valizadegan2007kernel} use the kernel based weighted nearest neighbor approach to jointly
estimate the probabilities of the examples being mislabeled. The joint estimation is posed as an
optimization problem and solved with Newton methods. A regularization is needed to 
avoid one of the classes deemed to be completely mislabeled.

In \cite{song2007conditional}, a user defines a partitioning of the features into two groups: the \textbf{indicator} features --- those that can be directly indicative of an anomaly and the \textbf{environmental} features, which cannot, but can influence the indicator features. The indicator ($y$) and the environmental ($x$) variables are modeled separately both as the mixtures of multivariate Gaussians ($y \sim U$ and $x \sim V$). A mapping function is defined between those mixtures as the probability of choosing a Gaussian for an indicator variable given an environmental one $p(V_j|U_i)$. The authors assume the following generative process for a datapoint $\langle x,y \rangle$: If $x$ is a sample from $U_i$ then a die is tossed, according to  $p(V_j|U_i)$, to determine which Gaussian from $V$ will produce $y$ and subsequently $y$ is produced. Since it is not known, which $U_i$ the $x$ was sampled from, the likelihood of $f_{CAD}(y|\Theta,x)$ is computed as a weighted sum over Gaussians $U_i$. The model is learned via EM, either directly --- optimizing all parameters at once (named as DIRECT), optimizing first parameters for Gaussians and then for the mapping function (FULL), or optimizing the indicator Gaussians, the environmental Gaussians and the mapping function separately (SPLIT).

The work on cross-outlier detection \cite{papadimitriou2003cross-outlier} is  also related to CAD.
Papadimitriou and Faloutsos~\cite{papadimitriou2003cross-outlier}
defined the notion of the cross-outliers as examples that seem normal when considering the distribution of examples from the assigned class, but are abnormal when considering the samples from the other class. For each sample $(\bx,y)$, they compute two statistics based on the similarity of $\bx$ to its neighborhood from the samples belonging to class $y$ and samples not belonging to class $y$. An example is considered anomalous if the first statistic is significantly smaller than the second statistic. Unfortunately, the method 
is not very robust to fringe points (Figure~\ref{fig:fringe}) \cite{papadimitriou2003cross-outlier}.

%

In his dissertation, \cite{das2009detecting} aims to detect several kinds of individual 
and group anomalies. The approaches relevant to this work are 
\textit{conditional} and \textit{marginal} methods for 
individual record anomalies, ignoring rare values. 
For the data $t$ and the subsets of attributes ($A,B,C$) he computes the ratios of the form 
$\frac{P(A,B)}{P(A)P(B)}$ for the marginal method and 
$\frac{P(A,B|C)}{P(A|C)P(B|C)}$ for the conditional method. 
The goal is to find unusual occurrences of the attribute values.
The records that have low ratios are considered anomalous.
The normalization of the joint probabilities by 
the marginal provabilities takes care of rare records, because
those also have small marginals. The dissertation describes several 
speedups to compute the ratios for exponentially many subgroups 
to allow the methods to scale up.

%% file: chapter_ssl.tex
\chapter{Semi-Supervised Learning}
\label{sec:SemiSupervisedLearning}

Semi-supervised learning (SSL) is a field of machine learning that studies learning from both labeled and unlabeled examples. This learning paradigm is suitable for real-world problems, where data is often abundant but the resources to label them are limited. As a result, many semi-supervised learning algorithms have been proposed in the past few years \cite{zhu2008semi-supervised}. The closest to this work are semi-supervised support vector machines (S3VMs) \cite{bennett1999semi-supervised}, manifold regularization of support vector machines (SVMs) \cite{belkin2006manifold}, and harmonic function solutions on data adjacency graphs \cite{zhu2003semi-supervised}. 

SSL is very closely related to transductive inference (Chapters 24 and 25 in~\cite{chapelle2006semi-supervised}).
In both approaches we have access to the unlabeled examples that we can take advantage of.  
Traditionally in SSL, we want to use the unlabeled data to learn a function $f$ that can be later used to classify previously unseen examples. 
We present one such approach where we combine the harmonic solution on a data similarity graph with a max-margin inference in Section~\ref{sec:MMGC}.
In other scenarios, we may not need to learn such a function. In this case, we can focus just on classifying the unlabeled examples at hand. 
Even then, the prediction on unseen examples is still possible using out of sample extension methods \cite{bengio2004out-of-sample}.

In this dissertation we study graph-based methods for SSL, because they can model complex interactions between the examples.
However, graph-based methods (such as label propagation) do not scale beyond several thousand examples unless we use parallel architectures.
One of the solutions is to reduce the number of nodes and create a representative back-bone graph. 
Typically, one can downsample the data or use some quantization technique (such as $k$-means) 
to come up with a smaller graph. Yet these approaches do not consider the quality
of semi-supervised learning inference for this backbone graph. In Section~\ref{sec:LargeScaleSemiSupervisedLearningWith} we introduce 
a new objective function that lets us incorporate the quality of inferences into the
construction of the backbone graph.

Furthermore, in Section~\ref{sec:OnlineLearningWithQuantizedGraphs} we investigate an online learning formulation of SSL,  which is suitable for  \emph{adaptive} machine learning systems.  In this setting, a few labeled examples are provided in advance  and set the initial bias of the system, while unlabeled examples are gathered online and update the bias continuously. In the online setting, learning is viewed as a repeated game against a potentially adversarial nature. At each step $t$ of this game, we observe an example $\bx_t$ and then \mbox{predict its} label $\hat{y}_t$. The challenge of the game is that after it started we do not observe the true label $y_t$. Thus, if we want to adapt to changes in the environment, we have to rely on indirect forms of feedback, such as the structure of data. When data arrive in a stream, the dual problems of computation and data storage arise for any SSL method. We therefore propose a fast approximate online SSL algorithm that solves for the harmonic solution on an approximate graph. 

For all our methods, we introduce the regularized harmonic solution (Section~\ref{sec:reg HFS}) to achieve better stability properties.
With such regularization, we can control the confidence of labeling unlabeled examples and discount the outliers in the data.
In the following, we start with some needed background in graph theory and then continue with the just mentioned approaches for SSL.

\input{sec_background.tex}

\section{Max-Margin Graph Cuts}
\label{sec:MMGC}

In this part we present our algorithm that combines the harmonic solution with max-margin learning to learn a classifier $f$ from some \emph{reproducing kernel Hilbert space (RKHS)}.  In the scenarios, where we do not want to store all the examples in the dataset and perform the inference transductively when the new data arrive, we may prefer to learn such $f$ instead.
Our semi-supervised learning algorithm involves two steps. First, we obtain the regularized harmonic function solution $\bell^\ast$ \eqref{eq:closed-form reg HFS}. The solution is computed from the system of linear equations $(L_{uu} + \gamma_g I) \bell_u = W_{ul} \bell_l$. This system of linear equations is sparse when the data adjacency graph $W$ is sparse. Second, we learn a max-margin discriminator, which is conditioned on the labels induced by the harmonic solution. The optimization problem is given by:
\begin{align}
  \min_{f \in \cH_K} & \quad \sum_{i : \abs{\ell_i^\ast} \geq \eps}
  \!\!\! V(f, \bx_i, \sgn(\ell_i^\ast)) + \gamma \normw{f}{K}^2
  \label{eq:MMGC} \\
  \textrm{s.t.} & \quad
  \bell^\ast = \arg\min_{\bell \in \realset^n}
  \bell\transpose (L + \gamma_g I) \bell \nonumber \\
  & \quad \textrm{s.t.} \ \ell_i = y_i \textrm{ for all } i \in l;
  \nonumber
\end{align}
where $V(f, \bx, y) = \max\{1 - y f(\bx), 0\}$ denotes the \emph{hinge loss}, $\cH_K$, and $\normw{\cdot}{K}$ is the norm that measures the complexity of $f$.
	
The training examples $\{\bx_i\}$ in our problem are selected based on our confidence into their labels. When the labels are highly \emph{uncertain}, which means that $\abs{\ell_i^\ast} < \eps$ for some small $\eps \geq 0$, the examples are excluded from learning. Note that as the regularizer $\gamma_g$ increases, the values $\abs{\ell_i^\ast}$ decrease towards 0 (Figure \ref{fig:HFS}), and the $\eps$ thresholding allows for smooth interpolations between supervised learning on labeled examples and semi-supervised learning on all data. The trade-off between the regularization of $f$ and the minimization of hinge losses $V(f, \bx_i, \sgn(\ell_i^\ast))$ is controlled by the parameter $\gamma$.

Due to the representer theorem \cite{wahba1999support}, the optimal solution $f^\ast$ to our problem has a special form:
\begin{equation*}
  f^\ast(\bx) =
  \sum_{i : \abs{\ell_i^\ast} \geq \eps} \alpha_i^\ast k(\bx_i, \bx),
\end{equation*}
where $k(\cdot, \cdot)$ is a Mercer kernel associated with the RKHS $\cH_K$. Therefore, we can apply the kernel trick and optimize rich classes of discriminators in a finite-dimensional space of ${\bf \alpha} = (\alpha_1, \dots, \alpha_n)$. Finally, note that when $\gamma_g = \infty$, our solution $f^\ast$ corresponds to supervised learning with SVMs.

In some aspects, manifold regularization (Section~\ref{sec:MR}) is similar to max-margin graph cuts. In particular, note that its objective (\ref{eq:MR}) is similar to the regularized harmonic function solution (\ref{eq:reg HFS}). Both objectives involve regularization by a manifold, ${\bf f}\transpose L {\bf f}$ and $\bell\transpose L \bell$, regularization in the space of learned parameters, $\normw{f}{K}^2$ and $\bell\transpose I \bell$, and some labeling constraints $V(f, \bx_i, y_i)$ and $\ell_i = y_i$. Since max-margin graph cuts are learned conditionally on the harmonic function solution, the problems (\ref{eq:MMGC}) and (\ref{eq:MR}) may sometimes have similar solutions. A necessary condition is that the regularization terms in both objectives are weighted in the same proportions, for instance, by setting $\gamma_g = \gamma / \gamma_u$. We adopt this setting when manifold regularization of SVMs is compared to max-margin graph cuts in Section \ref{sec:MaxMarginGraphCutsExperiments}.

\section{Joint Quantization and Label Propagation}
\label{sec:LargeScaleSemiSupervisedLearningWith}

Graph-based semi-supervised learning methods do not scale well to large data sets, mainly because their inference procedures require the computation of the inverse of an $n\times n$ matrix, where $n$ is the size of the underlying graph that is equal to the size of the dataset. A typical solution to address this problem is to downsize the graph to a smaller \emph{backbone} graph and perform the inference on this reduced representation. The key challenge is to decide on what elements should be included in the backbone graph. Typical solutions include sub-sampling, clustering, or a Nystr\"om approximation. However, these techniques do not consider the quality of semi-supervised learning inferences for this backbone graph. We introduce a new objective function that lets us incorporate the quality of inferences into the construction of the backbone graph. 

To reduce the computational complexity of \eqref{eq:soft HFS}, we replace all $n$ nodes of the
similarity graph $G$ with a set $C = [{\bf c}_1, ... {\bf c}_m, ..., {\bf c}_{m+k}]\transpose$ of $(m+k)\ll n$ representative nodes to create a backbone graph $\tilde{G}$. Notice that $ {\bf c}_i={\bf x}_i$ for $i=1,...,m$. We want to find $\tilde{G}$ such that it is a good representation of $G$ in constructing the manifold. Let us assume for a moment that we do know the best set of examples $\tilde{G}$. Then, Equation~\eqref{eq:soft HFS} becomes:

\begin{equation}[Quantized unconstrained regularization]
  \bell^\star = \argmin_{\bell \in \realset^n} \
  (\bell - \by)\transpose F^C (\bell - \by) + \bell\transpose L^C \bell.
	\label{eq:quantized ssl}
\end{equation}

\noindent In general, $C \in \realset^{(m+k)\times d}$ can be obtained by fixing the first $m$ labeled examples and choosing $k$ unlabeled points by subsampling the dataset, clustering or other means of quantization. As mentioned earlier, the common approach is to select the set $C$ first and only then perform the inference~\eqref{eq:quantized ssl}. In this work, we will perform both the quantization and the inference jointly by adding the quantization penalty
of the elastic nets to the objective function in~\eqref{eq:quantized ssl}. As we will see in Section~\ref{sec:AnalysisOfJointQuantizationAndLabelPropagation}, this simple joint approach will produce interesting properties. The new objective function is:
\begin{equation}[Joint quantization and soft harmonic solution]
  [\bell^\star, \{c_j\}_{j=m+1}^{m+k}]  =   \argmin_{\bell \in \realset^n, \{c_j\}_{j=m+1}^{m+k}} \
   (\bell - \by)\transpose  F^C (\bell - \by) + \bell\transpose L^C \bell \
	\gamma_q \left( \frac{(m+k)^2}{n} \sum_{x_i \in K_j} ||c_j - x_i||^2\right) 
	\label{eq:quantized kmeans}
\end{equation}	

\noindent where $K_j$ is the set of examples for which $c_j$ is the nearest centroid and $\gamma_q$ is a cost parameter for the quantization penalty. We emphasize that we automatically consider all labeled examples as a fixed part of $C$ and the optimization to learn the representing centroids are affected by the position of labeled examples. As we will see in Section~\ref{sec:AnalysisOfJointQuantizationAndLabelPropagation}, the above objective function has an interesting property: when optimized to find the centroids, it learns the principle manifold.

Adding the quantization penalty makes the objective function non-convex and hence difficult to optimize. To minimize \eqref{eq:quantized kmeans}, we propose to use an alternating optimization approach~\cite{bezdek2002some}, where we alternate between
1) \emph{label propagation} --- inferring labels $l$ on $\tilde{G}$, and
2) \emph{quantization} --- selecting the set $C$ for $\tilde{G}$.
Starting with random seeds of unlabeled examples (or the output of $k$-means algorithm) as the initial centroids, we iterate the following steps.

\subsection{Label Propagation}
\label{sec:LabelPropagation}
\noindent  Once $C$ is fixed, the labels can be computed by solving the following convex optimization problem:

\begin{equation*}
  \bell^\star = \argmin_{\bell \in \realset^n} \
  (\bell - \by)\transpose F^C (\bell - \by) + \bell\transpose L^C \bell	
	\label{eq:label propagation step}
\end{equation*}

\noindent This problem has a closed form solution: $\bell^\star  = ((F^C)^{-1}L^C+I)^{-1}\by$ (Section~\ref{sec:reg HFS}).

\subsection{Quantization}
\label{sec:Quantization}
\noindent To learn the centroids $C$ when $\bell$ is fixed, first notice that:
	\[
\bell\transpose L^C \bell=\sum_{i, j}\left(\frac{l_i}{n_i} - \frac{l_j}{n_j}\right)^2 W_{ij}^C
\]

\noindent where $(n_i=1)_{i=1}^{m+k}$ for unnormalized graph Laplacian $L=D^C-W^C$~\cite{luxburg2007tutorial} and $(n_i=\sqrt{d_i})_{i=1}^{m+k}$ for normalized graph Laplacian $L = I-D^{-1/2}WD^{-1/2}$~\cite{zhou2004learning}. Considering that $(\bell - \by)\transpose F^C (\bell - \by)$ in (\ref{eq:quantized kmeans}) is not dependent on $C$, we have the following optimization problem to learn $C$ if we use the widely used Gaussian kernel\footnote{It is straightforward to apply similar derivation for other similarity functions.} as the similarity function $W$:

\begin{equation}[Quantization step]
	\{c_j\}_{j=m+1}^{m+k} = \ 
	\argmin_{\{c_j\}_{j=m+1}^{m+k}} \
	\sum_{i, j}\left(\frac{l_i}{n_i} - \frac{l_j}{n_j}\right)^2 \exp\left(\frac{-||c_j-c_i||^2}{2\sigma^2}\right) \ 
	+ \gamma_q \left( \frac{(m+k)^2}{n}\sum_{i \in K_j} \
	||c_j - x_i||^2\right)
	\label{eq:quantization step}
\end{equation}

\noindent  To learn the centers by optimizing (\ref{eq:quantization step}), we first approximate the exponential function using Taylor expansion\footnote{Notice we could also use the convexity of the exponential function to obtain an upper bound and have a more rigorous derivation. However, the results are very similar and we omit the details to simplify the description.}:


\begin{equation*}
\exp\left(\frac{-||c_j-c_i||^2}{2\sigma^2}\right) \approx 1 - \frac{||c_j-c_i||^2}{2\sigma^2},
\end{equation*}

\noindent This results in the following optimization problem:

\begin{equation}[Approximate quantization step]
  \{c_j\}_{j=m+1}^{m+k} =  \argmin_{\{c_j\}_{j=m+1}^{m+k}}
	\frac{-1}{(m+k)^2}\sum_{i , j}\left( \frac{(l_i - l_j)^2}{2\sigma^2} \right)  ||c_j-c_i||^2 + \
	\frac{\gamma_q}{n}  \sum_{i \in K_j} ||c_j - x_i||^2 	
	\label{eq:quantization step2}
\end{equation}


\noindent Taking derivatives of \eqref{eq:quantization step2} with respect to $(c_j)_{j=m+1}^{m+k}$ and
setting them to zero, we obtain  the following system of $k$ linear equations for $j=m+1,\dots,m+k$ :

\begin{equation}[Solution for the quantization step]
	 \sum_{i} c_i  \frac{(l_i-l_j)^2}{(m+k)^2\sigma^2}  + \
	c_j \left(  2\gamma_q\frac{|K_j|}{n} - \sum_{i} \frac{(l_i-l_j)^2}{(m+k)^2\sigma^2}  \right) \
	=  \frac{2\gamma_q}{n}\sum_{i\in K_j}x_i,
\label{eq:sle for centroids}	
\end{equation}
where $|K_j|$ is the number of examples assigned to the center $c_j$. In order to optimize the system of linear equations in~\eqref{eq:sle for centroids}, we iterate between optimizing the centroids and the assignment of the examples to the centroids, a strategy similar to $k$-means.

\noindent Notice that the labeled examples $c_1,..,c_m$ affect learning the centroids by 

\begin{enumerate}
	\item absorbing some of the unlabeled examples that are close to the labeled examples and do not need an unlabeled examples representative. In other words, in the quantization process, we remove those examples that are very close to the labeled examples;
	\item controlling the position of unlabeled centers through the first term in (\ref{eq:sle for centroids}). 
\end{enumerate}

\noindent Algorithm~\ref{fig:QSSLPM} outlines the \emph{elastic-joint} algorithm.

\begin{algorithm}  
\small{
  \begin{tabbing}
    \hspace{0.1in} \= \hspace{0.1in} \= \hspace{0.1in} \= \hspace{0.1in} \= \kill
    {\bf Inputs:} \\		
    \> examples $\{\bx_i\}_{i=1}^n$ \\
		\> labels $l$, such that $l_i = \pm 1$ for labeled and $l_i = 0$ for unlabeled examples\\
		\> $k$: number of centroids  and size of the $\tilde{G}$ \\
		\> regularizer $\gamma_q$ for the quantization \\
    {\bf Algorithm (elastic-joint):} \\
		\> randomly initialize the set of $k$ centroids $C$ \\
		\> {\bf do until} convergence \\
		\>\> infer labels on the graph: \\
		\>\>\> build a quantized data similarity graph $\tilde{G}$ on $C$\\
		\>\>\> compute $L^C$ as the graph Laplacian of $\tilde{G}$\\
    \>\>\> $\bell^\star =   \argmin_{\bell \in \realset^n} (\bell - \by)\transpose F^C (\bell - \by) + \bell\transpose L^C \bell $  \\
    \>\> perform quantization \\
		\>\>\> calculate $C$ by solving the following system of linear equations for $j=m+1,\dots,m+k$:
		\\ 
		\>\>\> $\sum_{i} c_i\frac{(l_i-l_j)^2}{(m+k)^2\sigma^2} + \
	c_j \left(  2\gamma_q\frac{|K_j|}{n} - \sum_{i} \frac{(l_i-l_j)^2}{(m+k)^2\sigma^2}  \right) \
	=  \frac{2\gamma_q}{n}\sum_{i\in K_j}x_i$ \\
    {\bf Outputs:} \\
    \> predictions $\hat{y} = |\bell^\star|$
  \end{tabbing}
	}
  \caption{Quantized semi-supervised learning with principal manifolds}
  \label{fig:QSSLPM}
\end{algorithm}

\subsection{Final Inference Scheme for Unlabeled Examples}
\label{sec:appox}
After solving the objective function in~\eqref{eq:quantized kmeans} using Algorithm~\ref{fig:QSSLPM}, we need to infer the labels of unlabeled examples from the labels of the centroids. The common approach in the literature~\cite{chapelle2006semi-supervised,delalleau2005efficient} is to use the weighted $k$-NN. The label of any new example $x$ (including the unlabeled examples) is  computed as follows:
\begin{align}
\hat{y}=\frac{\sum_{i=1}^{m+k}W'(x,c_i)\bell_i}{\sum_{i=1}^{m+k}W'(x,c_i)}
\end{align}
where $W'$ is a symmetric edge weighting function, such as Gaussian kernel~\cite{chapelle2006semi-supervised}. We only use 1-NN  for the inference, as we found that it produces the best results for the proposed method and the baselines.

\subsection{Time Complexity}
\label{sec:complexity}

\begin{figure}	
\begin{center}
\includegraphics[width=0.5\columnwidth,clip]{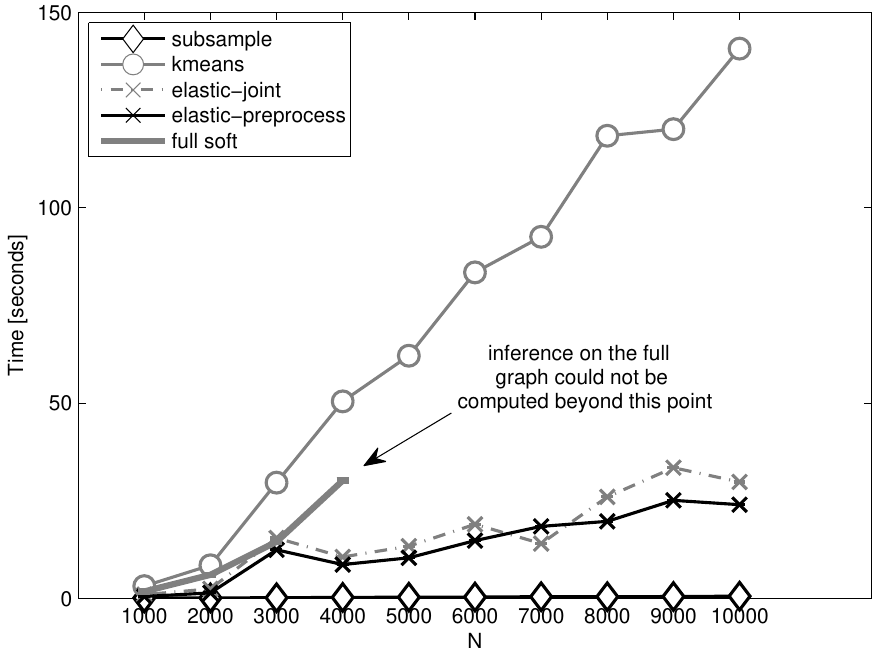}
\caption{Running time for different methods on the SecStr dataset}%
\label{fig:time} 
\end{center} 
\end{figure}

Suppose Algorithm~\ref{fig:QSSLPM} takes $T$ iterations to converge. Each iteration has two optimization steps: 1) running SSL, and 2) constructing the backbone graph. Each run of the SSL algorithm needs the computation of the inverse of a matrix of size $n+k$ that takes $O\big((m+k)^3\big)$. Each run of the backbone graph construction iterates between two steps 1) assigning examples to centroids, and 2) solving the system of linear equations~\eqref{eq:quantization step2}. The second step is the major step and takes $O\big(k^3\big)$ which results in $O\big(tk^3\big)$ time complexity for $t$ iterations.  Since $(m+k)^3\geq tk^3$ for even a small number of labeled examples, the complexity of the proposed method is $O\big(T(m+k)^3\big)$. In our experiments, we found that $T$ is usually very small; i.e. less than $10$. Figure~\ref{fig:time} shows the running time of different methods on 
SecStr dataset \cite{chapelle2006semi-supervised} by changing the total number of unlabeled examples from $1000$ to $10000$. 
Different quantization approaches are described in Section~\ref{sec:JointQuantizationAndHarmonicSolutionExperiments}. We fixed the number of labeled examples to $m=10$, the number of centroids to $k = 100$, and varied the number of sampled points $N$ from the original 83679 examples. This plot clearly shows that the proposed method scales very well with the large number of examples. Also note that we used the $k$-means function in MATLAB, which seems extremely slow.

\section{Online Semi-Supervised Learning With Quantized Graphs}
\label{sec:OnlineLearningWithQuantizedGraphs}

The regularized HS (Section \ref{sec:reg HFS}) is
an offline learning algorithm. This algorithm can be made na\"{i}vely online
by taking each new example, connecting it to its neighbors,
and recomputing the HS.
Unfortunately, this na\"ive implementation has computational complexity
$O(t^3)$ at step $t$, and computation becomes infeasible
as more examples are added to the graph.

To address the problem, we use \emph{data quantization} \cite{gray1998quantization}
and substitute the vertices in the graph with a smaller set of $k$ distinct centroids. 
The resulting $t \times t$
similarity matrix $\W$ 
has many identical rows/columns. We will show 
that the \emph{exact} HS using $\W$ may be reconstructed from
a much smaller $k \times k$ matrix $\compactWoq$, where $\compactWoq_{ij}$
contains the similarity between the $i^{th}$ and $j^{th}$ centroids,
and a vector $\bv$ of length $k$,
where $v_i$ denotes the number of points collapsed into the $i^{th}$ centroid. 
To show this, we introduce the matrix $\Woq = V \compactWoq V$
where $V$ is a diagonal matrix containing the counts in $\bv$ on the diagonal.


\begin{proposition}
\label{prop:compact HFS} 
The harmonic
solution \eqref{eq:reg HFS} using $\W$ can be computed compactly as
\begin{align*}
  \hfsoq =
  (\Loq_{uu} + \gamma_g V)^{-1} \Woq_{ul} \bell_l,
\end{align*}
where $\Loq$ is the Laplacian of $\Woq$.
\end{proposition}
\noindent {\bf Proof:} Our proof uses the electric circuit
interpretation of a random walk \cite{zhu2003semi-supervised}. More
specifically, we show that $\W$ and $\Woq$ represent identical
electric circuits and, therefore, their harmonic solutions are the
same.

In the electric circuit formulation of $\W$, the edges of the graph
are resistors with the conductance $w_{ij}$. If two vertices $i$ and
$j$ are identical, then by symmetry, the HS must assign the same value
to both vertices, and we may replace them with a single vertex. Furthermore,
 they correspond to the ends of resistors in
parallel. The total conductance of two resistors in
parallel is equal to the sum of their conductances. Therefore, the two
resistors can be replaced by a single resistor with the conductance of
the sum. A repetitive application of this rule gives $\Woq = V
\compactWoq V$, which yields the same HS as $\W$.
In Section \ref{sec:reg HFS}, we showed that the regularized HS can be
interpreted as having an extra sink in a graph. Therefore, when two
vertices $i$ and $j$ are merged, we also need to sum up their sinks. A
repetitive application of this rule yields the term $\gamma_g V$ in
our closed-form solution. \qed

We note that Proposition \ref{prop:compact HFS} may be applied
whenever the similarity matrix has identical rows/columns, not just
when quantization is applied. However, when the data points are 
quantized into a fixed number
of centroids $\ng$, it shows that we may compute the HS for the $t^{th}$ point
in $O(\ng^3)$ time. 
Since the time complexity of computation on the quantized graph
is independent of $t$, it gives a suitable algorithm for online learning.
%

We now describe how to perform incremental quantization with provably 
nearly-optimal distortion.


\begin{algorithm}  
  \small{
	\centering
  \begin{tabbing}
    \hspace{0.1in} \= \hspace{0.1in} \= \hspace{0.1in} \= \hspace{0.1in} \= \kill
    {\bf Inputs:} \\
    \> an unlabeled example $\bx_t$ \\
    \> a set of centroids $C_{t - 1}$ \\
    \> vertex multiplicities $\bv_{t - 1}$\\
    {\bf Algorithm:} \\
    \> if $(\abs{C_{t - 1}} = \ng + 1)$ \\
    \>\> $R \gets m R$ \\
    \>\> greedily repartition $C_{t - 1}$ into $C_t$ such that: \\
    \>\>\> no two vertices in $C_t$ are closer than $R$ \\
    \>\>\> for any $\bc_i \in C_{t - 1}$ exists $\bc_j \in C_t$ such that $d(\bc_i, \bc_j) < R$ \\
    \>\> update $\bv_t$ to reflect the new partitioning \\
    \> else \\
    \>\> $C_t \gets C_{t - 1}$ \\
    \>\> $\bv_t \gets \bv_{t - 1}$ \\
    \> if $\bx_t$ is closer than $R$ to any $\bc_i \in C_t$ \\
    \>\> $\bv_t(i) \gets \bv_t(i) + 1$ \\
    \> else \\
    \>\> $\bv_t(\abs{C_t} + 1) \gets 1$ \\
		\>\> $C_t(\abs{C_t} + 1) \gets\bx_t$ \\		
    \> build a similarity matrix $\compactWoq_t$ over the vertices $C_t$ \\
    \> build a matrix $V_t$ whose diagonal elements are $\bv_t$ \\
    \> $\Woq_t = V_t \compactWoq_t V_t$ \\
    \> compute the Laplacian $L^{\rm q}$ of the graph $\Woq_t$ \\
    \> infer labels on the graph: \\
    \>\> $\hfsoq[t]  \gets
		\arg\!\min_{\bell} 
		\bell\transpose (\Loq + \gamma_g V_t) \bell$ \\
    \>\> $\textrm{s.t.} \ \ell_i = y_i$
    for all labeled examples up to time $t$ \\
    \> make a prediction $\hat{y}_t = \sgn(\hfsoq_t[t])$ \\
    {\bf Outputs:} \\
    \> a prediction $\hat{y}_t$ \\
    \> a set of centroids $C_t$ \\
    \> vertex multiplicities $\bv_t$ 
  \end{tabbing}	
  }
  \vspace{-0.07in}
  \caption{Online quantized harmonic solution}
  \label{fig:online quantized HFS}
\end{algorithm}

\subsection{Incremental k-centers}
\label{sec:algorithm}

We make use of the \emph{doubling algorithm} for incremental $k$-center
clustering~\cite{charikar1997incremental} which assigns points
to centroids in a near optimal way. In particular, it is
a $(1+\epsilon)$-approximation 
with cost measured by the maximum quantization error over all points.
In Section \ref{quantization_bound}, we show
that under reasonable assumptions, the quantization error
goes to zero as the number of centroids increases.


The algorithm
of~\cite{charikar1997incremental} 
maintains a set of centroids $C_t = \set{\bc_1, \bc_2,
 \dots}$ such that the distance between any two vertices in $C_t$ is
at least $R$ and $\abs{C_t} \leq \ng$ at the end of each iteration. 
For each new \mbox{point $\bx_t$}, if its
distance to some $\bc_i \in C_t$ is less than $R$, the point is
assigned to $\bc_i$. Otherwise,
 the distance of $\bx_t$ to $\bc_i \in
C_t$ is at least $R$ and $\bx_t$ is added to the set of centroids
$C_t$. If adding $\bx_t$ to $C_t$ results in
$\abs{C_t} \! > \! \ng$, the scalar $R$ is
doubled and $C_t$ is greedily repartitioned such that no two vertices
in $C_t$ are closer than $R$. The doubling of $R$ also ensures 
that $\abs{C_t} \! < \! \ng$.

Pseudocode of our algorithm is given in Algorithm~\ref{fig:online quantized HFS}.
We make a small modification to the original quantization algorithm
in that, instead of doubling $R$, we multiply it with some $\Rmultiplier > 1$.
This still yields a $(1+\epsilon)$-approximation algorithm
as it still obeys the invariants given in Lemma 3.4 in 
~\cite{charikar1997incremental}.
We also 
maintain a vector of multiplicities ${\bf v}$,
which contains the number of vertices that each centroid represents.
At each time step, the HS is calculated using the updated quantized
graph, and a prediction is made.

The incremental $\ng$-centers algorithm also has the advantage that
it provides a variable $R$, which may be used to bound the maximum
quantization error.
In particular, at any point in time $t$, the distance
of any example from its centroid is at most
$R\Rmultiplier/(\Rmultiplier-1)$.
To see this, consider the following: 
As the new data arrive we keep increasing $R$ by multiplying it by some $\Rmultiplier>1$. 
But for any point at any time, the centroid assigned to a vertex is 
at most $R$ apart from the previously assigned centroid, which is 
at most $R/\Rmultiplier$ apart from the centroid assigned before, etc. 
Summing up, at any time, any point is at most
\begin{displaymath}
R + \frac{R}{\Rmultiplier} + \frac{R}{\Rmultiplier^2} + \cdots = R\left(1 + \frac{1}{\Rmultiplier} + \frac{1}{\Rmultiplier^2} + \cdots \right) = 
\frac{R\Rmultiplier}{\Rmultiplier-1}	
\end{displaymath}
apart from its assigned centroid, where $R$ is the most recent one.

\section{Parallel Multi-Manifold Learning}
\label{sec:ParallelMultiManifoldLearningMethods}

Most of the SSL methods that exploit the manifold assumption (such as graph-based SSL methods)
assume that the data lie on a single manifold. A more plausible 
setting, however, is that the data lie on a mixture of manifolds \cite{goldberg2009multi-manifold}.
For example, in digit recognition each digit lies on its
own manifold in the feature space \cite{goldberg2009multi-manifold}.

In this work, we use the multi-manifold idea from a different perspective.
We assume no or little interaction between the manifolds and learn 
the manifolds \emph{in parallel} to achieve a speedup in computation.
The speedup is accomplished in two ways:

\begin{enumerate}
	\item Assuming independence between the manifolds, we can solve several smaller problems instead.
	For example, in the ideal case (Section~\ref{sec:ParallelMultiManifoldLearning}), the similarity matrix
	will consist of $b$ block-diagonal blocks of the equal size. Therefore, to approximate 
	the harmonic solution (HS) 
	on a graph with $n$ nodes which takes $\O(n^3)$ time, we can instead solve $b$ HS problems 
	on $b$ graph with $n/b$ nodes, each taking only $\O((n/b)^3)$ and achieve a polynomial speedup.
		
	\item Using multi-core and/or multi-processor architectures, we can solve the smaller problems in 
	parallel and achieve additional, potentially linear speedup, up to the number of cores.
\end{enumerate}

\noindent Assuming the independence of the manifolds may come with a cost in accuracy
when the manifolds are not independent.
We study this theoretically in Section~\ref{sec:ParallelMultiManifoldLearning}
and empirically in Section~\ref{sec:ParellelSSL}, where we measure the 
trade-off between the computational speedup and decrease in prediction accuracy. 

%% file: sec_background.tex
\section{Graphs as Data Models}
\label{sec:GraphAsADataModel}

Many of the methods presented here are based on a graph representation of the data.
Having some data, we create a undirected weighted graph $G = (V,E)$ with set of
vertices $V$ and set of edges $E$, associating every data
point with a graph vertex.  Next, we define a 
non-negative weight function $V \times V \to \realset$ such that $w_{ij} = w_{ji}$.
In the case that $\{i,j\} \notin E(G)$, $w_{ij} = 0$. 
Let the similarity matrix $W = \{w_{ij}\}$ denote a matrix of all edge weights which encode
how similar the vertices are to each other. 
We define degree $d_i$ of the vertex $i$ as the sum of all edges coinciding with $i$:

\[d_i = \sum_j w_{ij} \]

\noindent and the diagonal matrix $D$ with $D_{ii} = d_i$.
Let volume ${\rm vol}(G)$ of graph $G$ be the sum of  all its weights:

\[{\rm vol}(G) = {\rm vol}(W) = \sum_i d_i = \sum_{i,j} w_{ij} \]

\noindent  Now, let us define an unnormalized graph Laplacian $L$ as

\[L(G) = L(W) = D - W\]

\noindent  and the symmetric normalized graph Laplacian as 

\[L_{\rm sym}(G) = L_{\rm sym}(W) = D^{-\frac12}LD^{-\frac12} = I - D^{-\frac12}WD^{-\frac12}.\]

\noindent It can be easily shown that for any $\bh \in \realset^n$:

\begin{equation*}
\bh \transpose L \bh = \frac12 \sum_{ij} w_{ij}(h_i - h_j)^2.
\end{equation*}

\subsection{Stationary Distribution of a Random Walk}
\label{sec:RandomWalkComputation}

Here we describe a way to compute a stationary distribution of 
a (non-absorbing) random walk on the data similarity graph in a closed form.
 Let us 
define the random walk as follows: In every step of a random walk, we jump from a node to its neighbors, proportionally  to their mutual weight:
$$P(\bx_i \to \bx_{j})  = \frac{W_{ij}}{\sum_{j'}W_{ij'}}$$
Let $D$ be the diagonal matrix with the sum of weights $W$ on the
diagonal: $D_{ii} = \sum_{j'}W_{ij'}$ for all $i$. The transition matrix of
this random walk is $P=D^{-1}W$. The approximation we use here is that we estimate the class conditional probability with
the proportion of the time that the random walk spends in the evaluated
example \cite{lee2010spectral}. We can calculate this proportion from the
stationary distribution of this random walk \cite{chung1997spectral}.  Let $s$
be a row vector of the stationary distribution of a random walk with the transition matrix $P$.
For a stationary distribution $s$, it has to hold that $sP = s$. 
Note
that $\mathbf{1}D =
\mathbf{1}W$,  where ${\mathbf 1}$ is all one row vector. It is easy to verify
that
\begin{equation}[Closed form solution for a stationary distribution of the random walk.]
s = \frac{{\mathbf 1}W}{\mathrm{vol}(W)}
\label{eq:rw_closed}
\end{equation}
\noindent satisfies the definition: 
\begin{align*}
s P  &=  \frac{{\mathbf 1}WP}{\mathrm{vol}(W)}= \frac{{\mathbf
1}DP}{\mathrm{vol}(W)} = \frac{{\mathbf 1}DD^{-1}W}{\mathrm{vol}(W)} =
\frac{{\mathbf 1}W}{\mathrm{vol}(W)}= s
\end{align*}
The equation \eqref{eq:rw_closed} enables us to compute the
stationary distribution in a closed form.

%
%
%
%

\section{Regularized Harmonic Function}
\label{sec:reg HFS}
In this section, we build on the harmonic solution
\cite{zhu2003semi-supervised}. Moreover, we show how to regularize it
such that it can interpolate between semi-supervised learning (SSL) on labeled
examples and SSL on all data.
A standard approach to SSL on graphs is to minimize
the quadratic objective function
\begin{align}
  \min_{\bell \in \realset^n} & \quad
  \bell\transpose L\bell \label{eq:HFS} \\
  \textrm{s.t.} & \quad
  \ell_i = y_i \textrm{ for all } i \in l; \nonumber
\end{align}
where $\bell$ denotes the vector of predictions. 
Using the notation from Section~\ref{sec:GraphAsADataModel}, this problem has a closed-form 
solution:
\begin{align*}
  \bell_u = (D_{uu} - W_{uu})^{-1} W_{ul} \bell_l,
\end{align*}

\noindent which satisfies the \emph{harmonic property} $\ell_i = \frac{1}{d_i}
\sum_{j \sim i} w_{ij} \ell_j$ ({$i \sim j$ denotes that $i$ neigbors~$j$}), and therefore is commonly known as the
\emph{harmonic solution}.

\noindent Since the solution can be also computed as:
\begin{align*}
  \bell_u = (I - P_{uu})^{-1} P_{ul} \bell_l,
\end{align*}
it can be viewed as a product of a random walk on the graph $W$ with the transition matrix $P = D^{-1} W$. The probability of moving between two arbitrary vertices $i$ and $j$ is $w_{ij} / d_i$, and the walk terminates when the reached vertex is labeled. 
Therefore, the harmonic solution is a form of \emph{label propagation} on the data similarity graph.
Each element of the solution is given by:
\begin{align*}
  \ell_i
  \ = & \ \ (I - P_{uu})_{iu}^{-1} P_{ul} \bell_l \nonumber \\
  \ = & \ \
  \underbrace{\sum_{j: y_j = 1} (I - P_{uu})_{iu}^{-1}
  P_{uj}}_{p_i^{(+1)}} -
  \underbrace{\sum_{j: y_j = -1} (I - P_{uu})_{iu}^{-1}
  P_{uj}}_{p_i^{(-1)}} \nonumber \\
  \ = & \ \ p_i^{(+1)} - p_i^{(-1)},
\end{align*}
where $p_i^{+1}$ and $p_i^{-1}$ are probabilities by which the walk starting from the vertex $i$ ends at vertices with labels $+1$ and $-1$, respectively. Therefore, when $\ell_i$ is rewritten as $\abs{\ell_i} \sgn(\ell_i)$, $\abs{\ell_i}$ can be interpreted as a \emph{confidence} in assigning the label $\sgn(\ell_i)$ to the vertex $i$. The maximum value of $\abs{\ell_i}$ is 1, and it is achieved when either $p_i^{+1} = 1$ or $p_i^{-1} = 1$. The closer the confidence $\abs{\ell_i}$ is to 0, the closer the probabilities $p_i^{+1}$ and $p_i^{-1}$ are to 0.5, and the more \emph{uncertain} the label $\sgn(\ell_i)$ is.

We propose to control the confidence of labeling by regularizing the Laplacian $L$ as $L + \gamma_g I$, where $\gamma_g$ is a
scalar and $I$ is the identity matrix. Similarly to 
(\ref{eq:HFS}), the corresponding problem
\begin{align}
  \min_{\bell \in \realset^n} & \quad
  \bell\transpose (L + \gamma_g I) \bell \label{eq:reg HFS} \\
  \textrm{s.t.} & \quad
  \ell_i = y_i \textrm{ for all } i \in l; \nonumber
\end{align}
can be computed in a closed form
\begin{align}
  \bell_u = (L_{uu} + \gamma_g I)^{-1} W_{ul} \bell_l.
  \label{eq:closed-form reg HFS}
\end{align}
and we will refer to it as \emph{regularized} HS.
It can be also interpreted as a random walk on the graph $W$ with an extra
sink. At every step, a walk at node $x_i$ may terminate at the sink with probability
$\gamma_g / (d_i + \gamma_g)$ where $d_i$ is the degree of the current node 
in the walk . Therefore, the scalar $\gamma_g$
essentially controls how the `confidence' $\abs{\ell_i}$ of labeling
unlabeled vertices decreases with the number of hops from labeled
vertices. The proposed regularization will essentially drive the confidence of distant vertices to zero.

\subsection{Soft Harmonic Solution}
\label{sec:SoftHarmonicSolution}

A related problem to \eqref{eq:HFS} is when the constraints representing the fit to the data are enforced in a soft manner
\cite{cortes2008stability}.  In such a case, we are able to bound the generalization 
error of the solution (Section~\ref{sec:SoftHarmonicSolutionCAD}). 
Moreover, the soft harmonic solution can be used for label propagation in case of noise labels (Section~\ref{sec:ConditionalAnomalyDetectionWithSoftHarmonicFunctions}).
One way of enforcing the fit constraints in a soft manner is by solving a following problem:

\begin{equation}[Soft harmonic solution]
\label{eq:soft HFS}
  \bell^{\star} = \min_{\bell \in \realset^n}  (\bell - \by)\transpose C (\bell - \by) + \bell\transpose K \bell,
\end{equation}

\noindent where $K = L + \gamma_g I$ is the regularized Laplacian of the similarity graph, $C$ is a diagonal
matrix such that $C_{ii} \! = \! c_l$ for all labeled examples, and
$C_{ii} = c_u$ otherwise, and $\by$ is a vector of pseudo-targets such
that $y_i$ is the label of the $i$-th example when the example is labeled,
and $y_i = 0$ otherwise. The appealing property of \eqref{eq:soft HFS} is that its solution can be computed in closed form as follows~\cite{cortes2008stability}:

\begin{equation}[Closed form for the soft harmonic solution]
\bell^\star  = (C^{-1}K+I)^{-1}\by
\label{eq:closed form for soft hfs}
\end{equation}

\noindent We will use soft harmonic solution \eqref{eq:soft HFS} particularly 
in the theoretical analysis in Chapter~\ref{sec:TheoreticalAnalysis}.

Several examples of how $\gamma_g$ affects the regularized solution are shown in Figure \ref{fig:HFS}.
Figure \ref{fig:HFS}\textbf{a} shows an example of a simple data adjacency graph. The vertices of the graph are depicted as dots. 
The bigger dots in the middle are labeled vertices. The edges of the graph are shown as dotted lines and weighted as $w_{ij} = \exp[- \normw{\bx_i - \bx_j}{2}^2 / 2]$. Figure \ref{fig:HFS}\textbf{b.} shows three regularized harmonic function solutions on the data adjacency graph from Figure \ref{fig:HFS}\textbf{a}. The plots are cubic interpolations of the solutions. The dark (or pink and blue) colors denote parts of the feature space $\bx$ where $\ell_i > 0$ and $\ell_i < 0$, respectively. The light (or yellow) color marks regions where the confidence $\abs{\ell_i}$ is less than 0.05. When $\gamma_g = 0$, the solution turns into the ordinary harmonic function solution. When $\gamma_g \! = \! \infty$, the confidence of labeling unlabeled vertices \mbox{decreases to zero.} Finally, note that our regularization corresponds to increasing all eigenvalues of the Laplacian $L$ by $\gamma_g$ \cite{smola2003kernels}. In Section  \ref{sec:TheoryMaxMarginGraphCuts} we use this property to bound the generalization error of our solutions.

\begin{figure}
  \centering
	\includegraphics[width=\columnwidth]{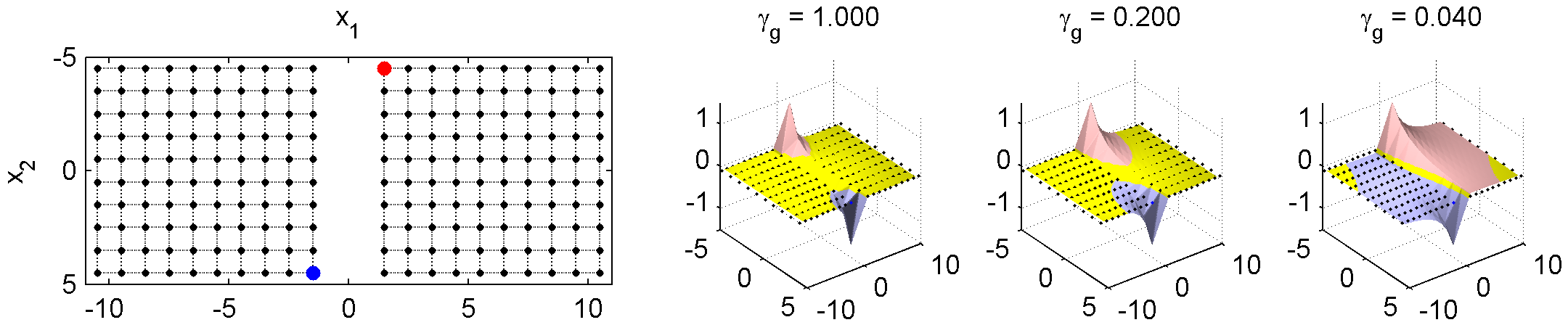}
  \vspace{0.1in} \\
  \hspace{-0.6in} \textbf{(a)} \hspace{2.55in} \textbf{(b)}
  \caption{\textbf{a.} Similarity graph \textbf{b.} Three regularized harmonic solutions}
  \label{fig:HFS}
\end{figure}

%% file: chapter_cad.tex
\chapter{Conditonal Anomaly Detection}
\label{sec:ConditonalAnomalyDetection}

\section{Introduction to Conditonal Anomaly Detection}
\label{sec:CADs}

Anomaly detection is the task of finding unusual elements in a set of observations. 
Most existing anomaly detection methods in data analysis are unconditional and look for outliers with respect to all data attributes
\cite{breunig2000lof:,akoglu2010oddball:,markou2003novelty,markou2003noveltya,chandola2009anomaly}. \emph{Conditional anomaly detection }(CAD) \cite{chandola2009anomaly} is the problem of  detecting unusual values for a subset of variables given the values of the remaining variables. In other words, one set of variables defines the context in which the other set is examined for anomalous values. 

CAD can be extremely useful for detecting unusual behaviors, outcomes, or unusual attribute pairings in many domains \cite{das2008anomaly}.  
Examples of such problems are the detection of unusual actions or outcomes in medicine \cite{hauskrecht2007evidence-based}, investments \cite{rubin2005auctioning}, law \cite{aktolga2010detecting}, social networks \cite{heard2010bayesian}, 
politics \cite{kolar2010estimating}, and other fields \cite{das2008anomaly}. In all these domains, the outcome strongly depends on the context (patient conditions, economy and market, case circumstances, etc.), hence the outcome is unusual only if it is compared to the examples with the same context. 

In this work, we study a special case of CAD that tries to identify the unusual values for just one target variable given the values of the remaining variables (attributes). The target variable is assumed to take on a finite set of values, which we also refer to as labels, because of its similarity to the classification problems. Therefore, we refer to conditional anomalies as mislabelings
\cite{valizadegan2007kernel} or cross-outliers \cite{papadimitriou2003cross-outlier}.
Our objective is to develop robust conditional anomaly methods that work well 
for high-dimensional datasets and let us capture various non-linearities in the underlying space.
This work is motivated primarily by clinical and biomedical datasets and applications. These datasets are highly heterogeneous, and may include hundreds of lab results of different nature, medications, and procedures performed during hospital stay.
In general, the distributions are multi-modal, reflecting many different patients' conditions \cite{hauskrecht2010conditional}.

\section{Definition of Conditional Anomaly}
\label{sec:DefinitionsOfConditionalAnomalyDetection}

In general, the concept of (conditional) anomaly in data in the existing literature is somewhat ambiguous and several definitions have been proposed in the past \cite{markou2003novelty,markou2003noveltya}. Typically, an example is considered anomalous when it is not expected from some underlying model.
A number of anomaly detection methods have been developed for this purpose (Section \ref{sec:UncoditionalAnomalyDetection}).
The conditional anomaly detection (CAD) problem (Section \ref{sec:ConditionalAnomalyDetection})
is different, but equally useful in practice. It seeks to detect  unusual values for a subset of variables $\mathcal{Y}$ given the values for the 
remaining variables $\mathcal{X}$.  Since in this dissertation we focus on CAD in one variable, we provide the definition for this case only.

Intuitively, we can define a conditional anomaly as follows: Given a set of $n$ past observed examples $(\bx_i,y_i)_{i=1}^{n}$ (with possible label noise), 
a \emph{conditional anomaly} is any instance $i$ among recent $m$ examples $(\bx_i,y_i)_{i=n+1}^{n+m}$  for which $y_i$ is unusual.
In this statement, we assume that the past observed examples $(\bx_i,y_i)_{i=1}^{n}$ are given. 
We do not assume that their labels are perfect; they may also be subject to the label noise. 

Let us motivate a formal definition of conditional anomaly by assuming that the $y_i$
is a continuous variable and has a standard normal distribution:

	\[
	   y_i|{\bf x_i}\ \sim\ \mathcal{N}(0,\,1). \,	
\]

\noindent As the standard normal distribution is a unimodal distribution with zero mean, the most anomalous values are the ones with the largest absolute value. 
Assuming a random sample of the size $n$, $Y^{(n)} = {y_1, y_2, \dots y_n}$, the extreme values for this distribution 
correspond to the first and the $n$-th order statistic. The expected $n$-th order statistic
for the standard normal can be approximated as $n \to \infty$ as \cite{crame1999mathematical}:

\begin{equation}[Expected value of the n-th order statistic for the standard normal]
Y^{(n)}_{(n)} \approx \sqrt{2 \ln n}
\label{eq:nth_n01}
\end{equation}

\noindent Therefore, the more samples we have, the larger extreme values we are likely to see
and the less we should be surprised by them.
This motivates our definition, which assumes some probabilistic model of data 
(not necessarily normal) and depends on the sample size $n$:

\begin{definition}
Given any probabilistic model $P$ and a random sample $(\bx_i,y_i)_{i=1}^{n}$, a {\bf conditional anomaly} of the $c$-level in the value $y_i$ given ${\bf x}_i$ is any instance $i$, such that $P(y_i|\bx_i) = O(e^{-cn})$.
\label{def:cad}
\end{definition}

It is not common that we would have access to such a model or that we 
would be able to estimate the class conditional probabilities reliably (especially in high dimensions).
Therefore, in practice we may need to assess the anomalies otherwise (e.g.,\,using human experts).

Not knowing the underlying model, which generates the (attribute, label) pairs, may lead to two major complications  illustrated in Figure~\ref{fig:fringe}.
First, a given instance may be far from the past observed data points (e.g. patient cases). Because of the lack of the support for alternative responses, it is difficult to assess the anomalousness of these instances.  We refer to these instances as \emph{isolated points}. Second, the examples on the boundary of the class distribution support may look anomalous due to their low likelihood. These boundary examples are also known as \emph{fringe points} \cite{papadimitriou2003cross-outlier}.
We aim to avoid both of those when we look for conditional anomalies. 

\section{Relationship to Mislabeling Detection}
\label{sec:RelationshipToMislabelingDetection}

The work on CAD,  when the target variable is restricted to discrete
values only,  is closely related to the problem of mislabeling detection \cite{brodley1999identifying}. The objective in this 
line of work is to 
1) to make a yes/no decision on whether the examples are mislabeled, and 2) to improve the classification accuracy by removing the mislabeled examples. \cite{Jiang-2004-Mislabeled} use an ensemble of
neural nets to remove suspicious samples to create a $k$-NN classifier. \cite{Sanchez-2003-Mislabeled} introduce several $k$-NN based approaches including \textit{Depuration}, \textit{Nearest Centroid Neighborhood} (NCN), and \textit{Iterative} $k$-NCN.
\cite{brodley1999identifying} tried different approaches to remove the mislabeled samples, including single and
ensemble classifiers. Bagging and boosting are applied in~\cite{Verbaeten-2003-MisLabeled} to detect and remove mislabeled examples.  \cite{valizadegan2007kernel} introduce an objective function based on the weighted $k$-NN approach to identify the mislabeled examples and solve it with the Newton method. 

While the objective of mislabeling detection research is to improve the classification 
accuracy by removing or correcting mislabeled examples, the objective of CAD is different:   
CAD is interested in ranking examples according to the severity of conditional anomalies in the data.
This is the main reason our evaluations of CAD in Chapter~\ref{sec:Experiments} measure the rankings of the cases being anomalous, not the improved classification accuracy when we remove them. Nevertheless, we do compare (Section~\ref{sec:EvaluationOfAnomalyDetection}) to the methods typically used in mislabeling detection.

\smallskip

There are various solutions to implement the conditional anomaly detection. We continue by outlining two
baseline approaches.   

\section{Class-outlier Approach}
\label{sec:ClassOutlierApproach}

The simplest approach is to use one of
the unconditional anomaly detection methods: For every possible class
value $y$ we learn a separate anomaly detection model $M_y$ using only the
values of $\bx$ attributes in the data. An example  $(\bx_i,y_i)$ is anomalous if $\bx_i$ is anomalous in $M_{y=y_i}$. We refer to this approach as
the {\em class-outlier} approach. The anomaly detection model $M_y$ can be
implemented with any unconditional anomaly detection model, such as the one-class SVM
\cite{scholkopf1999estimating}, local outlier factor \cite{breunig2000lof:} and
many others \cite{chandola2009anomaly,markou2003novelty,markou2003noveltya}.


The class-outlier approach comes with some
limitations. Such an approach detects anomalies with respect to its class label
and ignores the examples from the other classes. This  may not work well
for those examples for which $\bx$ is away from all of the classes and hence $\bx$ is an anomaly itself. To illustrate this, let us
assume we have two classes ($-1$ and $+1$) and an example $(\bx, -1)$, such
that $\bx$ is an anomaly in $M_{y=-1}$. The class-outlier approach compares this
example to all examples with the same label ($-1$) and declares it to be an
anomaly. However, the problem is when $\bx$ is also an anomaly with respect to
$M_{y=+1}$.  In such a case it is unclear whether $y$ should be $-1$ or $+1$
and hence the conclusion stating that $(\bx, -1)$ is a conditional anomaly may be incorrect.  

The other problem with class-outlier approach is that those methods often declare \emph{fringe} points (Figure~\ref{fig:fringe}) as
anomalies. Fringe points \cite{papadimitriou2003cross-outlier} are points on the outer boundary of a distribution support 
for a specific class. 


	   \begin{figure}
	   \begin{center}
		\includegraphics[width=\columnwidth]{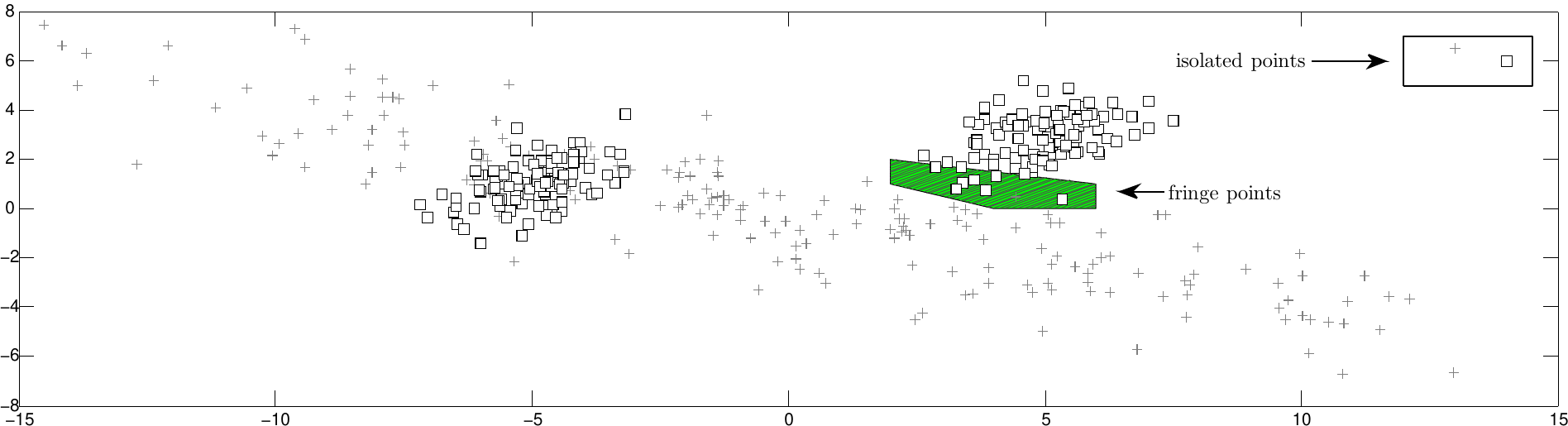}
	   \caption{Challenges for CAD: 1) {\bf fringe} and 2) {\bf isolated} points}%
		\label{fig:fringe}%
		\end{center}
		\end{figure}
			
\section{Discriminative Approach}
\label{sec:DiscriminativeApproach}

Another approach to detect conditional
anomalies is to estimate the posterior $P(y_i|\bx_i)$ for the observed example
$(\bx_i,y_i)$  and to use the posterior to measure how anomalous the data example
is \cite{song2007conditional,hauskrecht2007evidence-based,hauskrecht2010conditional,valko2008conditional}.
According to Definition~\ref{def:cad}, an example is conditionally anomalous
if the probability of the opposite label for this example is high.
Various classification machine learning models can be used to estimate the
posterior from the past data.  For example, one can use the logistic regression
model or generative probabilistic models  such as probabilistic graphical models
that come with an immediate probabilistic interpretation.  However, the output of
other classification models, such as SVM, can be modified and transformed to
produce a probabilistic output. For example, for the non-parametric Parzen window,  the
posterior probability can be estimated by summing the kernel weights for all 
examples with the same class label and by normalizing it with the sum of weights for all examples.
\begin{equation*}
P(y = y_i|\bx_i) = \frac{\sum_{y_i=y_j} K(\bx_j,\bx_i)}{\sum_{j} K(\bx_j,\bx_i)}
\end{equation*}
We will assume $y \in \{-1,+1\}$ from now on, but the
generalization to the multi-class case is straightforward.  
Without loss of generality, we assume that the testing example $\bx_i$ has $y_i
= +1$.  We want to compute $P(y_i \neq +1 | \bx_i)$ to see whether this quantity is not too high,
which would mean that $y_i$ is conditionally anomalous given $\bx_i$. Using Bayes
theorem we get:
\begin{equation}[Posterior probability estimation from a random walk]	
P(y \neq +1| \bx_i) =
\frac{P(\bx_i|y = -1 )P(y = -1)}
{\sum_{c\in \{-1,+1\}} P(\bx_i|y =c)P(y=c)}
\label{eq:cad_with_rw}
\end{equation}
Since we model both prior and class-conditional density, this is a generative model.
In the following we present a new discriminative method that uses random walks on the 
data similarity graph. We then modify it to address the problem of isolated and fringe points.

\subsection{CAD with Random Walks}
\label{sec:ConnectivityAD}

The following method is an example of the discriminative approach.
Let $(\bx_i,y_i)$ be the new example that we want to evaluate and
$P(\bx_i|y=+1)$ and $P(\bx_i|y=-1)$ be the probabilities we want to compute for \eqref{eq:cad_with_rw}.
In this part we show how we can estimate $P(\bx_i|y)$ from the similarity
graphs constructed separately for each class. A similarity graph for a set
of examples is built by assigning each example to a node in the graph.
The edges between the nodes and their weights represent the similarities between
the examples.

\begin{figure}
\begin{center}
\includegraphics[width=0.4\columnwidth]{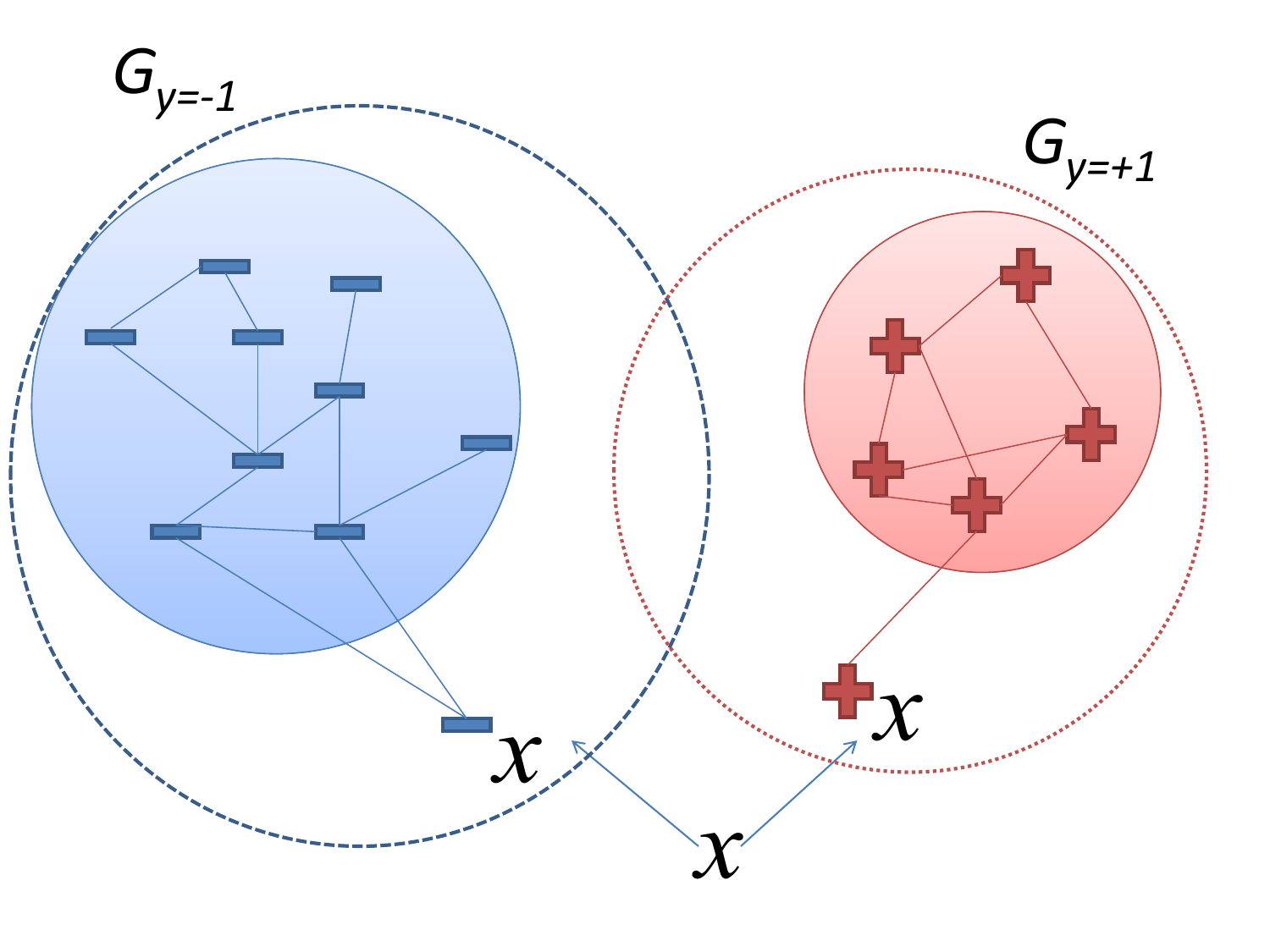}
\caption{Estimating class-conditional probabilities from two similarity graphs}%
\label{fig:2graphs}%
\end{center}
\end{figure}

To explain our method, let us consider (again, without the loss of generality)
the problem of estimating  $P(\bx_i|y=+1)$.  First, we take all $\bx_i$ from
the training set such that $y_i = +1$ and form a
similarity graph using these examples. We then add the new example $\bx_i$
pretending that its label is $y=+1$. 
Let $G_{y=+1}$ be the graph we get (Figure \ref{fig:2graphs}). In the following
we describe how we can use the stationary distribution of a random walk on
$G_{y=+1}$ and use the local connectivity 
as an \emph{approximation} for a density estimate \cite{lee2010spectral} for
$P(\bx_i|y=+1)$. We do the same for $P(\bx_i|y=-1)$ and plug the both estimates
into (\ref{eq:cad_with_rw}) to get an estimate for $P(y \ne +1 | \bx)$.

The equation \eqref{eq:rw_closed} enables us to compute the
stationary distribution in a closed form and ultimately allows us to compute
\eqref{eq:cad_with_rw} efficiently. Once we have the
stationary distribution $s$ of the random walk on $G_{y=+1}$ we approximate $P(\bx_i
| y = +1)$ with $s_i$.

\section{Regularized Discriminative Approach}
\label{sec:Regularization}

We now describe how to avoid detecting the fringe and isolated points using
regularization. Again, our approach considers both classes and $y$ becomes an anomaly if its posterior
probability given $\bx$ is small.
We stress again that in this work we are not interested in
isolated or fringe points.
Let us consider the case of isolated points. 
Imagine the scenario that we get such an anomaly $(\bx_a,y_a)$.
If we take the approach we just described, $\bx_a$ will be far from $G_{y=c}$
for all $c$. Intuitively, the posterior (\ref{eq:cad_with_rw})  compares the
weighted likelihoods of $\bx_a$ given the class, where weight is the class
prior. If these likelihoods are estimated from the training data (and possibly from a small sample size), the estimates of $P(\bx_a|y=-1)$ and $P(\bx_a|y=+1)$
may become unreliable. Consequently, the relative difference between these likelihoods
can strongly favor one class. Our model would then become overly confident in that $\bx_a$ belongs to that class.
We illustrate this behavior in Section~\ref{sec:CADOnCoreDatasetWithFringePoints}.

To alleviate these problems, we propose a new discriminative approach that
penalizes instances of $\bx$ that are anomalies themselves. We do it by
regularizing the model as follows:
\begin{equation}[Regularization of the discriminative approach for CAD]
P(y \neq +1|\bx) =
\frac{P(\bx|y = -1)P(y = -1)}
{\lambda + {\sum_{c\in \{-1,+1\}} P(\bx|y =c)P(y=c)}}
\label{eq:cad_with_rw_reg}
\end{equation}

\noindent Intuitively,  $\lambda$ is a placeholder for  the `everything else' class.
We point out that this is different from the Laplace correction, which is used to
smooth out probability estimates derived from the empirical counts\footnote{$P(y
= k) = (N_k + \lambda) / (\sum_k N_k + K \lambda)$, where $K$  is the number of classes and $N_k$ are the corresponding counts.}. First, 
this regularization is applied directly to Bayes theorem and not to a
probability estimate. Second, this regularization only changes the
denominator of the Bayes theorem and effectively creates the aforementioned
`everything else' class. The $\lambda$ is data dependent and can be set for example by cross-validation.

\subsection{Regularized Random Walk CAD}
\label{sec:ConnectivityGAD}

We will refer to the recently proposed algorithm as the $\lambda$-regularized random walk
CAD algorithm ($\lambda$-RWCAD). 
Algorithm~\ref{alg:reggad} displays the pseudo-code of the $\lambda$-RWCAD algorithm.
 Notice
that $\mathrm{vol}(W^{+})$ and $\mathrm{vol}(W^{-})$ are constants and can be
precomputed.
One of the benefits of the $\lambda$-RWCAD  algorithm is that it
does not require us to store the whole $n \times n$ similarity matrix.
Moreover, the method requires only a nearest neighbor type of a computation,
and therefore it has $O(n^2)$ time and $O(n)$ space
requirements. For sparse representations of the graph, the time is reduced to
$O(|E|)$, where $|E|$ is the number of edges in the graph. 
On the other hand, many other graph-based algorithms require quadratic space, and their
time complexity is related to the computation of the inverse of $n\times n$
matrix which is $\Omega(n^2)$ and $O(n^{2.807})$ in most practical implementations\footnote{The complexity
can be improved to $O(n^{2.376})$ by using the Coppersmith-Winograd algorithm.}.
Finally, modeling the data distribution with a graph can be extended to online
learning \cite{kivinen2002online}. Unlike the label propagation methods that
require us to store the whole $O(n^2)$ weight matrix
($O(|E|)$ when it is sparse) for the future computations, our method
requires only a summary statistic for each vertex, which is $O(n)$.

\begin{algorithm}
\small{
\begin{tabbing}
\hspace{0.1in} \= \hspace{0.1in} \= \hspace{0.1in} \= \hspace{0.1in} \= \kill
{\bf Inputs:} \\
\> new example $(\bx_e,y_e)$ \\
\> similarity metric $K(\cdot,\cdot)$ \\
\> $\mathrm{vol}(W^{+}) = \sum_{y_i = y_j = +1}W_{ij}$ \\  
\> $\mathrm{vol}(W^{-}) = \sum_{y_i = y_j = -1}W_{ij}$ \\
\> regularization coefficient $\lambda$	\\ {\bf Algorithm:} \\
\> $W^{+}_{i\bx_e} = K(\bx_i,\bx_e), \forall i$ positive \\
\> $W^{-}_{i\bx_e} = K(\bx_i,\bx_e), \forall i$ negative \\
\> $P(\bx_e | y = +1)= \sum_i W^{+}_{i\bx_e} / 
\left(\mathrm{vol}\left(W^{+}\right) + 2\times \sum_i W^{+}_{i\bx_e}\right) $ \\ 
\> $P(\bx_e | y = -1)= \sum_i W^{-}_{i\bx_e} / \left(
\mathrm{vol}\left(W^{-}\right) + 2\times \sum_i W^{-}_{i\bx_e}\right) $ \\ 
\> $ P(y \neq y_e|\bx_e) = \frac{P(\bx_e|y \neq y_e)P(y \neq y_e)} {\lambda + {\sum_{c\in \{-1,+1\}} P(\bx_e|y =c)P(y=c)}}$ \\
{\bf Outputs:} \\
\>  $P(y \ne y_e|\bx_e)$
\end{tabbing}
}
\caption{RWCAD that calculates the anomaly score}
\label{alg:reggad}
\end{algorithm}

\subsection{Conditional Anomaly Detection with Soft Harmonic Functions}
\label{sec:ConditionalAnomalyDetectionWithSoftHarmonicFunctions}

In this section we show how to solve the CAD problem  using label propagation on the data similarity graph
and how to compute the anomaly score.
In particular, we will build on the harmonic solution approach (Section~\ref{sec:SoftHarmonicSolution})
and adopt it for CAD in the following ways:
	1) show how to compute the confidence of mislabeling, %
	2) add a regularizer to address the problem of isolated and fringe points, %
	3) use soft constraints to account for a fully labeled setting, and %
	4) describe a compact computation of the solution from a quantized backbone graph. %

The label propagation method described in Section~\ref{sec:SoftHarmonicSolution}
can be applied to CAD by considering all observed data as labeled examples with no unlabeled examples. The setting for matrix $C$ is dependent on the quality of the past observed data. If the labels of the past observed data (or any example from the recent sample) are guaranteed to be correct, we set the corresponding diagonal elements of $C$ to a large value to make their labels fixed. Notice that specific domain techniques can be utilized to make sure that the collected examples from the past observed data have correct labels. We assume that we do not have the access to such prior knowledge and therefore, the observed data are also subject to label noise.

We now propose a way to compute the anomaly score from~\eqref{eq:closed form for soft hfs}.
The output $\bell^\star$ of  \eqref{eq:soft HFS} for the example $i$ can be rewritten as:
\begin{equation}[Confidence from the soft label]
\ell_i^\star = |\ell_i^\star| \times \mathrm{sgn}(\ell_i^\star)
\label{eq:absform}
\end{equation}
\noindent SSL methods use $\mathrm{sgn}(\ell^\star_i)$ in  \eqref{eq:absform}
as the predicted label for $i$. For an unlabeled example, when the value of $\ell_i$ is close to $\pm 1$,
then the labeling information that was propagated to it is more consistent.
Typically, that means that the example is close to the labeled examples of the respective class.
The key observation, which we exploit here, is
that we can interpret $|\ell^\star_i|$ as the confidence in the label.
Our situation differs from SSL, as all our examples are labeled and we aim to assess the confidence of \emph{already labeled} example.
Therefore, we define the \emph{anomaly score} as the
absolute difference between the actual label $y_i$ and the inferred soft label $\ell_i$:

\begin{equation}[Anomaly score for soft harmonic anomaly detection]
s_i = |\ell_i^\star - y_i|.
\label{eq:anomalyscore}
\end{equation}

We will now address the problems illustrated in Figure~\ref{fig:fringe}.
Recall that the isolated points are the examples
that are (with respect to some metric) far from the majority of the data.
Consequently, they are surrounded by few or no nearby points.
Therefore, no matter what their label is, we do not want to report them as conditional anomalies. In other words,
we want CAD methods to assign them a low anomaly score.
Even when the isolated points are far from the majority data, they  still can be orders
of magnitudes closer to the data points with the opposite label.
This can make a label propagation approach falsely confident about that example being a conditional anomaly.
In the same way, we do not want to assign a high anomaly score  to fringe points just because they lie on the distribution boundary.
To tackle these problems we set $K = L + \gamma_gI$, where we diagonally regularize the graph Laplacian.
Intuitively, such a regularization lowers the confidence value $|\bell^\star|$ of all examples; however it reduces the confidence score of far outlier points relatively more. To see this, notice (Section~\ref{sec:Parameters}) that
the similarity weight metric is an exponentially decreasing function of the Euclidean distance.
In other words, such a regularization can be interpreted as a label propagation on the graph with an extra
sink. The sink is an extra node in $G$ with label $0$ and every other node connected to it
with the same small weight $\gamma_g$. The edge weight of $\gamma_g$ affects the isolated points more than other points, because their connections to other nodes are small.


In the fully labeled setting, the \textit{hard} harmonic solution degenerates to the weighted $k$-NN.
In particular, the hard constraints of the harmonic solution do not allow the labels to spread beyond other labeled examples.
However, despite the fully labeled case, we still want to take the advantage of the manifold structure.
To alleviate this problem we allow labels to spread on the graph by
using soft constraints in the unconstrained regularization problem \eqref{eq:soft HFS}.
In particular, instead of $c_l=\infty$ we set $c_l$ to a finite constant and we set $C = c_l I$.
With such a setting of $K$ and $C$, we can solve \eqref{eq:soft HFS}
using \eqref{eq:closed form for soft hfs} to get:

\begin{equation}[Soft harmonic solution using matrix inversion]
\bell^\star  = \left(\left(c_l I\right)^{-1}\left(L+\gamma_g\right) + I\right)^{-1}\by \
       = \left(c_l^{-1}L+\left(1+\frac{\gamma_g}{c_l}\right)I\right)^{-1}\by.
			\label{eq:had}
\end{equation}

\noindent To avoid computation of the inverse,\footnote{due to numerical instability} we calculate \eqref{eq:had} using the following system of linear equations:

\begin{equation}[Soft harmonic solution using system of linear equations]
\left(c_l^{-1}L+\left(1+\frac{\gamma_g}{c_l}\right)I\right) \bell^\star = \by
\label{eq:had:sle}
\end{equation}

\noindent
We then plug the output of \eqref{eq:had:sle} into \eqref{eq:anomalyscore} to get the anomaly score.
We will refer to this score as the SoftHAD score.
Intuitively, when the confidence is high but $\mathrm{sign}(\ell_i^\star) \ne y_i$, we will consider the label $y_i$ of the case
$(\bx_i, y_i)$ conditionally anomalous.

\noindent {\bf Backbone graph} The computation of the system of linear equations \eqref{eq:had:sle}
scales with complexity\footnote{The complexity can be further improved to $O(n_u^{2.376})$ with
the Coppersmith-Winograd algorithm.} $O(n^3)$.
This is not feasible for a graph with more than several thousand nodes.
To address the problem, we use \emph{data quantization} \cite{gray1998quantization}
and sample a set of nodes from the training data to create $G$.
We then substitute the nodes in the graph with a smaller set of $k \ll n$ distinct centroids,
which results in $O(k^3)$ computation.

We improve the approximation of the original graph with the backbone graph,
by assigning different weights to the centroids.
We do it by computing the multiplicities (i.e.\,how many nodes each centroid represents).
In the following we will describe how to modify
 \eqref{eq:had:sle} to allow for the computation with multiplicities.

Let $V$ be the diagonal matrix of multiplicities with $v_{ii}$ being
the number of nodes that centroid $\bx_i$ represents.
We will set the multiplicities according to the empirical prior.

Let $W^V$ be the compact representation of the matrix $W$ on $G$,
where each node $\bx_i$ is replicated $v_{ii}$ times.
Let $L^V$ and $K^V$ be the graph Laplacian and regularized
graph Laplacian of $W^V$. Finally, let $C^V$ be the $C$ in \eqref{eq:soft HFS} with the adjustment for the multiplicities. $C^V$ accounts for the fact that we care about `fitting' to train data according to the multiplicities.
Then:

\begin{align*}
W^V &= VWV\\
L^V &= L(W^V)\\
K^V &= L^V + \gamma_gV\\
C^V &= V^{1/2}CV^{1/2}	
\end{align*}

\noindent
\noindent The unconstrained regularization \eqref{eq:soft HFS}
now becomes:

\begin{equation}[Compact computation of harmonic solution for the backbone graph]
  \bell^{V\star} = \min_{\bell \in \realset^n} \
  (\bell - \by)\transpose C^V (\bell - \by) + \bell\transpose K^V\bell
	\label{eq:unconstrained regularization:mul}
\end{equation}

\noindent and subsequently \eqref{eq:had} becomes:

\begin{eqnarray*}
 \bell^{V\star}  & = & \left(\left(C^V\right)^{-1}K^V+I\right)^{-1}\by \\
      & = & \left( V^{-1/2}C^{-1}V^{-1/2} (L^V + \gamma_gV)  +I\right)^{-1}\by \\
\label{eq:closed form:mul}
      & = & \left(\left(c_lV\right)^{-1}  (L^V + \gamma_gV)  +I\right)^{-1}\by \\
      & = & \left(1/c_l V^{-1}  L^V + c_l\gamma_g  +I\right)^{-1}\by
\end{eqnarray*}

\noindent With these adjustments the anomaly score that accounts
for the multiplicities is equal to $|\bell^{V\star} - \by|$.

%% file: chap_theory.tex
\chapter{Theoretical Analysis}
\label{sec:TheoreticalAnalysis}

In this chapter, we analyze the methods proposed in Chapter~\ref{sec:SemiSupervisedLearning} and Chapter~\ref{sec:ConditonalAnomalyDetection}.
We will analyze mostly:

\begin{itemize}
	\item \textbf{generalization} errors induced by harmonic solutions on the graph,
	\item errors induced by \textbf{quantization} of the graph to accommodate online learning, and
	\item errors due to the \textbf{online} setting.
\end{itemize}

\section{Soft Harmonic Solution}
\label{sec:SoftHarmonicSolutionCAD}

In this section we prove a bound on the generalization error of our transductive learner. 
The generalization error of the solution to the problem \eqref{eq:closed form for soft hfs} (and also \eqref{eq:soft HFS}) is bounded in Lemma \ref{lem:transductive bound}. 

\begin{lemma}
\label{lem:transductive bound} Let $\bell^\ast$ be a solution to the problem:
\begin{align*}
  \min_{\bell \in \realset^n} \
  (\bell - \by)\transpose C (\bell - \by) +
  \bell\transpose Q \bell,
\end{align*}
where $Q = L + \gamma_g I$ and all labeled examples $l$ are selected i.i.d. Then the inequality:
\begin{align*}
  R_P^{\scriptscriptstyle{W}}(\bell^\ast) \ \leq & \ \
  \widehat{R}_P^{\scriptscriptstyle{W}}(\bell^\ast) +
  \underbrace{\beta + \sqrt{\frac{2 \ln(2 / \delta)}{n_l}}
  (n_l \beta + 4)}_{\emph{transductive error }
  \Delta_T(\beta, n_l, \delta)} \\
  \beta \ \leq & \ \
  2 \left[\frac{\sqrt{2}}{\gamma_g + 1} +
  \sqrt{2 n_l} \frac{1 - \sqrt{c_u}}{\sqrt{c_u}}
  \frac{\lambda_M(L) + \gamma_g}{\gamma_g^2 + 1}\right]
\end{align*}
holds with probability $1 - \delta$, where:
\begin{align*}
  R_P^{\scriptscriptstyle{W}}(\bell^\ast)
  \ = & \ \
  \frac{1}{n} \sum_i (\ell_i^\ast - y_i)^2 \\
  \widehat{R}_P^{\scriptscriptstyle{W}}(\bell^\ast)
  \ = & \ \
  \frac{1}{n_l} \sum_{i \in l} (\ell_i^\ast - y_i)^2
\end{align*}
are risk terms for all vertices and labeled vertices, respectively, and $\beta$ is the stability coefficient of the solution $\bell^\ast$.
\end{lemma}
\noindent {\bf Proof:} 
To simplify the proof, we assume that $c_l = 1$ and $c_l > c_u$.
Our risk bound follows from combining Theorem 1 of \cite{belkin2004regularization} with the assumptions $\abs{y_i} \leq 1$ and $\abs{\ell_i^\ast} \leq 1$. The coefficient $\beta$ is derived based on Section 5 of \cite{cortes2008stability}. In particular, based on the properties of the matrix $C$ and Proposition 1 \cite{cortes2008stability}, we conclude:
\begin{align*}
  \beta = 2 \left[\frac{\sqrt{2}}{\lambda_m(Q) + 1} +
  \sqrt{2 n_l} \frac{1 - \sqrt{c_u}}{\sqrt{c_u}}
  \frac{\lambda_M(Q)}{(\lambda_m(Q) + 1)^2}\right],
\end{align*}
where $\lambda_m(Q)$ and $\lambda_M(Q)$ refer to the smallest and largest eigenvalues of $Q$, respectively, and can be further rewritten as $\lambda_m(Q) = \lambda_m(L) + \gamma_g$ and $\lambda_M(Q) = \lambda_M(L) + \gamma_g$. Our final claim directly follows from applying the lower bounds $\lambda_m(L) \geq 0$ and $(\lambda_m(L) + \gamma_g + 1)^2 \geq \gamma_g^2 + 1$. \qed

\bigskip Lemma \ref{lem:transductive bound} is practical when the error $\Delta_T(\beta, n_l, \delta)$ decreases at the rate of $O(n_l^{- \frac{1}{2}})$. This is achieved when $\beta \! = \! O(1 / n_l)$, which corresponds to $\gamma_g \! = \! \Omega(n_l^\frac{3}{2})$. Thus, when the problem (\ref{eq:soft HFS}) is sufficiently regularized, its solution is stable, and the generalization error of the solution is bounded.

\section{Analysis of Max-margin Graph Cuts}
\label{sec:TheoryMaxMarginGraphCuts}

\subsection{When Manifold  Regularization Fails}
\label{sec:MR fails}

The major difference between manifold regularization (\ref{eq:MR}) and the regularized harmonic function solution (\ref{eq:reg HFS}) is in the space of optimized parameters. In particular, manifold regularization is performed on a class of functions $\cH_K$. When this class is severely restricted, such as with linear functions, the minimization of ${\bf f}\transpose L {\bf f}$ may lead to results which are significantly worse than the harmonic function solution.

This issue can be illustrated on the problem from Figure \ref{fig:HFS}, where we learn a linear decision boundary $f(\bx) = \alpha_1 x_1 + \alpha_2 x_2$ through manifold regularization of linear SVMs:
\begin{equation}[Linear manifold regularization]
  \min_{\alpha_1, \alpha_2} \ \sum_{i \in l} V(f, \bx_i, y_i) +
  \gamma [\alpha_1^2 + \alpha_2^2] +
  \gamma_u {\bf f}\transpose L {\bf f}.
  \label{eq:linear MR}
\end{equation}
The structure of our problem simplifies the computation of the regularization term ${\bf f}\transpose L {\bf f}$. In particular, since all edges in the data adjacency graph are either horizontal or vertical, the term ${\bf f}\transpose L {\bf f}$ can be expressed as a function of $\alpha_1^2$ and $\alpha_2^2$. Therefore, for this particular problem we have:
\begin{align}
  {\bf f}\transpose L {\bf f}
  \ = & \ \ \frac{1}{2} \sum_{i, j} w_{ij} (f(\bx_i) - f(\bx_j))^2
  \nonumber \\
  \ = & \ \ \frac{1}{2} \sum_{i, j} w_{ij}
  (\alpha_1 (\bx_{i1} - \bx_{j1}) + \alpha_2 (\bx_{i2} - \bx_{j2}))^2
  \nonumber \\
  \ = & \ \ \frac{\alpha_1^2}{2} \underbrace{\sum_{i, j} w_{ij}
  (\bx_{i1} - \bx_{j1})^2}_{\Delta = 218.351} + \nonumber \\
  & \ \ \frac{\alpha_2^2}{2} \underbrace{\sum_{i, j} w_{ij}
  (\bx_{i2} - \bx_{j2})^2}_{\Delta = 218.351}.
  \label{eq:linear manifold}
\end{align}
After we incorporate \eqref{eq:linear manifold} to our objective function \eqref{eq:linear manifold}, we get \eqref{eq:linear manifold} as an additional weight at the regularizer $[\alpha_1^2 + \alpha_2^2]$:
\begin{equation}[Linear support vector machine]
  \min_{\alpha_1, \alpha_2} \ \sum_{i \in l} V(f, \bx_i, y_i) +
  \left(\gamma + \frac{\gamma_u \Delta}{2}\right)
  [\alpha_1^2 + \alpha_2^2]  = 
	\min_{\alpha_1, \alpha_2} \ \sum_{i \in l} V(f, \bx_i, y_i) +
  \gamma^\ast[\alpha_1^2 + \alpha_2^2],
  \label{eq:linear SVM}
\end{equation}
where $\gamma^\ast = \left(\gamma + \frac{\gamma_u \Delta}{2}\right)$.
Thus, manifold regularization of linear SVMs on our problem can be viewed as supervised learning with linear SVMs with a varying weight at the regularizer. In other words, in this particular problem, the unlabeled examples only influence the solution 
through the regularizer $\gamma^\ast$ on $f(\bx)$. That means we can get the same $f(\bx)$
for a different $\gamma^\ast$ if the unlabeled examples were not present at all. 
Since the problem involves only two labeled examples, changes in the weight $\gamma^\ast$ do not affect the direction of the discriminator $f^\ast(\bx) = 0$, because the margin is maximized by the hyperplane between them.
Therefore, different settings of the regularizer only change the slope of $f^\ast$ (Figure \ref{fig:synthetic cuts}, second row).
The above analysis shows that the discriminator $f^\ast(\bx) = 0$ does not change with $\gamma_u$. As a result, all discriminators are equal to the discriminator for $\gamma_u = 0$, which can be learned by linear SVMs, yet none of them solves our problem optimally. Max-margin graph cuts solve the problem optimally for small values of $\gamma_g$. If we included more unlabeled examples, we could get the error 
arbitrarily large, assuming our problems would consist of two coherent square-shaped classes, as in \ref{fig:HFS}.
Figure \ref{fig:synthetic cuts} shows
linear, cubic, and RBF decision boundaries, obtained by manifold regularization of SVMs (MR) and max-margin graph cuts (GC) on the problem from Figure \ref{fig:HFS}. The regularization parameter $\gamma_g = \gamma / \gamma_u$ is set as suggested in Section \ref{sec:MR}, $\gamma \! = \! 0.1$, and $\eps \! = \! 0.01$. The pink and blue colors denote parts of the feature space $\bx$ where the discriminators $f$ are positive and negative, respectively. The yellow color marks the regions where $\abs{f(\bx)} < 0.05$.

A similar line of reasoning can be used to extend our results to polynomial kernels. Figure \ref{fig:synthetic cuts} indicates that max-margin learning with the cubic kernel exhibits trends similar to the linear case.

\begin{figure}
  \centering
  \includegraphics[width=4.8in, viewport=1.25in 0.75in 7.25in 9.75in]{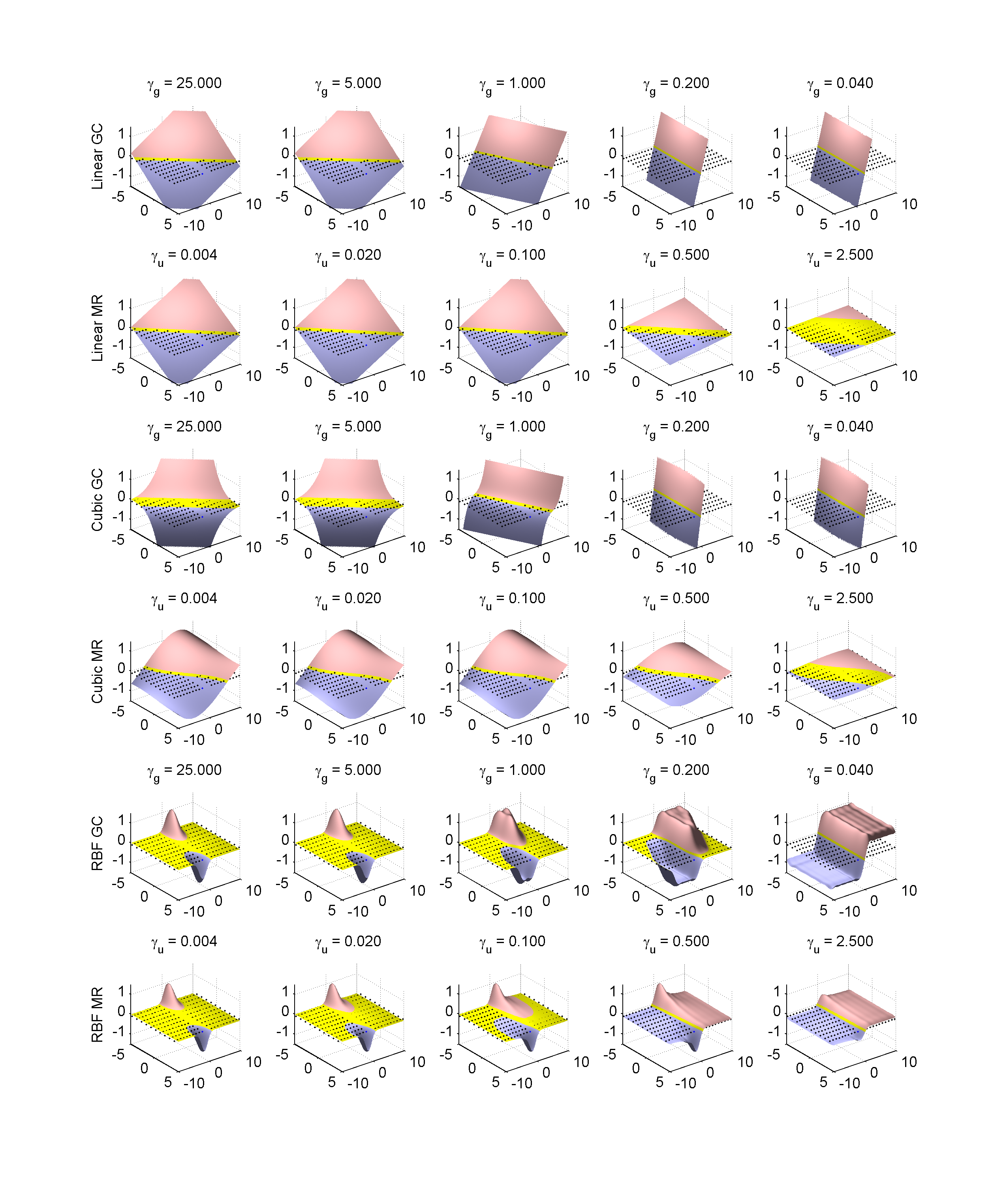}
  \caption{Linear, cubic, and RBF decision boundaries for different methods}
  \label{fig:synthetic cuts}
\end{figure}

The notion of algorithmic stability can be used to bound the generalization error of many learning algorithms \cite{bousquet2002stability}. In this section, we discuss how to make the harmonic function solution stable and prove a bound on the generalization error of max-margin cuts (\ref{eq:MMGC}). Our bound combines existing transductive \cite{belkin2004regularization,cortes2008stability} and inductive \cite{vapnik1995nature} bounds.

\subsection{Generalization Error}
\label{sec:generalization error}

Our objective is to show that the \emph{risk} of our solution $f$:
\begin{equation}[Risk of our solutions]
  R_P(f) = \E{P(\bx)}{\cL(f(\bx), y(\bx))}
  \label{eq:risk}
\end{equation}
is bounded by the \emph{empirical risk} on graph-induced labels:
\begin{equation}[Empirical risk on graph induced labels]
  \frac{1}{n} \sum_i \cL(f(\bx_i), \sgn(\ell_i^\ast))
  \label{eq:empirical risk}
\end{equation}
and error terms, which can be computed from training data. The function $\cL(y', y) \! = \! \I{\sgn(y') \! \neq \! y}$ computes the zero-one loss of the prediction $\sgn(y')$ given the ground truth $y$. $P(\bx)$ is the distribution of our data. For simplicity, we assume that the label $y$ is a deterministic function of $\bx$. Our proof starts by relating $R_P(f)$ and graph-induced labels $\ell_i^\ast$.

\begin{lemma}
\label{lem:inductive bound} Let $f$ be from a function class with the VC dimension $h$ and $\bx_i$ be $n$ examples, which are sampled i.i.d. with respect to the distribution $P(\bx)$. Then the inequality:
\begin{align*}
  R_P(f)
  \ \leq & \ \
  \frac{1}{n} \sum_i \cL(f(\bx_i), \sgn(\ell_i^\ast)) \ + \\
  & \ \
  \frac{1}{n} \sum_i (\ell_i^\ast - y_i)^2 + \\
  & \ \
  \underbrace{\sqrt{\frac{h (\ln(2 n / h) + 1) -
  \ln(\eta / 4)}{n}}}_{\emph{inductive error } \Delta_I(h, n, \eta)}
\end{align*}
holds with the probability of $1 - \eta$, where $y_i$ and $\ell_i^\ast$ represent the true and graph-induced soft labels, respectively.
\end{lemma}
\noindent {\bf Proof:} Based on Equations 3.15 and 3.24 \cite{vapnik1995nature}, the inequality:
\begin{align*}
  R_P(f) \leq \frac{1}{n} \sum_i \cL(f(\bx_i), y_i) +
  \Delta_I(h, n, \eta)
\end{align*}
holds with the probability of $1 - \eta$. Our final claim follows from bounding all terms $\cL(f(\bx_i), y_i)$ as:
\begin{align*}
  \cL(f(\bx_i), y_i) \leq \cL(f(\bx_i), \sgn(\ell_i^\ast)) +
  (\ell_i^\ast - y_i)^2.
\end{align*}
The above bound holds for any $y_i \in \set{-1, 1}$ and $\ell_i^\ast$. \qed

\bigskip It is hard to bound the error term $\frac{1}{n} \sum_i (\ell_i^\ast - y_i)^2$ when the constraints $\ell_i = y_i$ (\ref{eq:reg HFS}) are enforced in a hard manner. Thus, in the rest of our analysis, we consider a relaxed version of the harmonic function solution 
(Section~\ref{sec:SoftHarmonicSolution}).  Lemma \ref{lem:transductive bound} and its proof can be found in Section~\ref{sec:SoftHarmonicSolutionCAD}. Lemmas \ref{lem:transductive bound} and \ref{lem:inductive bound} can be combined using the union bound.

\begin{proposition}
\label{prop:MMGC bound} Let $f$ be from a function class with the VC dimension $h$. Then the inequality:
\begin{align*}
  R_P(f)
  \ \leq & \ \
  \frac{1}{n} \sum_i \cL(f(\bx_i), \sgn(\ell_i^\ast)) \ + \\
  & \ \
  \widehat{R}_P^{\scriptscriptstyle{W}}(\bell^\ast) +
  \Delta_T(\beta, n_l, \delta) + \Delta_I(h, n, \eta)
\end{align*}
holds with probability $1 - (\eta + \delta)$.
\end{proposition}

\bigskip The above result can be viewed as follows. If both $n$ and $n_l$ are large, the sum of $\frac{1}{n} \sum_i \cL(f(\bx_i), \sgn(\ell_i^\ast))$ and $\widehat{R}_P^{\scriptscriptstyle{W}}(\bell^\ast)$ provides a good estimate of the risk $R_P(f)$. Unfortunately, our bound is not practical for setting $\gamma_g$, because it is hard to find a $\gamma_g$ that minimizes both $\widehat{R}_P^{\scriptscriptstyle{W}}(\bell^\ast)$ and $\Delta_T(\beta, n_l, \delta)$. The same phenomenon was observed by \cite{belkin2004regularization} in a similar context. To solve our problem, we suggest 
setting $\gamma_g$ based on the validation set. This methodology is used in the experimental section.

\subsection{Threshold epsilon}
\label{sec:threshold eps}

Finally, note that when $\abs{\ell_i^\ast} < \eps$, where $\eps$ is a small number, $\abs{\ell_i^\ast - y_i}$ is close to 1 irrespective of $y_i$, and a trivial upper bound $\cL(f(\bx_i), y_i) \! \leq \! 1$ is almost as good as $\cL(f(\bx_i), y_i) \! \leq \! \cL(f(\bx_i), \sgn(\ell_i^\ast)) + (\ell_i^\ast - y_i)^2$ for any $f$. This allows us to justify the $\eps$ threshold in the problem (\ref{eq:MMGC}). In particular, note that $\cL(f(\bx_i), y_i)$ is bounded by $1 - (\ell_i^\ast - y_i)^2 + (\ell_i^\ast - y_i)^2$. When $\abs{\ell_i^\ast} < \eps$, $1 - (\ell_i^\ast - y_i)^2 < 2 \eps - \eps^2$, we conclude the following:

\begin{proposition}
\label{prop:MMGC thresholded bound} Let $f$ be from a function class with the VC dimension $h$ and $n_\eps$ be the number of examples such that $\abs{\ell_i^\ast} < \eps$. Then the inequality:
\begin{equation*}
  R_P(f) \leq  
  \frac{1}{n} \!\!\! \sum_{i : \abs{\ell_i^\ast} \geq \eps}
  \!\!\! \cL(f(\bx_i), \sgn(\ell_i^\ast)) +
  \frac{2 \eps n_\eps}{n} + 
  \widehat{R}_P^{\scriptscriptstyle{W}}(\bell^\ast) +
  \Delta_T(\beta, n_l, \delta) + \Delta_I(h, n, \eta)
\end{equation*}
holds with probability $1 - (\eta + \delta)$.
\end{proposition}
\noindent {\bf Proof:} The generalization bound is proved as:
\begin{align*}
  R_P(f)
  \ \leq & \ \ \widehat{R}_P(f) + \Delta_I(h, n, \eta) \\
  \ = & \ \
  \frac{1}{n} \!\!\! \sum_{i : \abs{\ell_i^\ast} \geq \eps}
  \!\!\! \cL(f(\bx_i), y_i) +
  \frac{1}{n} \!\!\! \sum_{i : \abs{\ell_i^\ast} < \eps}
  \!\!\! \cL(f(\bx_i), y_i) \ + \\
  & \ \ \Delta_I(h, n, \eta) \\
  \ \leq & \ \
  \frac{1}{n} \!\!\! \sum_{i : \abs{\ell_i^\ast} \geq \eps}
  \!\!\! \left[\cL(f(\bx_i), \sgn(\ell_i^\ast)) +
  (\ell_i^\ast - y_i)^2\right] + \\
  & \ \
  \frac{1}{n} \!\!\! \sum_{i : \abs{\ell_i^\ast} < \eps}
  \!\!\! \left[1 - (\ell_i^\ast - y_i)^2 +
  (\ell_i^\ast - y_i)^2\right] + \\
  & \ \ \Delta_I(h, n, \eta) \\
  \ = & \ \
  \frac{1}{n} \!\!\! \sum_{i : \abs{\ell_i^\ast} \geq \eps}
  \!\!\! \cL(f(\bx_i), \sgn(\ell_i^\ast)) \ + \\
  & \ \
  \frac{1}{n} \!\!\! \sum_{i : \abs{\ell_i^\ast} < \eps}
  \!\!\! \left[1 - (\ell_i^\ast - y_i)^2\right] +
  \frac{1}{n} \sum_i (\ell_i^\ast - y_i)^2 + \\
  & \ \ \Delta_I(h, n, \eta) \\
  \ \leq & \ \
  \frac{1}{n} \!\!\! \sum_{i : \abs{\ell_i^\ast} \geq \eps}
  \!\!\! \cL(f(\bx_i), \sgn(\ell_i^\ast)) +
  \frac{2 \eps n_\eps}{n} \ + \\
  & \ \
  \widehat{R}_P^{\scriptscriptstyle{W}}(\bell^\ast) +
  \Delta_T(\beta, n_l, \delta) + \Delta_I(h, n, \eta).
\end{align*}
The last step follows from the inequality $1 - (\ell_i^\ast - y_i)^2 < 2 \eps$ and Lemma \ref{lem:transductive bound}. \qed

\bigskip \noindent When $\eps \leq n_l^{- \frac{1}{2}}$, the new upper bound is asymptotically as good as the bound in Proposition \ref{prop:MMGC bound}. As a result, we get the same convergence guarantees, although highly-uncertain labels $\abs{\ell_i^\ast} < \eps$ are excluded from our optimization.

\noindent In practice, optimization of the thresholded objective often yields a lower risk $$\frac{1}{n} \sum_{i : \abs{\ell_i^\ast} \geq \eps} \cL(f^\ast(\bx_i), \sgn(\ell_i^\ast)) + \frac{2 \eps n_\eps}{n},$$ and also lower training and test errors. This is a result of excluding the most uncertain examples \mbox{$\abs{\ell_i^\ast} \! < \! \eps$ from learning.} Figure \ref{fig:thresholded objective} illustrates these trends on three learning problems. In particular it shows the thresholded empirical risk $\frac{1}{n} \sum_{i : \abs{\ell_i^\ast} \geq \eps} \cL(f^\ast(\bx_i), \sgn(\ell_i^\ast)) + \frac{2 \eps n_\eps}{n}$ of the optimal max-margin graph cut $f^\ast$ (\ref{eq:MMGC}), its training and test errors, and the percentage of training examples such that $\abs{\ell_i^\ast} \geq \eps$, on 3 letter recognition problems from the UCI ML repository. The plots are shown as functions of the parameter $\gamma_g$ and correspond to the thresholds $\eps = 0$ (light gray lines), $\eps = 10^{-6}$ (dark gray lines), and $\eps = 10^{-3}$ (black lines). All results are averaged over 50 random choices of 1 percent of labeled examples.

\begin{figure}
  \centering
  \includegraphics[width=4.8in, viewport=1.25in 3.5in 7.25in 7.5in]{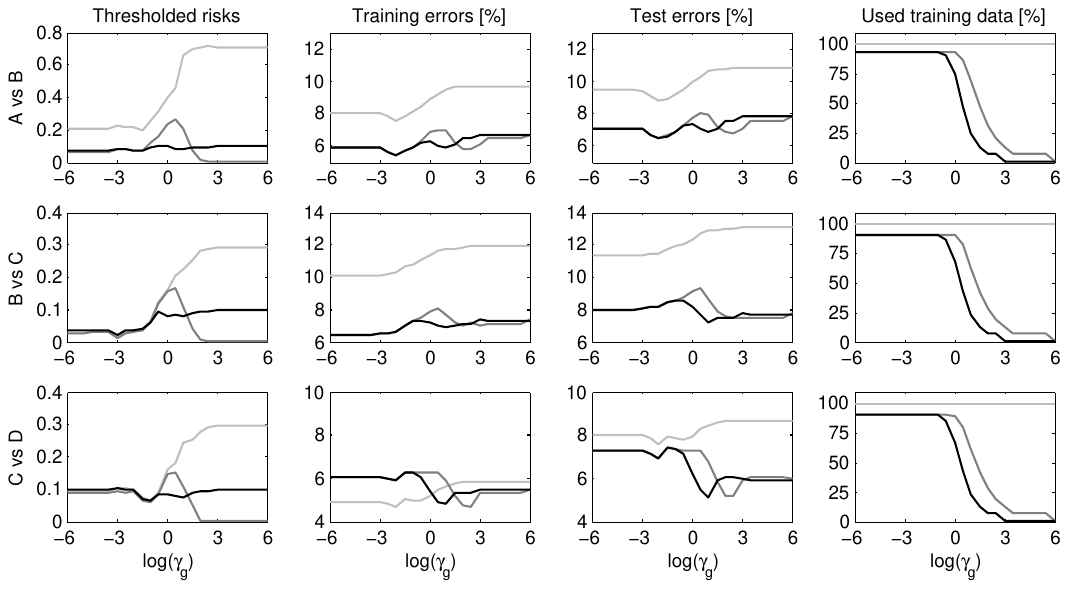}
  \caption{The thresholded empirical risk}
  \label{fig:thresholded objective}
\end{figure}

Note that the parameters $\gamma_g$ and $\eps$ are redundant in the sense that the same result is often achieved by different combinations of parameter values. This problem is addressed in the experimental section by fixing $\eps$ and optimizing $\gamma_g$ only.

\section{Analysis of Joint Quantization and Label Propagation}
\label{sec:AnalysisOfJointQuantizationAndLabelPropagation}

In this section we analyze our method of jointly optimizing for
the backbone graph and the harmonic solution (Section~\ref{sec:LargeScaleSemiSupervisedLearningWith})
by showing its connection to the principal manifold approach.
One interesting property of the objective function in \eqref{eq:quantization step2} for learning the centroids is that it has a similar form to the objective function of the elastic net model\cite{gorban2009principal}. The elastic net is a well-known technique based on an analogy between principle manifolds and an elastic membrane. It is a fast approximation of the principle manifolds and produces results similar to Kohonen's self-organized maps (SOM)~\cite{haykin1994neural}. Given a set of initial centroids and a given connectivity between the centroids (just like SOM), the elastic net has the following form:

\begin{equation}[Elastic net]
U = \gamma_q  \sum_{i \in K_j} ||x_i - c_j||^2+\sum_{i, j  \in \tilde{G} } \lambda_{ij} ||c_i - c_j||^2+\sum_{i, j , k \in \tilde{G}} \mu_{ijk} ||c_i + c_k - 2c_j||^2
\label{eq:elastic net}
\end{equation}
where $\tilde{G}$ is the graph connectivity between centroids and is assumed to be given. The objective function of the elastic net model consists of three terms: the k-means term $U^Y  = \gamma_q  \sum_{i \in K_j} ||x_i - c_j||^2$, the term $U^E  = \sum_{i, j  \in \tilde{G} } \lambda_{ij} ||c_i - c_j||^2$ for stretching elasticity, and the term $U^R = \sum_{i, j , k \in \tilde{G}} \mu_{ijk} ||c_i + c_k - 2c_j||^2$ for bending elasticity. $\lambda_{ij}$ and $\mu_{ijk}$ are the coefficients of stretching elasticity of edge between nodes $i$ and $i$ and the coefficients of bending elasticity of edge between nodes $i$, $j$, and $k$, respectively.

Notice that $U^Y$ is equivalent to the quantization penalty~\eqref{eq:quantized kmeans} for $\gamma_q = 1$. Moreover, if we set $\lambda_{ij} = -(l_i - l_j)^2/2\sigma^2$, then $U^E$ approximates $\bell\transpose L^C \bell$. Therefore, the objective function in \eqref{eq:quantization step2} is the Elastic net with no bending term and with stretching coefficients dependent on the labels of the centroids; if the labels of two centroids are similar, the objective function tries to keep them close to each other and if the labels of two centroids are different, the objective function keeps them apart.

\section{Analysis of Online SSL on Quantized graphs}
\label{sec:AnalysisOfOnlineSSLOnQuantizedGraphs}

In the rest of this section, $W$ denotes the full data similarity matrix,
$\Wo_t$ its observed portion up to time $t$ and $\Woq_t$ 
the quantized version of $\Wo_t$.
For simplicity, we do not consider the compact version of quantized matrix.
In other words, $\Woq_t$ is $t \times t$ matrix with at most $\ng$ distinct rows/columns.
The Laplacians and regularized Laplacians of these matrices
are denoted as $L, \Lo, \Loq$ and $K, \Ko, \Koq$ respectively.
Similarly, we use $\hfs$, $\hfso[t]$, and $\hfsoq[t]$ to refer to the harmonic solutions on 
$W$, $\Wo_t$, and $\Woq_t$ respectively. Finally, $\hfs_t$, $\hfso_t[t]$, and $\hfsoq_t[t]$ refer to the predicted label of the example $\bx_t$.

In this section, we use a stability argument to bound quality of the predictions.
We note that the derived bounds are not tight. 
Our online learner (Figure \ref{fig:online quantized HFS}) solves an
online regression problem. As a result, it should ideally minimize the
error of the form $\sum_t (\hfsoq_t[t] - y_t)^2$, where $\hfsoq_t[t]$ is 
the prediction at the time step $t$ (again, time is denoted in the 
square brackets).
 In the following proposition we decompose this error into three terms.
The first term (\ref{big:hfs}) corresponds to the generalization error
of the HS and is bounded by the algorithm stability argument.
The second term (\ref{big:online}) appears in our online setting because 
the similarity graph is only partially revealed.
Finally, the third term (\ref{big:quantization})
quantifies the error introduced due to quantization
of the similarity matrix.

\begin{proposition}
\label{prop:onlineHFSbound} 
Let $\hfsoq_t[t]$, $\hfso_t[t]$, $\hfs_t$ be the predictions as defined above 
and let $y_t$ be the true labels. 
Then the error of our predictions $\hfsoq_t[t]$ is bounded as
\begin{align}
\label{big:hfs} 
\frac{1}{n} \sum_{t=1}^n (\hfsoq_t[t] - y_t)^2 
 & \leq   \frac{9}{2n} \sum_{t=1}^n  (\hfs_t - y_t)^2 \\
 \label{big:online}
 & +   \frac{9}{2n} \sum_{t=1}^n  (\hfso_t[t] - \hfs_t)^2  \\
\label{big:quantization}
 & +  \frac{9}{2n}  \sum_{t=1}^n (\hfsoq_t[t] - \hfso_t[t])^2.
\end{align}
\end{proposition}
\noindent {\bf Proof:} Our bound follows from the inequality
\begin{align*}
  (a - b)^2 \leq
  \frac{9}{2}\left[(a - c)^2 + (c - d)^2 + (d - b)^2\right],
\end{align*}
which holds for $a$, $b$, $c$, $d \in [-1, 1]$. \qed

We continue by bounding all the three sums in Proposition \ref{prop:onlineHFSbound}.
These sums  can be bounded
if the constraints $\ell_i = y_i$ are enforced in a soft manner
\cite{cortes2008stability}. One way of achieving this is by solving the
related problem
\begin{equation*}
\label{eq:soft HFS matrix}
  \min_{\bell \in \realset^n} \
  (\bell - \by)\transpose C (\bell - \by) + \bell\transpose K \bell,
\end{equation*}
where $K = L + \gamma_g I$ is the regularized Laplacian of the similarity graph, $C$ is a diagonal
matrix such that $C_{ii} \! = \! c_l$ for all labeled examples, and
$C_{ii} = c_u$ otherwise, and $\by$ is a vector of pseudo-targets such
that $y_i$ is the label of the $i$-th example when the example is labeled,
and $y_i = 0$ otherwise.

\input{subsec_hfs}

\input{subsec_online}
\input{subsec_quantization}

\subsection{Discussion}
\label{sec:Discussion}
Our goal in this section is to show how much  
of regularization $\gamma_g$ is needed for the
error of our predictions to reasonably decrease over time.
We point out that in Proposition \ref{lem:transductive bound}
the lower bound for $\gamma_g$ for reasonable convergence 
is a function of $n_l$ labeled examples.
On the other hand, in Propositions \ref{prop:onlinebound} and \ref{prop:quant}
those lower bounds are the functions of all $n$ examples.

In particular, Proposition \ref{lem:transductive bound}
requires $\gamma_g = \Omega(n_l^{1+\alpha})$, $\alpha > 0$
for the true risk not to be much different from the empirical risk on the labeled points.
Next, 
Propositions \ref{prop:onlinebound} and \ref{prop:quant}
require $\gamma_g = \Omega(n^{1/4})$ and
$\gamma_g = \Omega(n^{1/8})$, respectively, for the terms 
(\ref{big:online}) and (\ref{big:quantization}) to be $O(n^{-1/2})$.

For many applications of online SSL,
a small set of $n_l$ labeled example is given in advance,  
 the rest of the examples are unlabeled. That means we usually
 expect $n \gg n_l$. 
Therefore, if we regard $n_l$ as a constant, we need to regularize 
as much as $\gamma_g = \Omega(n^{1/4})$. For such a setting of 
$\gamma_g$ we have that for $n$ approaching infinity,
the error of our predictions is getting close to 
the empirical risk on labeled examples with the rate of
$O(n^{-1/2})$.

\section{Parallel Multi-Manifold Learning}
\label{sec:ParallelMultiManifoldLearning}

In this section we analyze the approximation proposed in Section~\ref{sec:ParallelMultiManifoldLearningMethods},
when instead of computing the harmonic solution (HS) on the whole graph, we 

\begin{enumerate}
	\item decompose the graph into several smaller subgraphs,
	\item compute the HSs on the smaller graphs \emph{in parallel}, and
	\item aggregate the partial HSs.
\end{enumerate}

In the ideal case, the similarity matrix has a block-diagonal (BD) structure, which
corresponds to the graph with disconnected components. In this case, such an approximation is exact.
Since the harmonic solution for $n$ nodes of the graph has computational complexity of $\O(n^3)$,
the time savings can be significant (Section~\ref{sec:ParallelMultiManifoldLearning}).

In the rest of this section we analyze the general case, when the  similarity matrix does not have BD structure. 
Intuitively, the closer the similarity matrix resembles BD structure, the smaller decrease in prediction accuracy we expect. 

Again, if similarity and its Laplacian are BD, then HS calculated per block and as a whole are identical (even with the regularization), because it can be rewritten as solving two independent systems of linear equations.  On the other hand, an \emph{impurity} of BD structure can change HS a lot (think of the case when we merge blocks with labeled examples from different classes). 
We continue by extending the analysis in Section~\ref{sec:AnalysisOfOnlineSSLOnQuantizedGraphs}
and follow Proposition~\ref{prop:quant}:

\begin{lemma}
Let $\ello$ and $\elloq$ minimize (\ref{eq:soft HFS}) and its perturbed version, respectively.
Then 
\begin{equation*}
\|\elloq - \ello\|_2  \leq \frac{\sqrt{n_l}}{c_u \gamma_g^2} 
\|\Koq- \Ko \|_F.
\end{equation*}
\end{lemma}

\noindent The proof is in Section~\ref{quantization_bound}. The question now is how to bound $\|\Koq- \Ko \|_F$ or $\|\Loq- \Lo \|_F$ if the same regularization is used. Let $\Lbd$ denote general block-diagonal approximation of $\Loq$, where the entries outside the BD structure are ignored (ie.\,are assumed
to be zero). Then

\begin{equation}[Approximating Laplacian by block-diagonal structure]
\|\Lbd- \Lo \|_F \leq \|\Lbd- \Loq \|_F + \|\Loq- \Lo\|_F.
\label{eq:Lbd_Lo}
\end{equation}

\noindent Let $\dmax$ be the value of the maximum entry in $\Koq$, which is ignored when the approximation is performed.
In general, for a BD setting, we can have $n^2/2$ to $n^2$ ignored
entries. Therefore,

\begin{equation}[Approximation bound of block-diagonal structure]
n \sqrt{\dmax/2} \leq \|\Lbd- \Loq \|_F \leq  n\sqrt{\dmax}.
\label{eq:Lbd_upper}
\end{equation}

\noindent  This approximation adds a factor of $\Theta(n)$ to the quantization bound (Section~\ref{quantization_bound}).
To maintain the overall convergence of $O(n^{-1/2})$ we need to 
have $\gamma_g = \Omega{(n^{3/8})}$, along with the discussion in Section~\ref{sec:Discussion}.

\section{Analysis of Conditional Anomaly Detection}
\label{sec:AnalysisOfConditionalAnomalyDetection}


\begin{figure}
\begin{center}
\includegraphics[width = 0.2\columnwidth, clip, viewport = 215 155 532 435]{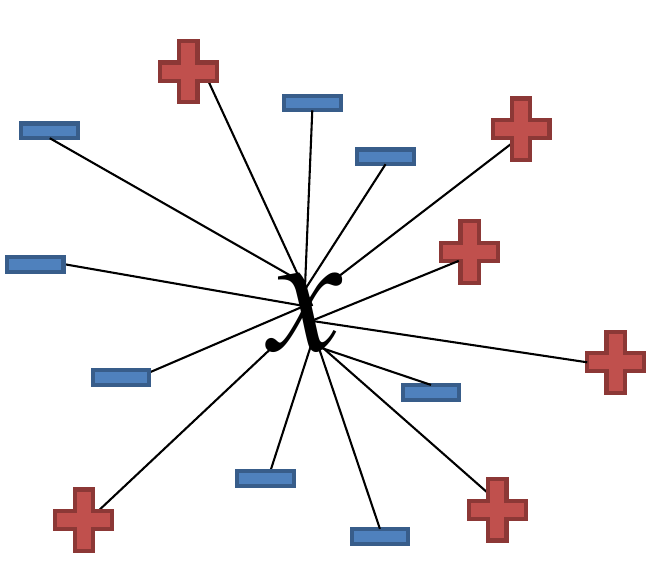} 
\caption{Estimating the likelihood ratio from a single graph}
\label{fig:greg_graph}%
\end{center}
\end{figure}

In this part we show that the weighed $k$-NN is a special case
of $\lambda$-RWCAD for $\lambda = 0$ and $n \to \infty$. Rewriting
(\ref{eq:cad_with_rw}) we get:

\begin{equation}[Class-conditional probability of a different label]
P(y \ne +1|\bx_i) = \frac{1}{1+\frac{P(\bx_i|y = +1)P(y = +1)}{P(\bx_i|y  =
-1)P(y = -1)}}
\label{eq:likelihood_rw}
\end{equation}

\noindent Let us estimate $P(\bx_i|y = +1)/P(\bx_i|y = -1)$ 
 -- \emph{conditional likelihood} --
in (\ref{eq:likelihood_rw})
also from a stationary distribution of a random walk shown in
Figure~\ref{fig:greg_graph}, where we connect the node representing $\bx_i$ with
all examples in the training set (from all classes) and define the likelihood ratio
as the ratio between the time spent in the nodes with the respective labels:

\begin{equation}[Approximation of the likelihood ratio]
\frac{P(\bx_i|y = +1)}{P(\bx_i|y = -1)} = \frac{\#T(y = +1)}{\#T(y = -1)},
\label{eq:greg_times}
\end{equation}

\noindent where $\#T(y = c)$ is the time spent in the nodes of class $c$ during the random walk.
Let $W$, $W^+$, $W^-$  be the weight matrices
for all, just the positive, and just the negative nodes, respectively.
Combining \eqref{eq:greg_times} and \eqref{eq:rw_closed}, we get:
\begin{equation}[Estimation expressed with similarity weights]
\frac{P(\bx_i | y=+1)}{P(\bx_i|y=-1)} = \frac{\sum_j W^+_{j\bx_i}}{\sum_j
W^-_{j\bx_i}}
\label{eq:close_form_greg}
\end{equation}

\noindent  which is equal to the weighted $k$-NN method.
Now, let $T^+ = \mathrm{vol}(W^+)$ and $T^- = \mathrm{vol}(W^-)$ be the sums of
all weights in $W^+$ and $W^-$. Moreover, let $T_{\bx_i}^+$, $T_{\bx_i}^-$ be
the total edge sums of the respective graphs including the node $\bx_i$.
The conditional likelihood of the $\lambda$-RWCAD for $\lambda = 0$ can
be derived combining \eqref{eq:cad_with_rw} and \eqref{eq:rw_closed} to get:

\begin{equation}[Equivalent form of RWCAD algorithm]
\frac{P(\bx_i | y=+1)}{P(\bx_i|y=-1)} = \frac{\sum_j W^+_{j\bx_i}}{\sum_j
W^-_{j\bx_i}}
\times
\frac{T_{\bx_i}^-}{T_{\bx_i}^+}
\label{eq:close_form_had}
\end{equation}

\noindent Equations \eqref{eq:close_form_greg} and \eqref{eq:close_form_had} are the
conditional likelihoods for the weighted $k$-NN and RWCAD for $\lambda = 0$,
respectively. Notice that as the number of nodes increases,
$T_{\bx_i}^-/T_{\bx_i}^+$ approaches $T^-/T^+$, which is a constant. Therefore, the influence
of one node $(\bx_i)$ in the ratio becomes negligible. In that case, both
methods will yield comparable results.

%% file: subsec_hfs.tex
\subsection{Bounding Transduction Error (\ref{big:hfs})}
\label{ss:softhfs}

%
The following proposition bounds the generalization error of the solution to the problem (\ref{eq:soft
HFS}). We then use it to  bound the HS part (\ref{big:hfs}) of  Proposition \ref{prop:onlineHFSbound}. 


\begin{proposition}
\label{lem:transductive bound2} Let $\bell^\ast$ be a solution to the
problem (\ref{eq:soft HFS}),
where all labeled examples $l$ are selected
i.i.d.  If we assume that $c_l = 1$ and $c_l \gg c_u$,
then the inequality
\begin{align*}
  R(\bell^\ast) \ \leq & \ \
  \widehat{R}(\bell^\ast) +
  \underbrace{\beta + \sqrt{\frac{2 \ln(2 / \delta)}{n_l}}
  (n_l \beta + 4)}_{\emph{transductive error }
  \Delta_T(\beta, n_l, \delta)} \\
  \beta \ \leq & \ \
  2 \left[\frac{\sqrt{2}}{\gamma_g + 1} +
  \sqrt{2 n_l} \frac{1 - \sqrt{c_u}}{\sqrt{c_u}}
  \frac{\lambda_M(L) + \gamma_g}{\gamma_g^2 + 1}\right]
\end{align*}
holds with the probability of $1 - \delta$, where
\begin{align*}
  R(\bell^\ast) =
  \frac{1}{n} \sum_t (\hfs_t - y_t)^2{\rm\ and \ }
  \widehat{R}(\bell^\ast) =
  \frac{1}{n_l} \sum_{t \in l} (\hfs_t - y_t)^2
\end{align*}
are risk terms for both all and labeled vertices, respectively, and $\beta$ is
the stability coefficient of the solution $\bell^\ast$.
\end{proposition}

The proof can be found in Section~\ref{sec:generalization error}. 
Proposition \ref{lem:transductive bound2} shows
that when $\Delta_T(\beta, n_l, \delta) = o(1)$,
the true risk is not much different from the empirical risk on the labeled points
which bounds the generalization error.
This occurs when $\beta \! = \! o(n_l^{-1/2})$,
which corresponds to  setting $\gamma_g =
\Omega(n_l^{1+\alpha})$ for any $\alpha > 0$. 

%% file: subsec_online.tex
\subsection{Bounding Online Error (\ref{big:online})}
\label{ss:online}

In the following, we will bound the difference between the online and offline HS and use it to bound (\ref{big:online}) of the Proposition \ref{prop:onlineHFSbound}. 
The idea is that when Laplacians $L$ and $\Lo$ are regularized enough by $\gamma_g$,
the resulting harmonic solutions are close to zero and therefore close to each other.
We first show that any regularized HS can be bounded as follows:

\begin{lemma}
\label{lemma:reg_hfs_bound} Let $\ell$ be a regularized harmonic 
solution, i.e.~$\ell = (C^{-1}K  + I)^{-1}\by$  where $K = L +
\gamma_gI$. Then
\[
\|\ell\|_2 \leq \frac{\sqrt{n_l}}{\gamma_g + 1}.
\]

\end{lemma}

\noindent {\bf Proof:} If $A \in \realset^{n \times n}$ is a symmetric matrix and $\lambda_m (A)$ and
$\lambda_M (A)$ are its smallest and largest eigenvalues, then for any
$\bv \in \realset^{n \times 1}$, $\lambda_m (A)\|\bv\|_2 \leq \| A\bv \|_2 \leq \lambda_M (A) \|\bv\|_2$. Then
\begin{align*}
\| \ell \|_2 & \leq  \frac{ \| \by \|_2}{\lambda_m (C^{-1}K  + I)}
  =  \frac{ \| \by \|_2}{\frac{\lambda_m(K)}{\lambda_M(C)}+ 1}
  \leq  \frac{\sqrt{n_l}}{\gamma_g + 1}.
\quad \qed
\end{align*}

\noindent The straightforward implication of Lemma \ref{lemma:reg_hfs_bound} is
that any 2 regularized harmonic solutions can be bounded as in the
following proposition:

\begin{proposition}
\label{prop:onlinebound} Let $\hfso[t]$ be the predictions of the online
HS, and $\hfs$ be the predictions of the offline
HS. Then

\begin{equation}[Upper bound on the online regret]
\label{eq:sum_online}
\frac{1}{n} \sum_{t=1}^n  (\hfso_t[t] - \hfs[t])^2 \leq \frac{4n_l}{(\gamma_g + 1)^2} \cdot
\end{equation}
\end{proposition}

\noindent  {\bf Proof:}
We use the fact that $\|\cdot\|_2$ is an upper bound on $\|\cdot\|_\infty$.
Therefore, for any $t$
\begin{align*}
(\hfso_t[t] - \hfs_t)^2 
& \leq \|\hfso[t] - \hfs\|^2_\infty 
\leq \|\hfso[t] - \hfs\|^2_2 \\
& \leq \left(\frac{2\sqrt{n_l}}{\gamma_g + 1}\right)^2,
\end{align*}
where in the last step we used Lemma \ref{lemma:reg_hfs_bound} twice. By summing
over $n$ and dividing by $n$ we get (\ref{eq:sum_online}). \qed

From Proposition \ref{prop:onlinebound} we see that we can achieve 
 convergence of the term (\ref{big:online}) at the rate of $O(n^{-1/2})$  with $\gamma_g = \Omega(n^{1/4})$.

%% file: subsec_quantization.tex
\subsection{Bounding Quantization Error (\ref{big:quantization})  }
\label{quantization_bound}

In this section, we show in Proposition \ref{prop:quant} how to 
bound the error for the HS between the full and quantized graph,
and then use it to bound the difference between the \emph{online} and 
\emph{online quantized} HS in (\ref{big:quantization}).
Let us consider the perturbed version of the problem (\ref{eq:soft HFS}),
where we replace the regularized Laplacian $\Ko$ with $\pertK$;
i.e., $\pertK$ corresponds to the regularized Laplacian of the 
quantized graph. 
Let $\ello$ and $\elloq$ minimize (\ref{eq:soft HFS}) and its 
perturbed version respectively.
Their closed-form solutions are given by
$\ello = (C^{-1}\Ko + I)^{-1}\vecy$
and 
$\elloq = (C^{-1}\pertK + I)^{-1}\vecy$
respectively. We now follow the derivation of 
\cite{cortes2008stability} that derives
stability coefficients for unconstrained regularization algorithms.
Instead of considering perturbation on $C$, we consider
the perturbation on $\Ko$. Our goal is to derive
a bound on a difference in HS when we use $\pertK$ instead of $\Ko$. 

\begin{lemma}
\label{lemma:stability_laplacian}
Let $\ello$ and $\elloq$ minimize (\ref{eq:soft HFS}) and its perturbed version respectively.
Then 
\begin{equation*}
\|\elloq - \ello\|_2  \leq \frac{\sqrt{n_l}}{c_u \gamma_g^2} 
\|\Koq- \Ko \|_F.
\end{equation*}
\end{lemma}
\noindent {\bf Proof:} 
Let $\Zoq = C^{-1}\Koq + I$ and $\Zo = C^{-1}\Ko + I$. By definition
\begin{align*}
\elloq - \ello & = 
(\Zoq)^{-1}\vecy - (\Zo)^{-1}\vecy
 =   (\Zoq\Zo)^{-1}(\Zo - \Zoq)\vecy\\
& =  (\Zoq\Zo)^{-1} C^{-1} (\Ko- \pertK) \vecy.
\end{align*}
Using the eigenvalue inequalities from the proof  of Lemma \ref{lemma:reg_hfs_bound} we get
\begin{equation}[Upper bound on stability of quantized harmonic solution]
\label{eq:stability_eigen}
\|\elloq - \ello\|_2  \leq \frac{\lambda_M(C^{-1})\|(\Koq- \Ko) \vecy\|_2}{\lambda_m(\Zoq)\lambda_m(\Zo)}.
\end{equation}
By the compatibility of $||\cdot||_F$ and $||\cdot||_2$ and 
since $\vecy$ is zero on unlabeled points, we have
\begin{equation*}
\|(\Koq- \Ko) \vecy\|_2 \leq \|\Koq- \Ko \|_F \cdot \|\vecy\|_2 
\leq \sqrt{n_l} \|\Koq- \Ko \|_F.
\end{equation*}
Furthermore,
\begin{equation*}
\lambda_m(\Zo) \geq \frac{\lambda_m(\Ko)}{\lambda_M(C)} + 1 \geq \gamma_g \quad {\rm and} \quad \lambda_M(C^{-1}) \leq c_u^{-1},
\end{equation*}
where $c_u$ is a small constant as defined  in (\ref{eq:soft HFS}). 
By plugging these inequalities into (\ref{eq:stability_eigen}) we get the desired bound.
\qed

\begin{proposition}
\label{prop:quant} 
Let $\hfsoq_t[t]$ be the predictions of the
online  harmonic  solution on the quantized graph at the time step $t$ and $\hfso_t[t]$ be
predictions of the online harmonic solution at the time step $t$. Then
\begin{equation}[Upper bound on the quality of quantization]
\label{eq:quant}
 \frac{1}{n}\sum_{t=1}^n (\hfsoq_t[t] - \hfso_t[t])^2
 \leq
 \frac{n_l}{c_u^2\gamma_g^4}
\|\Loq - \Lo\|_F^2.
 \end{equation}
\end{proposition}
\noindent {\bf Proof:}
Similarly as in Proposition \ref{prop:onlinebound}, we get 
\begin{align*}
(\hfsoq_t[t] - \hfso_t[t])^2 & \leq \|\hfsoq[t] - \hfso\|^2_\infty 
\leq \|\hfsoq[t] - \hfso\|^2_2 \\
& \leq \left( \frac{\sqrt{n_l}}{c_u \gamma_g^2}||\Koq-\Ko||_F \eat{\|\hfsoq[t]\|_2} \right)^2,
\end{align*}
where we used (\ref{eq:stability_eigen}) the last step. 
We also note that 
$$ \Koq -\Ko =
 \Loq + \gamma_gI - (\Lo + \gamma_gI) =
\Loq - \Lo, $$
which gives us
$(\hfsoq_t[t] - \hfso_t[t])^2 \leq \|\Loq - \Lo\|_F^2 \cdot n_l/(c_u^2\gamma_g^4)$.
By summing
over $n$ and dividing by $n$ we get (\ref{eq:quant}). \qed

If $\|\Loq - \Lo\|_F^2 = O(1)$,
the left-hand side of \eqref{eq:quant} converges to zero at the 
rate of $O(n^{-1/2})$ with $\gamma_g = \Omega(n^{1/8})$.
We show this condition is achievable whenever the Laplacian 
is scaled appropriately. Specifically, we demonstrate that normalized
Laplacian achieves this bound when the quantization is 
performed using incremental $k$-center clustering in 
Section~\ref{sec:OnlineLearningWithQuantizedGraphs}, and when the weight function
obeys a Lipschitz condition 
(e.g. the Gaussian kernel). We also show that this error goes 
to zero as the number of center points $\ng$ goes to infinity.

Suppose the data $\{\bx_i\}_{i=1,...,n}$ lie on a  smooth $d$-dimensional compact manifold $\mathcal{M}$ with boundary of bounded geometry as defined in Definition~11 (Manifold with boundary of bounded geometry) in~\cite{hein2007graph}. 
Intuitively, the manifold should not intersect itself or fold back onto itself.
We first demonstrate that the distortion introduced by quantization is small, and
then show that small distortion gives a small error in the Frobenius norm.
\begin{proposition}
\label{prop:distortion} 
Using incremental $k$-center clustering for quantization has maximum 
distortion $R \Rmultiplier / (\Rmultiplier-1) = \max_{i=1,...,n} \|\bx_i - \bc\|_2 = O(\ng^{-1/d})$, 
where $\bc$ is the closest centroid to $\bx_i$.
\end{proposition}
\noindent {\bf Proof:} 
Consider a sphere packing with $\ng$ 
centers contained in $\M$ and each with radius~$r$.
Since the manifold is compact and the boundary has bounded geometry, 
it has finite volume $V$ 
and finite surface area $A$. 
The maximum volume that the packing can occupy obeys the inequality
$\ng c_d r^d \leq V + A c_{\M} r$
for some constants $c_d,c_{\M}$ that only depend on 
the dimension and the manifold.
For a sufficiently large $\ng$, $r$ will be smaller than 
the \emph{injectivity radius} of $\M$ \cite{hein2007graph}. 
Moreover, if $\ng$ is sufficiently large, then $r < 1$, 
and we have an upper bound 
$r < ((V + A c_{\M}) / (\ng c_d) )^{1/d} = O(\ng^{-1/d})$.
An $r$-packing is a $2r$-covering, so we have an upper bound
on the distortion of the optimal $k$-centers solution. 
Since the incremental $k$-centers algorithm is 
a $(1+\epsilon)$-approximation algorithm ~\cite{charikar1997incremental}, 
it follows that the maximum distortion returned by 
the algorithm is $R \Rmultiplier / (\Rmultiplier-1) = 2 (1+\epsilon) O(\ng^{-1/d})$.
\qed

We now show that with appropriate normalization, the error 
$\|\Loq - \Lo\|_F^2 = O(\ng^{-2/d})$. 
If $\Loq$ and $\Lo$ are normalized Laplacians, then this bound holds if the underlying
density is bounded away from 0. Note that since we use the Gaussian kernel,
the Lipschitz condition is satisfied.
\begin{proposition}
\label{prop:frob error}
Let $\Wo_{ij}$ be a weight matrix constructed from $\{x_i\}_{i=1,...,n}$
and a bounded, Lipschitz function $\omega(\cdot,\cdot)$
with Lipschitz constant $M$.
Let $\Do$ be the corresponding degree matrix and
$\Lo_{ij} = (\Do_{ij} - \Wo_{ij})/\co_{ij}$ 
be the normalized Laplacian.
Suppose $\co_{ij} =\sqrt{\Do_{ii}\Do_{jj}}> c_{min}n$ for 
some constant $c_{min} > 0$ that does not depend on $\ng$. 
Likewise define $\Woq, \Loq, \Doq$ on the quantized points.
Let the maximum distortion be $R \Rmultiplier / (\Rmultiplier-1) = O(\ng^{-1/d})$.
Then $\|\Loq - \Lo\|_F^2 = O(\ng^{-2/d})$.
\end{proposition}
\noindent {\bf Proof:}
Since $\omega$ is Lipschitz, we have
that $|\Woq_{ij} - \Wo_{ij}| < 2MR \Rmultiplier / (\Rmultiplier-1)$ and $|\coq_{ij} - \co_{ij}| < 2 n M R \Rmultiplier / (\Rmultiplier-1)$.
The error of a single off-diagonal entry of the Laplacian matrix is 
\begin{align*}
\Loq_{ij} - \Lo_{ij} &= \frac{\Woq_{ij}}{\coq_{ij}} - \frac{\Wo_{ij}}{\co_{ij}}\\
&\leq \frac{\Woq_{ij} - \Wo_{ij}}{\coq_{ij}}  +  
\frac{\Woq_{ij} (\coq_{ij} - \co_{ij})}{\co_{ij}\coq_{ij}} \\
&\leq \frac{4 M R \Rmultiplier}{(\Rmultiplier-1)c_{min}n} + \frac{4 M (n M R \Rmultiplier)}{((\Rmultiplier-1)c_{min}n)^2} \\
& = O\left(\frac{R}{n}\right).
\end{align*}
The error on the diagonal entries is 0 since
the diagonal entries of $\Loq$ and $\Lo$ are all $1$.
Thus $\|\Loq - \Lo\|_F^2 \leq n^2 O(R^2/n^2) = O(\ng^{-2/d})$.
\qed

Here we showed the asymptotic behavior $\|\Loq - \Lo\|_F$
in term of the number of vertices used in the quantized graph.
In Section~\ref{sec:UCI ML repository experiments UAI}, we empirically show 
that $\|\Loq - \Lo\|_F$ vanishes quickly as the number of
vertices increases (Figure~\ref{fig:results UCI ML centroids}).
Moreover, with a fixed number of vertices, $\|\Loq-\Lo\|_F$
quickly flattens out even when the data size (time) 
keeps increasing (Figure~\ref{fig:results UCI ML time}).

%


%% file: chap_experiments.tex
\chapter{Experiments}
\label{sec:Experiments}

This chapter presents the set of experiments we performed for semi-supervised learning (SSL) and 
conditional anomaly detection (CAD). We present our SSL results in Section~\ref{sec:EvaluationsOfPredictiveModels}
and our CAD results in Section~\ref{sec:EvaluationOfAnomalyDetection}.
We start each of these sections with the descriptions of the data.
We used medical, vision, the UCI ML repository, and synthetic datasets.

\input{sec_eval_predictive}

\input{sec_evaluations}

%% file: sec_eval_predictive.tex
\section{Evaluations of Semi-Supervised Learning Models}
\label{sec:EvaluationsOfPredictiveModels}

In this section we evaluate the predictive performance of our graph-based model on semi-supervised tasks.
Our goal is to demonstrate that graph-based methods can yield predictors  
that outperform the current state-of-the-art methods. We continue with the 
description of the datasets we used. 

\input{datasets_ssl}

\subsection{Max-margin Graph Cuts Experiments}
\label{sec:MaxMarginGraphCutsExperiments}

The experiments with max-margin graph cuts are divided into two parts. The first part compares max-margin graph cuts to manifold regularization of SVMs on the problem from Figure \ref{fig:HFS}. The second part compares max-margin graph cuts, manifold regularization of SVMs, and supervised learning with SVMs on three UCI ML repository datasets \cite{asuncion2007uci}.
Manifold regularization of SVMs is evaluated based on the implementation of \cite{belkin2006manifold}. Max-margin graph cuts and SVMs are implemented using LIBSVM \cite{chang2001libsvm:}.

\subsubsection{Synthetic problem}
\label{sec:synthetic experiments}

The first experiment (Figure \ref{fig:synthetic cuts}) illustrates linear, cubic, and RBF graph cuts (\ref{eq:MMGC}) on the synthetic problem from Figure \ref{fig:HFS}. The cuts are shown for various settings of the regularization parameter $\gamma_g$. As $\gamma_g$ decreases, note that the cuts gradually interpolate between supervised learning on just two labeled examples and semi-supervised learning on all data. The resulting discriminators are max-margin decision boundaries that separate the corresponding colored \mbox{regions in Figure \ref{fig:HFS}.}

Figure \ref{fig:synthetic cuts} also shows that the manifold regularization of SVMs (\ref{eq:MR}) with linear and cubic kernels cannot perfectly separate the two clusters in Figure \ref{fig:HFS} for any setting of the parameter $\gamma_u$. The reason for this problem is discussed in Section \ref{sec:MR fails}. Finally, note the similarity between max-margin graph cuts and manifold regularization of SVMs with the RBF kernel. This similarity was suggested in Section \ref{sec:MR}.

\subsubsection{UCI ML repository datasets}
\label{sec:UCI ML repository experiments}

\begin{figure}[ht]
  \begin{center}
  {\small
  \begin{tabular}{l r|r r r|r r r|r r r} \hline \hline
    & & \multicolumn{9}{|c}{Misclassification errors [\%]} \\ \cline{3-11}
    Dataset & $L$ & \multicolumn{3}{|c|}{Linear kernel} &
    \multicolumn{3}{|c|}{Cubic kernel} & \multicolumn{3}{|c}{RBF kernel} \\
    & & SVM & MR & GC & SVM & MR & GC & SVM & MR & GC \\ \hline
    &  1 & 18.90 & 30.94 & {\bf 15.79} & 20.54 & 25.96 & {\bf 17.45} & 20.06 & 17.61 & {\bf 16.01} \\
    Letter
    &  2 & 12.92 & 28.45 & {\bf 10.79} & 12.18 & 18.34 & {\bf 10.90} & 13.52 & 13.10 & {\bf 11.83} \\
    recognition
    &  5 &  8.21 & 27.13 & {\bf  5.65} &  5.49 & 18.77 & {\bf  4.80} &  6.81 &  8.06 & {\bf  5.65} \\
    & 10 &  6.51 & 25.45 & {\bf  3.96} &  4.17 & 14.03 & {\bf  2.96} &  4.95 &  6.14 & {\bf  3.32} \\ \hline
    &  1 &  7.06 &  9.59 & {\bf  6.88} &  9.62 & {\bf  5.29} &  8.55 &  8.22 & {\bf  6.36} &  7.65 \\
    Digit
    &  2 &  4.87 &  7.97 & {\bf  4.60} &  6.06 & {\bf  5.06} &  5.09 &  6.17 & {\bf  4.21} &  5.61 \\
    recognition
    &  5 &  2.97 &  3.68 & {\bf  2.29} &  3.04 & {\bf  2.27} &  2.36 &  2.74 &  2.29 & {\bf  2.19} \\
    & 10 &  1.70 &  2.86 & {\bf  1.59} &  1.87 & {\bf  1.60} &  1.74 &  1.68 &  1.75 & {\bf  1.35} \\ \hline
    &  1 & 14.02 & 11.81 & {\bf 10.27} & 23.30 & {\bf 12.02} & 14.10 & 14.02 & 11.60 & {\bf  9.51} \\
    Image
    &  2 &  8.54 & 10.87 & {\bf  7.69} & 14.28 & 13.07 & {\bf  7.73} &  9.06 &  8.93 & {\bf  7.34} \\
    segmentation
    &  5 &  4.73 &  7.83 & {\bf  4.49} &  8.32 &  8.79 & {\bf  7.17} &  5.87 &  5.43 & {\bf  5.31} \\
    & 10 &  3.30 &  6.26 & {\bf  3.28} &  3.65 &  6.64 & {\bf  3.60} &  3.84 &  4.81 & {\bf  3.73} \\
    \hline \hline
  \end{tabular}
  }
  \caption{Comparison of SVMs, GC and MR on 3 datasets from the UCI ML repository}
  \label{fig:UCI ML repository cuts}
	\end{center}
\end{figure}

The second experiment (Figure \ref{fig:UCI ML repository cuts}) shows that max-margin graph cuts (\ref{eq:MMGC}) typically outperform manifold regularization of SVMs (\ref{eq:MR}) and supervised learning with SVMs. 
In particular it shows the comparison of SVMs, max-margin graph cuts (GC), and manifold regularization of SVMs (MR) on three datasets from the UCI ML repository. The fraction of labeled examples $L$ varies from 1 to 10 percent.

The experiment is done on three UCI ML repository datasets: letter recognition, digit recognition, and image segmentation. The datasets are multi-class and thus, we transform each of them into a set of binary classification problems. The digit recognition and image segmentation datasets are converted into 45 and 15 problems, respectively, where all classes are discriminated against every other class. The letter recognition dataset is turned into 25 problems that involve pairs of consecutive letters. Each dataset is divided into three folds. The first fold is used for training, the second one for selecting the parameters $\gamma \! \in \! [0.01, 0.1] n_l$, \mbox{$\gamma_u \! \in \! [10^{-3}, 10^3] \gamma$, and} $\gamma_g = \gamma / \gamma_u$, and the last fold is used for testing.\footnote{Alternatively, the regularization parameters $\gamma$, $\gamma_u$, and $\gamma_g$ can be set using leave-one-out cross-validation on labeled examples.} The fraction of labeled examples in the training set is varied from 1 to 10 percent. All examples in the validation set are labeled, and its size is limited to the number of labeled examples in the training set.

In all experiments, we use 5-nearest neighbor graphs whose edges are weighted as $w_{ij} = \exp[- \normw{\bx_i - \bx_j}{2}^2 / (2 K \sigma^2)]$, where $K$ is the number of features, and $\sigma$ denotes the mean of their standard deviations. The width of radial basis functions (RBFs) is set accordingly to $\sqrt{K} \sigma$, and the threshold $\eps$ for choosing training examples (\ref{eq:MMGC}) is $10^{-6}$.

The test errors of all compared algorithms are averaged over all binary problems within each dataset and shown in Figure \ref{fig:UCI ML repository cuts}. Max-margin graph cuts outperform manifold regularization of SVMs in 29 out of 36 experiments. Note that the lowest errors are usually obtained for linear and cubic kernels, and our method improves the most over manifold regularization of SVMs in these settings.

\subsection{Joint Quantization and Label Propagation Experiments}
\label{sec:JointQuantizationAndHarmonicSolutionExperiments}

In this part, we evaluate the method we proposed in Section~\ref{sec:LargeScaleSemiSupervisedLearningWith}
that combines the creation of a backbone graph with label propagation.
The benefit of our algorithm comes when the data lies on a low dimensional manifold.
In this section, we show the 2 data sets when this is the case.
For data sets without a manifold structure or for the data sets
where a cluster assumption holds, the performance of our method is comparable to the case when $k$-means is
used as a preprocessing step. We compare our algorithm to several quantization approaches:

\begin{enumerate}
	\item \emph{random subsampling}: We randomly sample $k$ examples from the unlabeled data. Then we apply SSL method on the selected samples.
	\item \emph{$k$-means}: We cluster the unlabeled data using $k$-means \cite{hastie2001elements} to get $k$ cluster centers and then apply SSL algorithms to get their labels.
	\item \emph{elastic nets}: We use elastic net \cite{gorban2009principal} as a preprocessing to get $k$
cluster centers. We then apply SSL to get their labels.
	\item \emph{elastic-joint}: We apply the proposed algorithm in this dissertation to get both the centroids and their labels.
    \item \emph{full-soft}: We apply SSL algorithm on the full set of examples as a reference point.
\end{enumerate}
After obtaining the labels of the centroids using items 1-4 above, we apply the approximation in Section~\ref{sec:appox} to get the labels for unlabeled examples.

\subsubsection{Experimental setup}
\label{sec:ExperimentalSetup}
We use a small subset of examples as labeled examples. To see the sensitivity of the method on a different number of labeled examples, we try $m = 2, 10, 20,$~and~50 as the number of labeled examples.
To allow for the fair comparison between the methods, we run all the algorithms on the same set of labeled examples.
Moreover, all the approximation methods are initialized with the same cluster centers (seeds) as the ones that were drawn by  random subsampling. 

Finally, we fix all the parameters for the semi-supervised prediction in Equation~\eqref{eq:soft HFS} to the same settings as follows. We create a 3-nearest neighbors similarity graph, and we use the Gaussian kernel with the kernel width $\sigma$ equal to 10\% of the standard deviation of the distances as suggested in \cite{luxburg2007tutorial}.

For each of the methods we compute the regularized graph Laplacian, where we add $\gamma_g = 10^{-6}$ to the diagonal. For the diagonal matrix $F$ of empirical weights we set $f_l = 10$ for the labeled and $f_u = 0.1$ for the unlabeled examples. We set parameter $\gamma_q$ in our method to $10^5$. Finally, we vary the number of cluster centers as $k=15,20,25,30,60,$~and~90.

\subsubsection{Results}
\label{sec:Results}
The results are shown in Figure~\ref{fig:results} for the varying number of labeled examples $m$ and centroids $k$.
Error bars show the 95\% confidence intervals over 50 runs. The 5 compared methods are
1) subsample --- random subsampling, 2) $k$-means as a preprocessing, 3) our method: elastic-joint,
4) elastic net as a preprocessing, and  5) full soft --- harmonic solution using all unlabeled examples
to create the full graph. For the Car dataset and $m=2$ unlabeled examples, our method
outperforms the other baselines for the different number of cluster centers up to $k=60$, where all the methods achieved the performance of the full non-approximated graph. For $m=10,20,$~and~$50$, all the subsampling methods are comparable. For the COIL dataset, $m=2$ of labeled examples was not sufficient for learning, as the classes are perfectly balanced and all the methods produced a trivial classifier comparable to a random one, including the SSL on the full graph with all the examples. For $m=10,20,$~and~$50$, our method significantly outperforms all the other approximation methods. The result for SecStr (Figure~\ref{fig:time}) is similar for all the baselines. We utilize this data set to show the time complexity of different methods. Notice that the same observation and setup is used in~\cite{chapelle2006semi-supervised}.

\begin{figure}	
\begin{center}
\includegraphics[width=0.40\columnwidth,clip]{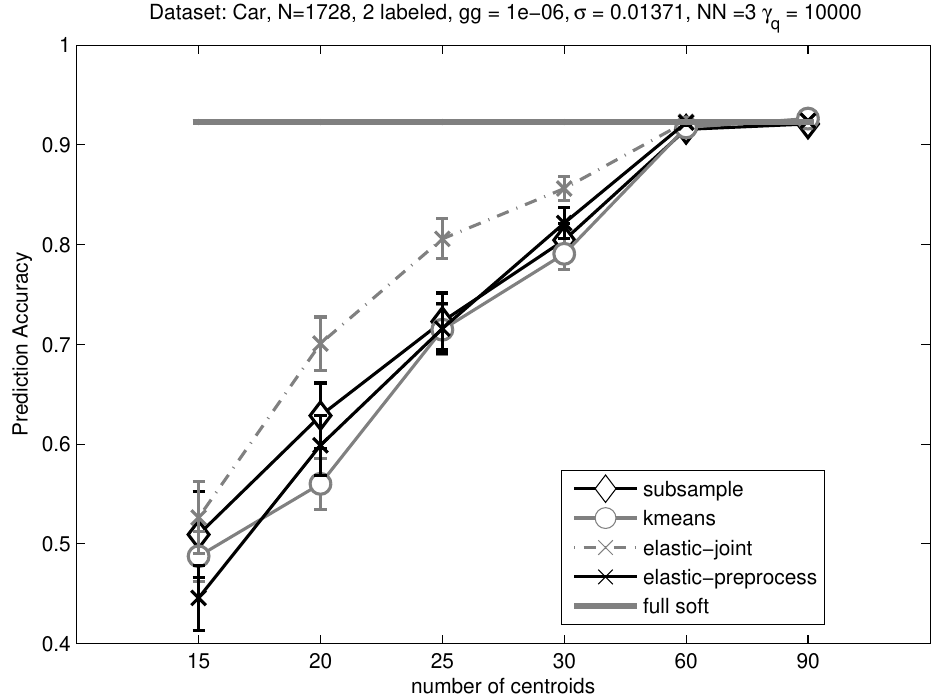}
\includegraphics[width=0.40\columnwidth,clip]{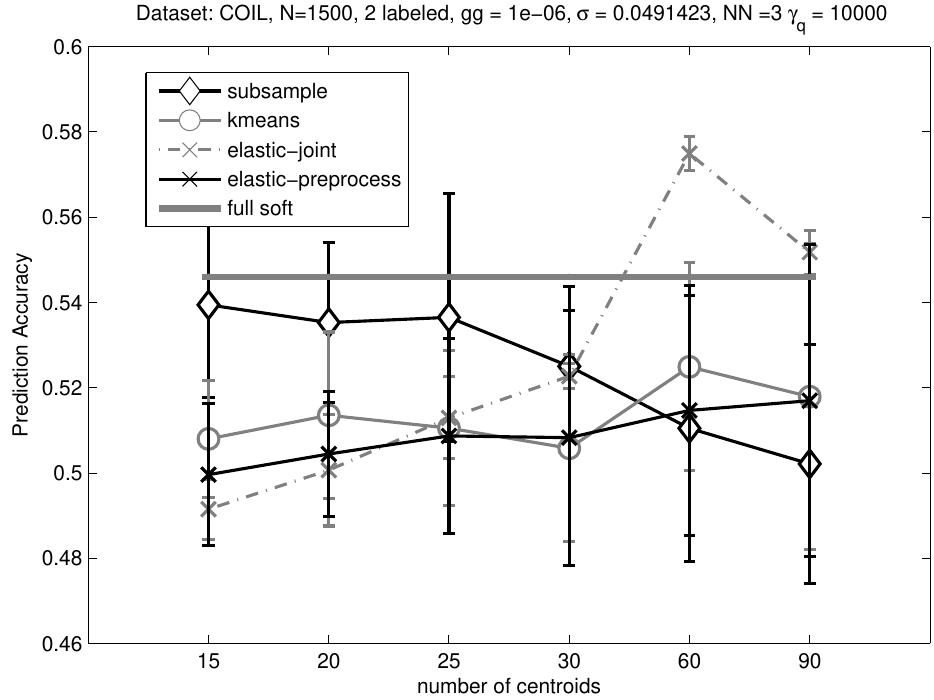}
\includegraphics[width=0.40\columnwidth,clip]{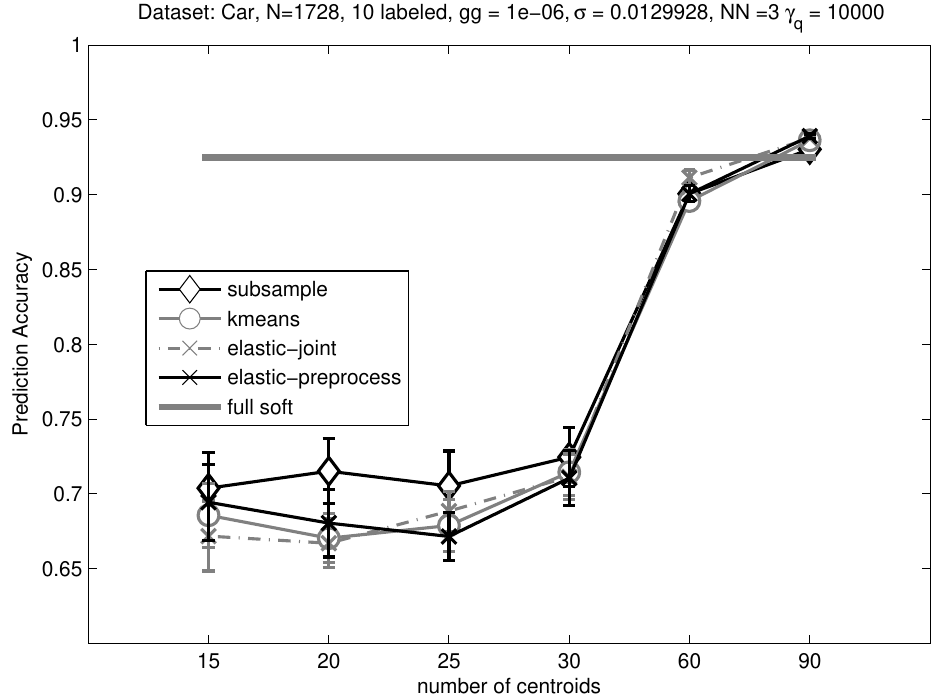}
\includegraphics[width=0.40\columnwidth,clip]{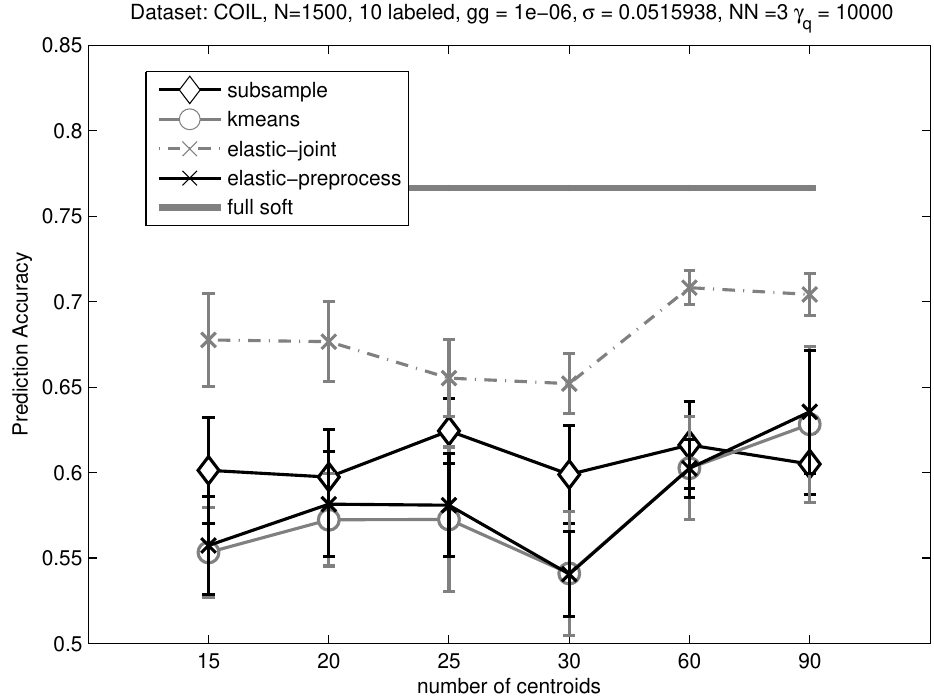}
\includegraphics[width=0.40\columnwidth,clip]{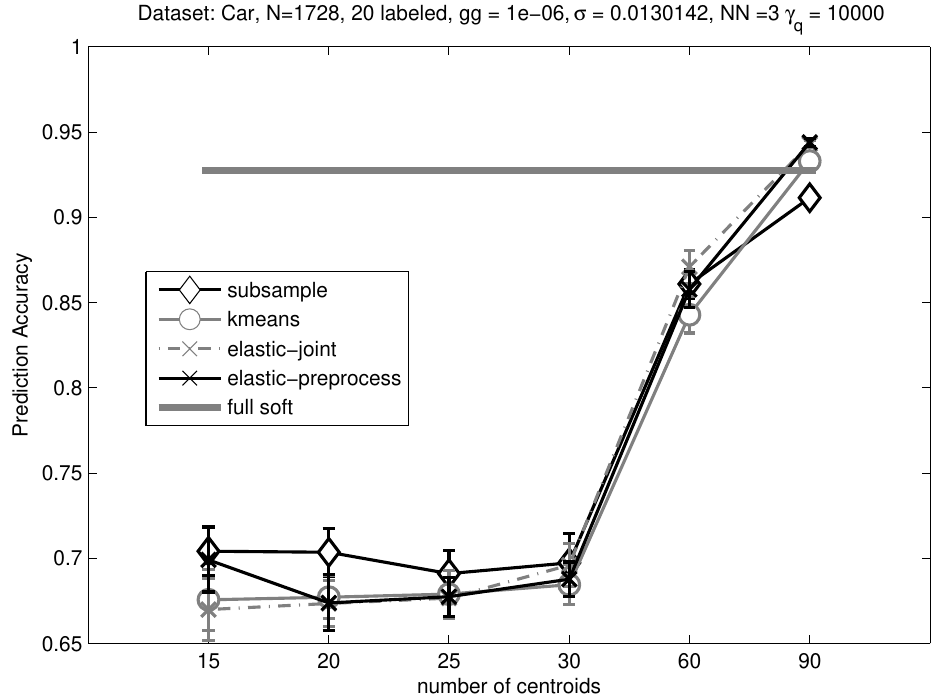}
\includegraphics[width=0.40\columnwidth,clip]{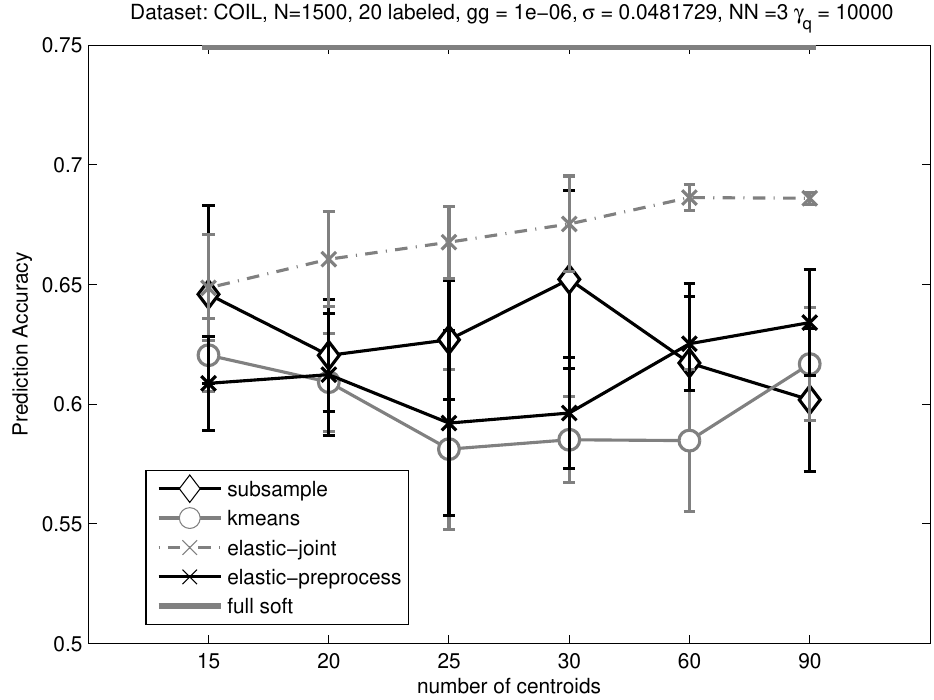}
\includegraphics[width=0.40\columnwidth,clip]{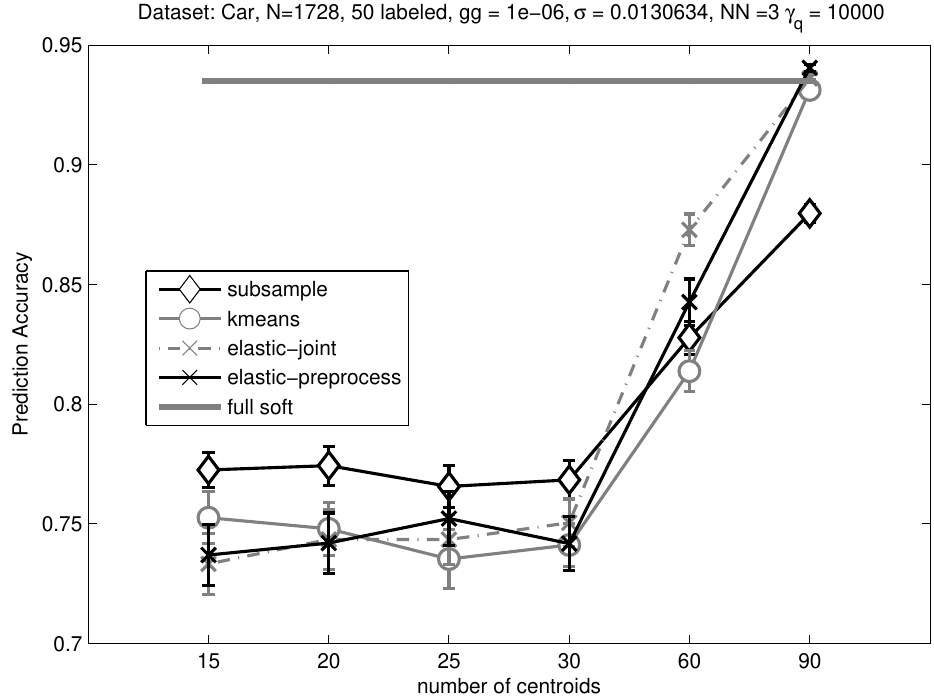}
\includegraphics[width=0.40\columnwidth,clip]{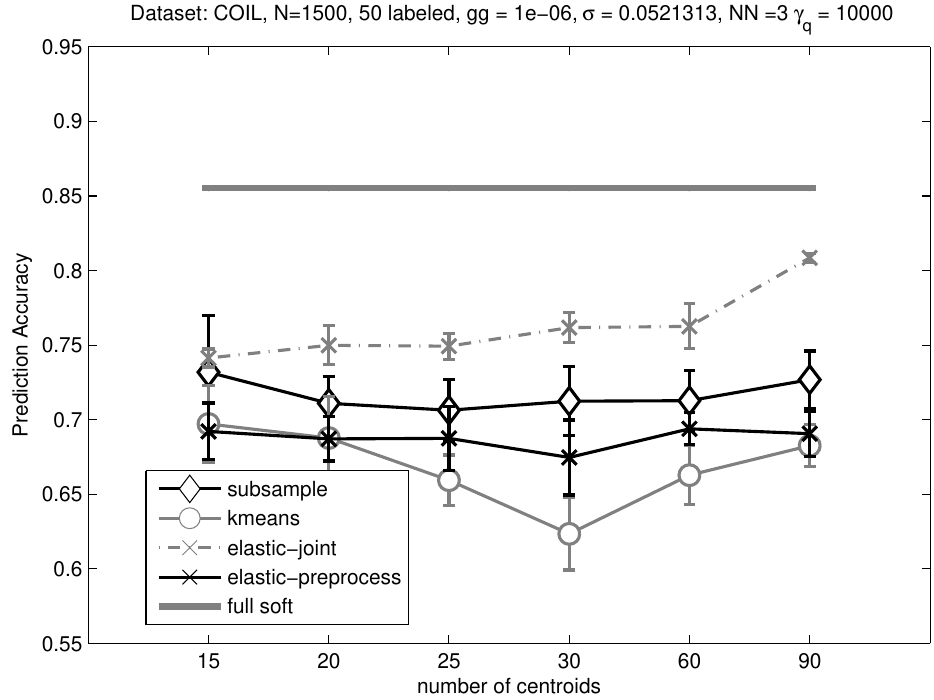}
\caption{Coil and Car datasets from UCI ML Repository}%
\label{fig:results}
\end{center}
\end{figure}

\input{uai_experiments}

\subsection{Parallel SSL}
\label{sec:ParellelSSL}

\begin{figure}%
\centering
\includegraphics[width=\columnwidth]{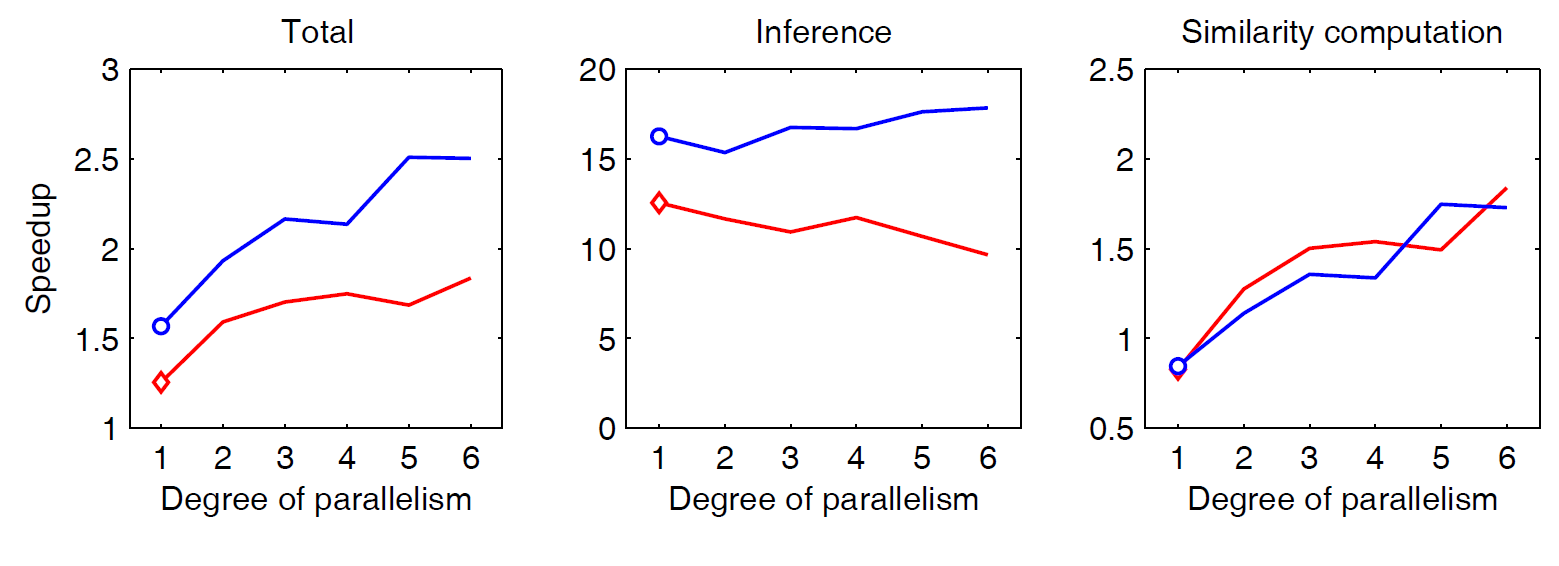}%
\caption{Speedups in the total, inference, and similarity computation times}%
\label{fig:pararesults}%
\end{figure}

In this experiment, we demonstrate how to speed up the online HS on a graph using
an additional structure and parallelization (Section~\ref{sec:ParallelMultiManifoldLearningMethods}). 
Therefore, we perform our experiments on an Intel Xeon workstation with six cores. 
The experimental setup is the same as in Section~\ref{sec:OnlineQuantizedSSLExperiments}.  
The number of labeled examples used for training models
of Person 1 and 13 (from Figure~\ref{fig:LoginFaceDataset}) is 5 and 6, respectively.
Figure~\ref{fig:pararesults} reports speedups due to decomposing the online HFS on 300 vertices into
$n_l$ smaller graphs of 50 vertices.  The plots correspond to Person 1 (red lines) and 13 (blue lines) in our dataset.
 The diamonds and
circles mark speedups that are obtained by the decomposition alone.
We observe two main trends. First, the decomposition
alone yields a modest speedup of 35\% on average. The speedup is due to 15 times faster
inference, which is a result of solving $n_l$ smaller systems of linear equations, each with
50 variables, instead of a bigger one with 300.
Second, we parallelize the online HS on the $n_l$ smaller graphs using OpenMP \cite{openmp2008openmp}.
The problem is trivially parallelizable because the graphs can be updated independently.
Figure~\ref{fig:pararesults} shows that as the number of used cores increases, the online HFS can be sped
up more than two times on average. The speedup is due to parallelizing the computation
of similarities $w_{ij}$ , which at this point consumes much more time than inference.
Finally, note that the proposed decomposition has almost no impact on the quality of
our solutions. For Person 1 and Person 13, the loss in accuracy is 2.5\% and 1\%, respectively.

\subsection{Conclusions}
\label{sec:PredictiveConclusions}

In this section, we have evaluated our algorithms for the semi-supervised
learning tasks.  Max-margin graph cuts algorithm learns max-margin graph cuts that
are conditioned on the labels induced by the harmonic function
solution. The approach is evaluated on a synthetic problem and three UCI ML repository
datasets, and we have showed that it usually outperforms manifold
regularization of SVMs.
Next, we have evaluated our joint optimization approach for graph quantization and label propagation.
 We have experimentally showed that this approach can lead to a significant gain in classification accuracy over the competing quantization approaches.
In the online SSL experiments we approximated
a similarity graph for a harmonic solution. This algorithm
significantly reduces the expense of the matrix
computation in the harmonic solution, while retaining
good control on the classification accuracy. Our
evaluation shows that a significant speedup for semi-supervised
learning can be achieved with little degradation
in classification accuracy. We have further approximated 
the computation by decomposing the graph into several
smaller graphs, thereby performing parallel multi-manifold learning.
With such a decomposition we were able to speed up the
computation even more with almost no loss in accuracy.

%% file: datasets_ssl.tex

\subsection{UCI ML Datasets}
\label{sec:UCIMLDatasets}

In this part we describe the datasets from the UCI ML Repository \cite{asuncion2007uci} that we used to test 
our semi-supervised algorithms. We used \emph{Digit}, \emph{Letter}, and \emph{Image segmentation} as
the benchmark datasets to compare our max-margin graph cuts to manifold regularization of SVMs. Moreover,
we used \emph{Digit} and \emph{Letter}, due to their small size, to compare 
the performance of our online semi-supervised algorithm on quantized graphs
to the performance of a full offline non-quantized harmonic solution. 
Finally, we used  \emph{COIL}, \emph{Car}, and \emph{SecStr}
as the benchmark datasets for large scale semi-supervised learning, as suggested
by \cite{chapelle2006semi-supervised}.

\subsubsection{Digit recognition}
This dataset was preprocessed by programs made available by NIST to extract normalized bitmaps of handwritten digits from a preprinted form. From a total of 43 people, 30 contributed to the training set and the remaining 13 contributed to the test set. 32x32 bitmaps are divided into non-overlapping blocks of 4x4, and the number of on pixels are counted in each block. This generates an input matrix of 8x8, where each element is an integer in the range 0--16. This reduces dimensionality and gives invariance to small distortions.

\subsubsection{Letter recognition}
The objective is to identify each of a large number of black-and-white rectangular pixel displays as one of the 26 capital letters in the English alphabet. The character images were based on 20 different fonts and each letter within these 20 fonts was randomly distorted to produce a file of 20,000 unique stimuli. Each stimulus was converted into 16 primitive numerical attributes (statistical moments and edge counts), which were then scaled to fit into a range of integer values from 0 through 15.

\subsubsection{Image segmentation}

The Segmentation dataset, created in 1990 by the Vision Group,
University of Massachusetts, consists of 2310 instances.
Each instance was drawn randomly from a database of seven outdoor images.
The image, a $3 \times 3$ region, was hand-segmented to create a
classification for each pixel.
The seven classes are brickface, sky, foliage, cement, window, path,
and grass.
Each of the 7 images is represented by 330 instances.
The extracted features are 19 continuous attributes that describe the
position of
extracted image, line densities, edges, and color values.

\subsubsection{COIL}
\label{sec:COIL}

The Columbia object image library (COIL-100) is a set of color images of 100 different objects taken from different angles (in steps of 5 degrees) at a resolution of $128 \times 128$ pixels \cite{nene1996columbia}.
We use the binary version of this dataset as preprocessed by \cite{chapelle2006semi-supervised}.

\subsubsection{Car}
\label{sec:Car}
The Car evaluation data set classifies cars into four categories using 6 features including buying price, number of doors, etc. We converted the Car dataset into a binary problem to classify first two vs.\,the second two car categories.

\subsubsection{SecStr}
The SecStr is a benchmark data set designed by \cite{chapelle2006semi-supervised} to investigate how far current methods can cope with large-scale application. The task is to predict the secondary structure of a given amino acid in a protein, based on a sequence window centered around that amino acid.

\smallskip

\noindent For the multi-class datasets we sometimes transformed them into a set of binary problems.

\subsection{Vision Datasets}

In this section, we describe several face recognition datasets that we recorded
to evaluate the performance of online semi-supervised learning algorithms 
on noisy real world data that involve outliers.

\label{sec:VisionDatasets}
The \emph{environment adaptation} dataset consists of faces of a single person, which are captured at various locations, such as a cubicle, a conference room, and a corner with a couch (Figure \ref{fig:videos}). The first four faces in the cubicle are labeled, and we want to learn a face recognizer for all locations. To test the sensitivity of the recognizer to outliers, we augmented the dataset with random faces. The \emph{office space} dataset (Figure \ref{fig:videos}) is multi-class and involves 8 people who walk in front of a camera and make funny faces. When a person shows up on the camera for the first time, we label four faces of the person. Our goal is to learn good face recognizers for all 8 people.  

\begin{figure}[t]
  \centering
  \includegraphics[width=3.2in, viewport=2.25in 4.5in 6.25in 6.5in]{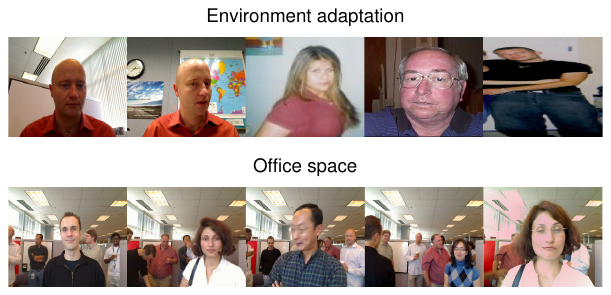}
  \caption{Snapshots from the environment adaptation and office space datasets}
  \label{fig:videos}
\end{figure}

Another vision dataset is a \emph{face-based authentication} dataset of 16 people (Figure~\ref{fig:LoginFaceDataset}).
The people try to log into a tablet PC with their face, while being recorded by its embedded 
camera. The data are collected at 10 indoor locations, which differ by backgrounds and
lighting conditions. In short, we recorded 20 10-second videos per person, each at 10 fps. 
Therefore, our face-based authentication dataset contains a total of $16 \times 20 \time 100 = 32 000$ images.
Faces in the images are detected using OpenCV \cite{bradski2000opencv}, converted to grayscale, resized
to $96\times96$, smoothed using the $3\times3$ Gaussian kernel, and equalized by the histogram of
their pixel intensities.

\begin{figure}[t]
  \centering
  \includegraphics[width=6.8in, viewport=0in 4.5in 8.5in 6.5in]{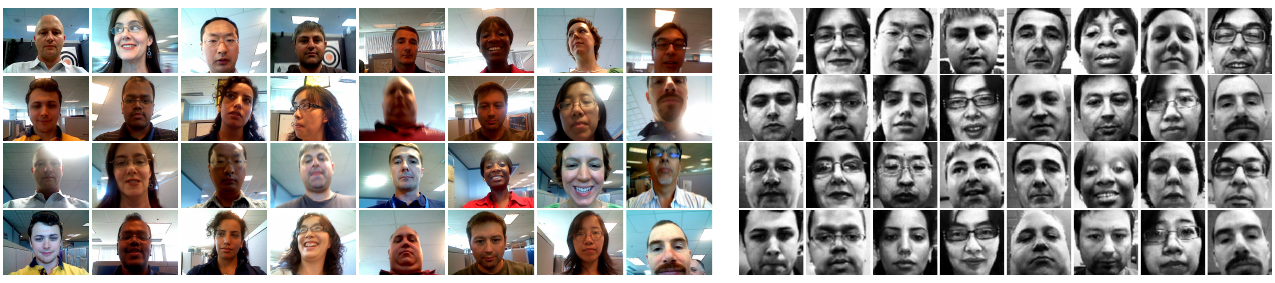}
  \caption{Face-based authentication dataset (left) and examples of labeled faces (right)}
  \label{fig:LoginFaceDataset}
\end{figure}

%% file: uai_experiments.tex
\subsection{Online Quantized SSL Experiments}
\label{sec:OnlineQuantizedSSLExperiments}

\begin{figure}[t]
  \centering
  \includegraphics[width=1.6in, viewport=3.25in 4in 5.25in 7in]{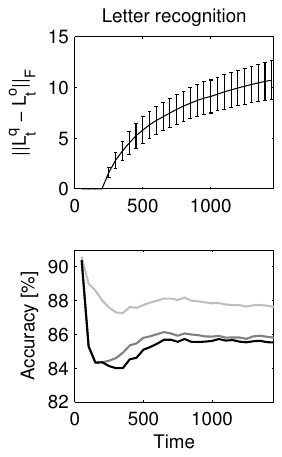}
  \includegraphics[width=1.6in, viewport=3.25in 4in 5.25in 7in]{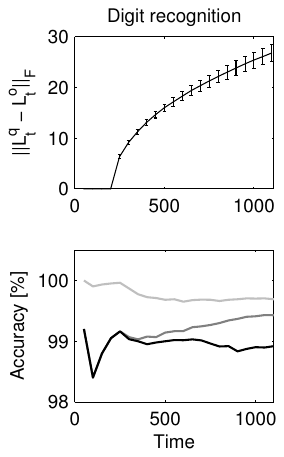}  
  \caption{UCI ML: Quality of approximation as a function of time}
	\label{fig:results UCI ML time}
\end{figure}

\begin{figure}[t]
  \centering
  \includegraphics[width=1.6in, viewport=3.25in 4in 5.25in 7in]{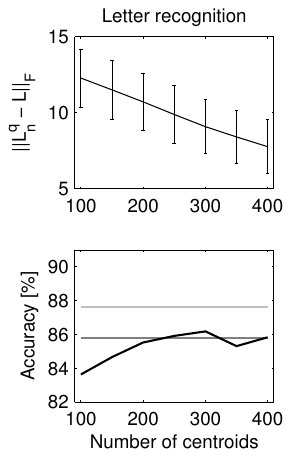}
  \includegraphics[width=1.6in, viewport=3.25in 4in 5.25in 7in]{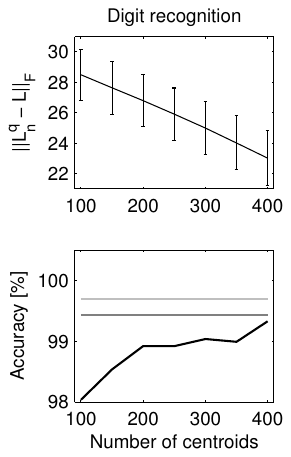}
  \caption{UCI ML: Quality of approximation as a function of number of centroids}
  \label{fig:results UCI ML centroids}
\end{figure}

The experimental section is divided into two parts. In the first part,
we evaluate our online learner (Figure \ref{fig:online quantized HFS})
on UCI ML repository datasets (Section~\ref{sec:UCIMLDatasets}). In the second
part, we apply our learner to solve two face recognition problems. In
all experiments, the multiplicative parameter $\Rmultiplier$ of the
$k$-centers algorithm is set to 1.5\,.

\subsubsection{UCI ML Repository Experiments}
\label{sec:UCI ML repository experiments UAI}

In the first experiment, we study the online quantization error
$\normw{\Loq_t - \Lo_t}{F}$ and its relation to the HS
on the quantized graphs $\Woq_t$. This experiment is performed on two
datasets from the UCI ML repository: letter and optical digit
recognition. The datasets are converted into a set of binary
problems, where each class is discriminated against every other
class. 
The similarity weights are computed as $w_{ij} = \exp[- \normw{\bx_i - \bx_j}{2}^2 / (2 p \sigma^2)]$, where $p$ is the number of features and $\sigma$ denotes the mean of their standard deviations. 
Our results are averaged over 10 problems from each dataset
and shown in Figures \ref{fig:results UCI ML time} and
\ref{fig:results UCI ML centroids}.

In Figure \ref{fig:results UCI ML time}, we fix the number of
centroids at $k = 200$ and study how the quality of our solution
changes with the learning time $t$. 
The upper plots show the difference between the normalized Laplacian $\Lo_t$ and its \mbox{approximation $\Loq_t$} at time $t$. The bottom plots show the cumulative accuracy of the harmonic solutions on $W$ (light gray lines), $\Wo_t$ (dark gray lines), and $\Woq_t$ (black lines) for various times $t$.
Two trends are
apparent. \mbox{First, as} time $t$ increases, the error
$\normw{\Loq_t - \Lo_t}{F}$ slowly levels off. Second, the accuracy of
the harmonic solutions on $\Woq_t$ changes little with $t$. These
trends indicate that a fixed number of centroids $k$ may be sufficient
for quantizing similarity graphs that grow with time. In Figure
\ref{fig:results UCI ML centroids}, we fix the learning time at $t =
n$ and vary the number of centroids $k$. 
The upper plots show the difference between the normalized Laplacian $L$ and its approximation $\Loq_n$. The difference is plotted as a function of the number of centroids $\ng$. The bottom plots compare the cumulative accuracy of the harmonic solutions up to time $n$ on $W$ (light gray lines), $\Wo_t$ (dark gray lines), and $\Woq_t$ (black lines).
Note that as $k$ increases,
the quantization error decreases and the quality of the solutions on
$\Woq_t$ improves. This trend is consistent with the theoretical
results in our work (Section~\ref{quantization_bound}).


\subsubsection{Face recognition}
\label{sec:face recognition experiments}

In the second experiment, we evaluate our learner on 2 face recognition datasets: office space and environment adaptation. 
(Section~\ref{sec:VisionDatasets}).

The similarity of faces $\bx_i$ and $\bx_j$ is computed as $w_{ij} \! = \! \exp\left[- d(\bx_i, \bx_j)^2/2 \sigma^2\right]$, where $\sigma$ is a heat parameter, which is set to $\sigma = 0.025$, and $d(\bx_i, \bx_j)$ is the distance of the faces in the feature space. To make the \mbox{graph $W$} sparse, we treat it as an $\eps$-neighborhood graph and set $w_{ij}$ to 0 when $w_{ij} < \eps$. The scalar $\eps$ is set as $\eps = 0.1 \gamma_g$. As a result, the lower the regularization \mbox{parameter $\gamma_g$,} the higher the number of edges in the graph $W$ and our learner extrapolates to more unlabeled examples. If an example is disconnected from the rest of the graph $W$, we treat it as an outlier and neither predict the label of the example, nor use it to update the quantized graph. This setup makes our algorithm robust to outliers and allows for controlling its precision and recall by a single parameter $\gamma_g$. In the rest of the section, the \mbox{number of} centroids $k$ is fixed at 500. More details are provided in Section~\ref{sec:UCI ML repository experiments}.

In Figure \ref{fig:face recognition results}, we compare our online algorithm to online semi-supervised boosting \cite{grabner2008semi-supervised} and a $k$-NN classifier, which is trained on all labeled faces. 
The recognizers are trained by a NN classifier (gray lines with circles), online semi-supervised boosting (thin gray lines), and our online learner (black lines with diamonds). The plots are generated by varying the parameters $\eps$ and $\gamma_g$. From left to right, the points on the plots correspond to decreasing values of the parameters. Online semi-supervised boosting is \mbox{performed on} 500 weak NN learners, which are sampled at random from the whole environment adaptation dataset (solid line), and its first and last quarters (dashed line).
The algorithm of \cite{grabner2008semi-supervised} is modified to allow for a fair comparison to our method. First, all weak learners have the nearest neighbor form $h_i(\bx_t) = \I{w_{it} \geq \eps}$, where $\eps$ is the radius of the neighborhood. Second, outliers are modeled implicitly. The new algorithm learns a regressor $H(\bx_t) = \sum_i \alpha_i h_i(\bx_t)$, which yields $H(\bx_t) \! = \! 0$ for outliers and $H(\bx_t) \! > \! 0$ when the detected face is recognized.

Figure \ref{fig:face recognition results}\textbf{a} clearly shows that our learner is better than the nearest neighbor classifier. Furthermore, note that online semi-supervised boosting yields results as good as our method when given a good set of weak learners. However, future data are rarely known in advance, and when the weak learners are chosen using only a part of the dataset, the quality of the boosted results degrades significantly (Figure \ref{fig:face recognition results}\textbf{a}). In comparison, our algorithm constantly adapts its representation of the world. How to incorporate a similar adaptation step in online semi-supervised boosting is not obvious.

In Figure \ref{fig:face recognition results}\textbf{b}, we evaluate our learner on an 8-class face recognition problem. Despite the fact that only 4 faces of each person are labeled, we can identify people with 95 percent precision and 90 percent \mbox{recall. In general,} our precision is 10 percent higher than the \mbox{precision of} the NN classifier at the same recall level.

\begin{figure}
  \centering
  \includegraphics[width=3.2in, viewport=2.25in 4.5in 6.25in 6.5in]{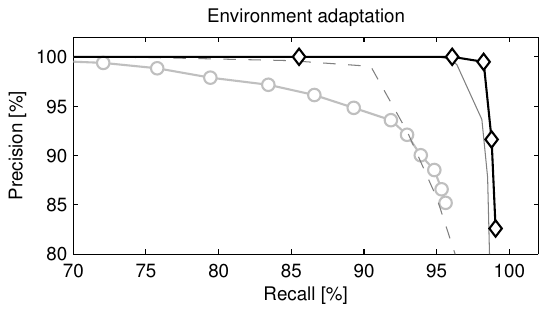}
  \vspace{0.05in}
  \includegraphics[width=3.2in, viewport=2.25in 4.5in 6.25in 6.5in]{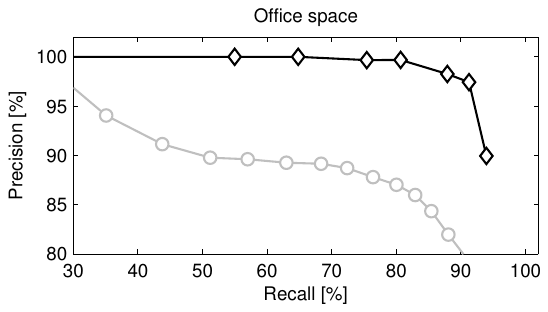} \\
  \textbf{a)} \hspace{2.95in} \textbf{b)}
  \caption{Comparison of 3 face recognizers on 2 face recognition datasets}
  \label{fig:face recognition results}
\end{figure}

%% file: sec_evaluations.tex
\section{Evaluations of Conditional Anomaly Detection Methods}
\label{sec:EvaluationOfAnomalyDetection}

In this section we present the experiments using the CAD methods
from Chapter~\ref{sec:ConditonalAnomalyDetection}. 
In all our experiments, we focus on the conditional anomalies in the class labels with respect to the features.
In general, in the whole field of anomaly detection and in medical domain especially,
the evaluation is extremely challenging. Most of time, it is subjective. 
The most veracious evaluations would have human experts judging the goodness of the methods.
Since this is a very expensive way, most researchers resort to some surrogate measures.
In the area of mislabel detection, the most common surrogate measure 
is to change the labels of a fraction of the dataset and observe how many of those were detected as mislabeled.
The problem with this measure is that the anomalies in the real life datasets are really \emph{sampled} randomly.
We describe the data we use in Sections~\ref{sec:SyntheticDatasets} and~\ref{sec:MARS}
and the algorithms we use for the comparison in Section~\ref{sec:AlgorithmsForComparison}.
We then provide two kinds of evaluations: the evaluation when the ground truth 
is known or can be computed (Section~\ref{sec:EvaluationOfCADWithKnownGroundTruth})
and then the evaluation with human experts (Section~\ref{sec:ExpertAssesedClinicallyUsefulAnomalies}).

\input{datasets_cad}

\subsection{Algorithms for Comparison}
\label{sec:AlgorithmsForComparison}

In this section we review the CAD algorithms chosen for the comparison
with our CAD methods.

\subsubsection{Discriminative SVM anomaly detection}
\label{sec:SVMBaseLineMethod}
For the baseline method we use an SVM based method \cite{valko2008conditional,hauskrecht2010conditional}, 
that computes an anomaly score from the distance from the hyperplane. SVM
\cite{vapnik1995nature,burges1998tutorial} is a discriminative method
that learns the  decision boundary as 
 $$\bw^T\bx + w_0 = \sum_{i\in SV} \hat\alpha_iy_i(\bx_i^T\bx) + w_0,$$
where only samples in the support vector set ($SV$) contribute to the computation of the decision boundary. 
To support classification tasks, the projection defining the decision boundary is used to 
determine the class of a new example. That is, if the value  $$\bw^T\bx + w_0 \ge 0$$   
is positive, then $C(\bx)$ belong to one class, but if it is negative it belongs to the other class. 
However, for conditional anomaly detection we use the projection itself for the positive class
and the negated projection for the negative class to measure the deviation:

$$d(y|\bx) = y(\bw^T\bx + w_0), \,\, \mathrm{where} \,\, y\in \{-1, 1\}$$  

\noindent In other words, the smaller the projection is the more likely the example is anomalous. 
We note that the negative projections correspond to misclassified examples.

\subsubsection{One-class SVM}
\label{sec:OneClassSVM}

As an example of a classical anomaly detection method converted to 
the CAD method we compare to the one-class SVM\cite{manevitz2002one-class}.
Originally proposed in  \cite{scholkopf1999estimating}, the method only needs positive examples to learn the margin.
The idea is that the space origin (zero) is treated as the only example of the `negative' class. 
In that way the learning essentially estimates the support of the distribution. The data that do not fall into this support have negative projections and can be considered anomalous. In our scenario, we will learn one one-class SVM for
each of the classes and based on the test label (which is known) we calculate the anomaly score. 
The more negative the score the higher the rank of the anomaly.

\subsubsection{Quadratic discriminant analysis}
In the quadratic discriminant analysis (QDA) model \cite{hastie2001elements}, we model each class by a multivariate Gaussian,
and the anomaly score is the class posterior of the opposite class.

\subsubsection{Weighted NN}
We also use the weighted $k$-NN approach \cite{hastie2001elements} that uses the same weight metric $W$ as SoftHAD, but relies on
only on the labels in the local neighborhood and does not account for the manifold structure.

\subsubsection{Parameters for the graph-based algorithms}
\label{sec:Parameters}

The similarity weights are 	computed as 
$$w_{ij} = \exp\left[- \frac{||\bx_i - \bx_j||_{2,\psi}^2 }{p \sigma^2}\right],$$ 

\noindent  where $p$ is the number of features and $\psi = (p \times 1)$ is a weighing of the features based on their discriminative power.
Including $p$ in the weight metric allows us to control the connectivity of the graph.
Next, $\sigma$ is chosen so that the graph is reasonably sparse \cite{luxburg2007tutorial}.
We follow \cite{valizadegan2007kernel} and chose $\sigma$ as 10\% of the mean of empirical standard deviations of all features.
Based on the experiments, our algorithm is not sensitive to the small perturbations of $\sigma$; what is important
is that the graph does not become disconnected by having all edges of several nodes with weights close to zero.

\begin{figure}
\begin{center}
\includegraphics[width=0.4\columnwidth]{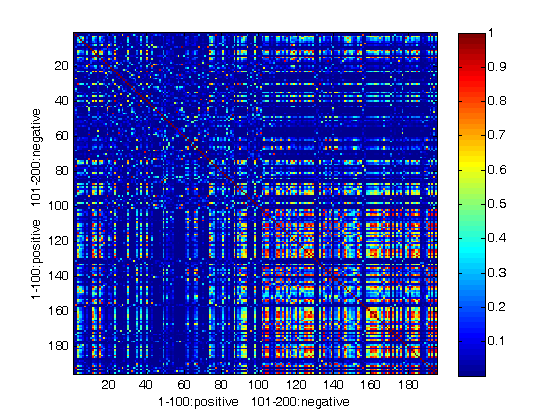}
\caption{The weight matrix for 100 negative and 100 positive cases of HPF4 order}%
\label{fig:W_HPF4_heat5_100x100}%
\end{center}
\end{figure}

For the feature weights $\psi$ for PCP data we used the univariate Wilcoxon (ROC) score \cite{hanley1982meaning},
which is typically used for medical data \cite{hauskrecht2006fundamentals}.
Since this score ranges from 0.5 to 1, we modify the score by subtracting 0.5
and raising it to a power of 5 to make the differences between the weights larger.
We use to same metric for the weighted NN anomaly detection from Section~\ref{sec:AnalysisOfConditionalAnomalyDetection}.
We vary the regularization parameter as $\lambda \in \{ 10^{-5}, 10^{-9},\dots, 10^{5} \}$.  
Figure~\ref{fig:W_HPF4_heat5_100x100} illustrates this metric on a binary classification
task for the heparin induced thrombocytopenia (a life threatening condition that may occur with
prolonged heparin treatments). One hundred negative and one hundred 
positive cases and their mutual similarities are shown. We can see that 
the positive cases are much closer to each other (bottom right part in Figure~\ref{fig:W_HPF4_heat5_100x100}) than the negatives. 
For the other datasets, we used a uniform $\psi$ for all features.

\subsection{Evaluation of CAD with Known Ground Truth}
\label{sec:EvaluationOfCADWithKnownGroundTruth}

\subsubsection{CAD on synthetic datasets with known distribution}
\label{sec:ConnectivityADOnSyntheticDatasetsWithKnownUnderlyingDistribution}

The evaluation of a CAD is a very challenging
task when the true model is not known. Therefore,  we first
evaluate and compare the results of different CAD methods
on three synthetic datasets (D1, D2, and D3) with known underlying distributions
that let us compute the true anomaly scores (Section~\ref{sec:MixturesOfGaussians}).
Then, we show the advantage of regularizing a discriminative approach on a synthetic dataset. 
We will use a 2D synthetic dataset, where we can demonstrate the ability 
to tackle fringe and isolated points as described in Section~\ref{sec:Regularization}.

For each experiment we sample the datasets 10 times.  After the sampling, we
randomly switch the class labels for three percent of examples.
We then calculate the true anomaly score as
$P(y \neq y_i|\bx_i)$, reflecting how anomalous the label of the example is with respect to the true model.

Each of the methods outputs a score which orders the examples according to the belief of the anomalous labeling.
For each of the CAD methods, we assess how much this ordering is consistent
with the ordering of the true anomaly score. In particular, we calculated the
area under the receiver operating characteristic (AUROC), which is inversely
proportional to the number of swaps between the ordering induced by the evaluated method and the true ordering.

\input{kdd_syn_table}

\noindent Table~\ref{tab:syn} compares the AUROCs of the experiment for all methods for
1000 samples per dataset. The results demonstrate that our $\lambda$-RWCAD
method outperforms the weighted $k$-NN, the one-class SVM, and the discriminative SVM with the RBF kernel\footnote{We also evaluated the linear versions of SVM and the one-class SVM, but the
results were inferior to the ones with the RBF kernel.},
and it is comparable to our label propagation SoftHAD algorithm on D2 and D3.
SoftHAD seems to be the best choice overall because it takes advantage of both local and global consistency.
However, it is computationally more expensive.

In the next experiment we evaluate the scalability of the graph-based methods
as we increase the number of examples. All of the graph methods were given the
same graph (with the same weight matrix). Figure~\ref{fig:graph_time_comparison}
compares the running times of these algorithms.
We see that while the running time of the SoftHAD algorithm
becomes prohibitive once the number of examples 
gets into thousands, our algorithm scales similarly to the $k$-NN method.
Figure~\ref{fig:graph_time_comparison} also shows the time spent
in constructing the graph from the data, which is the same among 
all the graph-based methods. Observe that both the weighted $k$-NN and our
$\lambda$-RWCAD algorithm take very little time over the necessary graph construction time
to do their calculations. 

\begin{figure}%
\centering
\includegraphics[width=0.6\columnwidth]{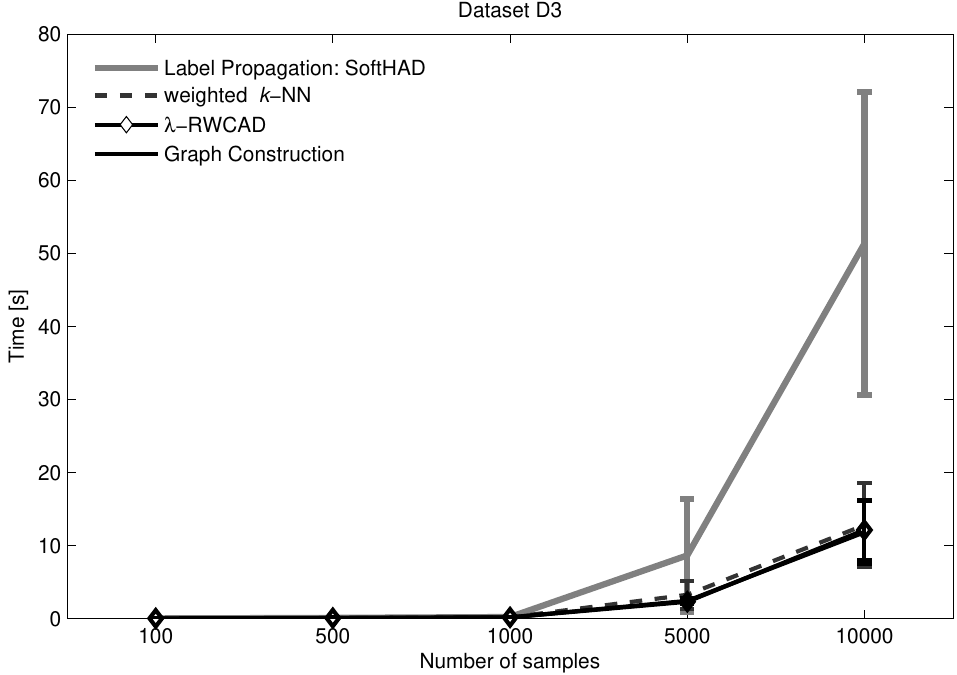}%
\caption{Computation time comparison for the three graph-based methods}%
\label{fig:graph_time_comparison}%
\end{figure}

\subsubsection{CAD on UCI ML datasets with ordinal response variable}
\label{sec:CADOnUCIMLDatasetsWithOrdinalResponseVariable}

We also evaluated our method on the three UCI ML datasets~\cite{asuncion2007uci}, for 
which an ordinal response variable was available to calculate the true anomaly score. In particular, we
selected 1) \emph{Wine Quality} dataset with the response variable \emph{quality},
2) \emph{Housing} dataset with the response variable \emph{median value of owner-occupied homes},
and 3) \emph{Auto MPG} dataset the response variable \emph{miles per gallon}.
In each of the datasets we scaled the response variable $y_r$ to the $[-1,+1]$ interval
and set the class label as $y := y_r \geq 0$. 
As with the synthetic datasets, we randomly switched the class labels for 
three percent 
of examples.
The true anomaly score was computed as the absolute difference 
between the original response variable $y_r$ and the (possibly switched) label.
Table~\ref{tab:uci} compares the agreement scores  (over 100 runs) to the true score for all methods on 
(2/3, 1/3) train-test split. The results in bold show when a method significantly outperforms the rest.
 Again, we see that SoftHAD either performed the best
or was close to the best method.

\begin{table}[htbp]
  \centering
 \begin{tabular}{rc|c|c}
    \addlinespace
    \toprule
        & \phantom{a}\textbf{Wine Quality}\phantom{a}&\phantom{a} \textbf{Housing} \phantom{a}& \phantom{a}\textbf{Auto MPG}\phantom{a} \\
    \midrule
	   \emph{QDA}& \textbf{75.1\%} (1.3) & 56.7\% (1.5) & 65.9\% (2.9) \\
   \emph{SVM}  & 75.0\% (9.3) & 58.5\% (4.4) & 37.1\% (8.6) \\
    \emph{one-class SVM}  & 44.2\% (1.9) & 27.2\% (0.5) & 50.1\% (3.5) \\
    \emph{weighted $k$-NN}         & 67.6\% (1.4) & 44.4\% (2.0) & 61.4\% (2.3) \\
		\emph{SoftHAD}      & \textbf{74.5\%} (1.5) & \textbf{71.3\% }(3.2) & \textbf{72.6\%} (1.7) \\	
    \bottomrule
    \end{tabular}%
		\vskip 0.25cm
		\caption{Mean anomaly agreement score and variance on 3 UCI ML datasets}
  \label{tab:uci}%
\end{table}%

%


\subsubsection{CAD on Core dataset with fringe points}
\label{sec:CADOnCoreDatasetWithFringePoints}

\begin{figure}%
\begin{center}
\includegraphics[width=0.75\columnwidth]{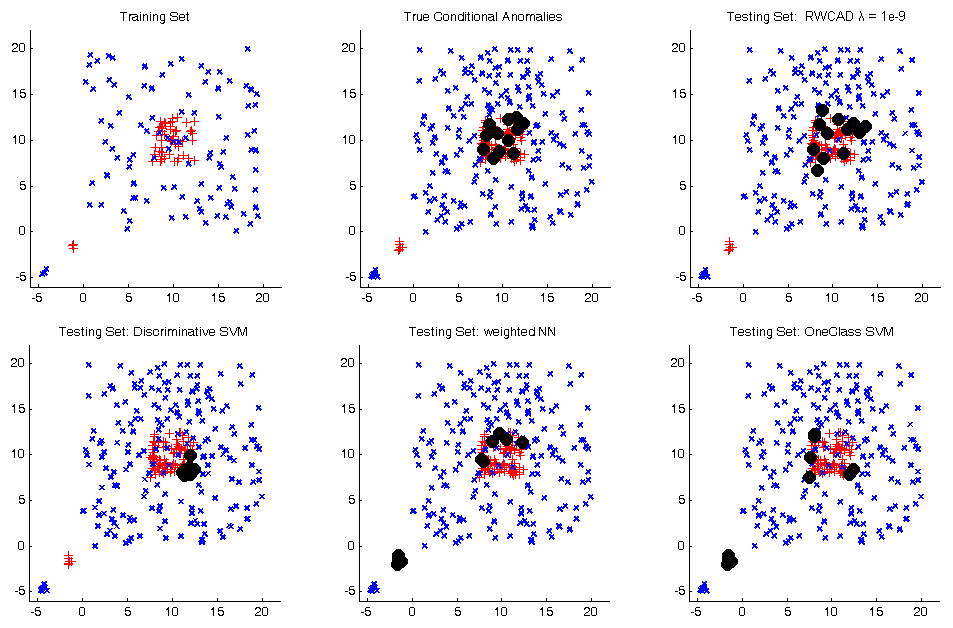}%
\caption{Conditional anomaly detection on a synthetic \emph{Core} dataset}%
\label{fig:core}%
\end{center}
\end{figure}

In this part we tested our CAD method on the synthetic Core dataset (Section~\ref{sec:CoreDataset}). 
Besides the one-class SVM (Section~\ref{sec:OneClassSVM}), we also compared to the weighted $k$-NN described in 
Section~\ref{sec:AnalysisOfConditionalAnomalyDetection}
and to the cross-outlier method \cite{papadimitriou2003cross-outlier} described in Section \ref{sec:ConditionalAnomalyDetection}.

In Figure \ref{fig:core}, the training data consists of a bigger square of 100 uniformly distributed points (blue `x'),  a smaller square of 50 uniformly distributed points (red `+'), and 2 small groups 
of points (3 points from each class). The testing dataset is twice as big sampled from the same distribution.
The big black dots display true conditional anomalies and the top 12 highest ranked conditional anomalies for 
1) our $\lambda$-RWCAD method, 2) the discriminative SVM anomaly detection,
3) the weighted $k$--NN, and 4) the one-class SVM learned for both of the classes.

 The cross-outlier method \cite{papadimitriou2003cross-outlier}
was able to find all of the conditional anomalies in the middle square, but also 
declared many \emph{fringe} points (points at the outer boundary of the bigger square) 
as anomalous (see Figure 2, middle row in \cite{papadimitriou2003cross-outlier}).
Although the authors claim that the fringe points are `clearly different from the rest of the points' \cite{papadimitriou2003cross-outlier},
we prefer methods that only find anomalously labeled instances. 

In Figure \ref{fig:core}, we also show the top 12 highest scored anomalies from the 4 competing methods.
The discriminative SVM anomaly detection (Figure \ref{fig:core}, bottom left) could only detect the fringe points
from the smaller square, since the anomaly score there corresponds to 
the most incorrectly classified testing points. Next, the objective of the one-class SVM is to detect the points with minimal support.
In Figure \ref{fig:core}, bottom right, we see that the one-class SVM ranked with the highest score 
the fringe points of the smaller square and one of the tiny squares.
The weighted $k$-NN (Figure \ref{fig:core}, bottom middle) detects half of the true anomalies,
but also falsely detects one of the tiny squares as anomalous.
Our method (Figure \ref{fig:core}, top right) avoids such a mistake 
due to the regularization. Although the results of our method
do not completely match with the truth, the 3 points detected outside the smaller square
are in its vicinity.

\subsubsection{Conclusions}
\label{sec:ConclusionsSurrogate}

We showed how we use regularization to avoid the detection of isolated and the fringe points.
In general, the advantage of the CAD approach over knowledge-based error detection approaches is that the
method is evidence-based, and hence requires little or no 
input from a domain expert.

\subsection{Evaluation of Expert Assessed Clinically Useful Anomalies}
\label{sec:ExpertAssesedClinicallyUsefulAnomalies}


\subsubsection{Pilot study in 2009}
\label{sec:PilotStudyIn2009}

The aim of the study \cite{hauskrecht2010conditional}  was to test the hypothesis that clinical anomalies lead to good clinical alerts.

\noindent \textbf{Learning anomaly detection models} The training set was used to build three types of anomaly detection models: 1) models for detecting unexpected lab-order omissions, 2) models for detecting unexpected medication omissions, and 3) models for detecting unexpected continuation of medications (commissions).

\noindent  \textbf{Selection of alerts for the study}  The alerts for the evaluation study were selected as follows. We first applied all the above anomaly detection models to matching patient instances in the test set. The following criteria were then applied. First, only models with AUC of 0.68 or higher (to limit the number of models to those with a good predictive performance) were considered. This means that many predictive models built did not qualify and were never used. Second, the minimum anomaly score for all alert candidates had to be at least 0.15. Third, for each decision, only the top 125 anomalies and the top 20 alerts obtained from the test data were considered as alert candidates. This lead to 3,768 alert candidates, from which we selected 222 alerts for 100 patients, such that 101 alerts were lab-omission alerts, 55 were medication-omission alerts, and were 66 medication-commission alerts. The cases were selected such that their alert scores cover the whole range of alert scores, biased towards the more anomalous cases.  Figure \ref{fig:dist2009} shows the distribution of alerts in the study according to the alert score. 
 
\begin{figure}
\begin{center}
\includegraphics[width=0.6\columnwidth]{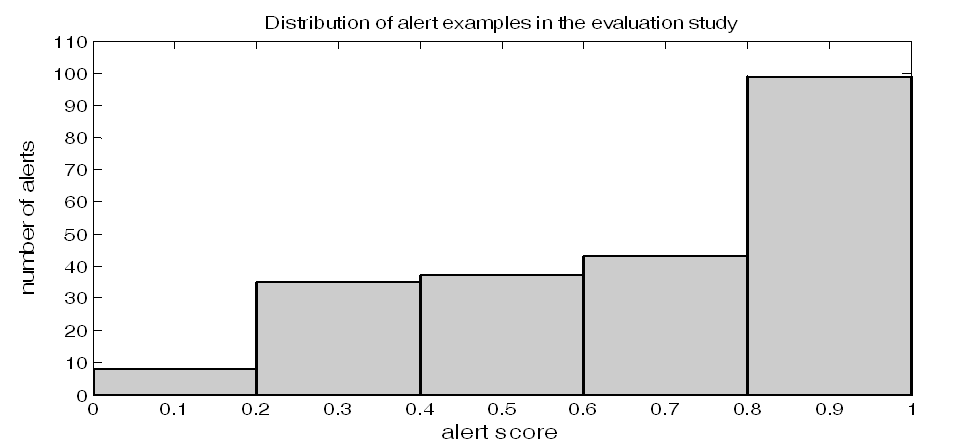}%
\caption{Histogram of alert examples in the study according to their alert score}%
\label{fig:dist2009}%
\end{center}
\end{figure}
 
\noindent  \textbf{Alert reviews.} The alerts selected for the study were assessed by physicians with expertise in post-cardiac surgical care. The reviewers 1) were given the patient cases and model-generated alerts for some of the patient management decisions, and 2) were asked to assess the clinical usefulness of these alerts. We recruited 15 physicians to participate in the study, of which 12 were fellows and 3 were faculty from the Departments of Critical Care Medicine and Surgery. The reviewers were divided randomly into five groups, with three reviewers per group, for a total of 15 reviewers. Overall, each clinician made assessments of 44 or 45 alerts, generated for 20 different patients. The total number of alerts reviewed by all clinicians was 222 and included: 101 lab omission alerts, 55 medication omission alerts, and 66 medication commission alerts.  The survey was conducted over the Internet using a secure web-based interface 
\cite{post2008temporal}.

\noindent  \textbf{Alert assessments. }The pairwise kappa agreement scores for the groups of three ranged from 0.32 to 0.56.  We use the majority rule to define the gold standard. That is, an alert was considered to be useful if at least two out of three reviewers found it to be useful. Out of the 222 alerts selected for the evaluation study, 121 alerts were agreed upon by the panel (via the majority rules) as a useful alert. 

\noindent  \textbf{Analysis of clinical usefulness of alerts. }  We analyze the extent to which the alert score from a model was predictive of it producing clinically important alerts.  Figure \ref{fig:alertscore2009} summarizes the results by binning the alert scores (in intervals of the width of 0.2, as in Figure~\ref{fig:dist2009}) and presenting the true alert rate per bin. The true alert rates vary from 19\% for the low alert scores to 72\% for the high alert scores, indicating that higher alert scores are indicative of higher true alert rates. This is also confirmed by a positive slope of the line in Figure \ref{fig:alertscore2009}, which is obtained by fitting the results via linear regression and the results of the ROC analysis.  All alerts reviewed were ordered according to their alert scores, from which we generated an ROC curve. The AUC for our alert score was 0.64. This is statistically significantly different from 0.5, which is the value one expects to see for random or non-informative orderings. Again, this supports that higher alert scores induce better true alert rates.  Finally, we would like to note that alert rates in Figure 4 are promising and despite alert selection restrictions, they compare favorably to alert rates of existing clinical alert systems \cite{schedlbauer2009what,bates2003ten}. 

\begin{figure}%
\begin{center}
\includegraphics[width=0.6\columnwidth]{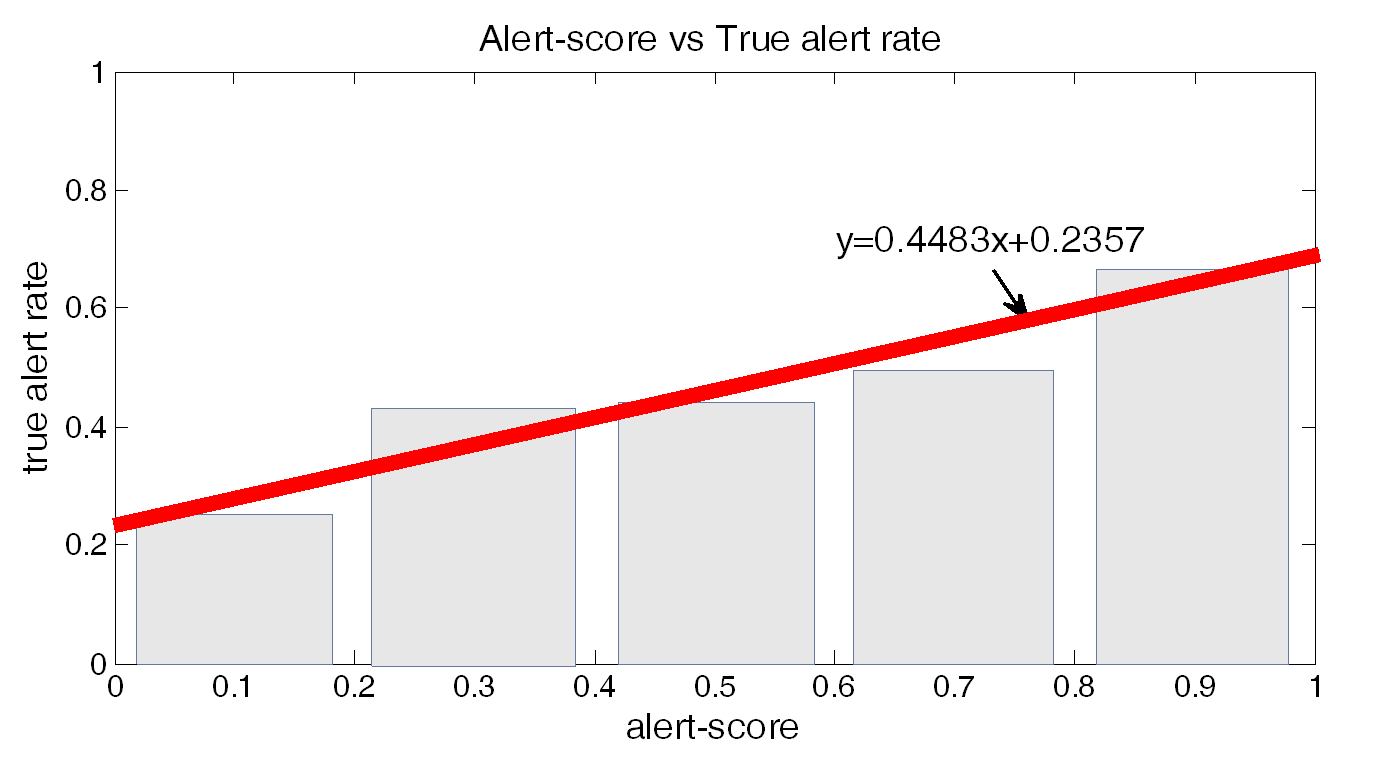}%
\caption{The relationship between the alert score and the true alert rate}%
\label{fig:alertscore2009}%
\end{center}
\end{figure}

\subsubsection{Soft harmonic anomaly detection}
\label{sec:SoftHarmonicAnomalyDetection}

For this experiment, we use the PCP dataset (Section~\ref{sec:MARS})
and reuse the human expert evaluations from Section~\ref{sec:PilotStudyIn2009}.
We compute the anomaly scores according to (Section~\ref{sec:ConditionalAnomalyDetectionWithSoftHarmonicFunctions})
%
%

\noindent {\bf Scaling for multi-task anomaly detection} 
So far, we have described CAD only for a single task (anomaly in a single label). 
In this dataset, we have 749 binary tasks that correspond
to 749 different possible orders of lab tests or medications.
In our experiments, we compute the CAD score for each task independently.
Figure \ref{fig:scaling2} shows the CAD scores for two of them.
CAD scores close to 1 indicate that the order should be done, while
the scores close to 0 indicate the opposite. 
The ranges for the anomaly scores can vary
among the different tasks, as one can notice in Figure \ref{fig:scaling2}.
The scores for the top and bottom task range from 0.1 to 0.9
and from 0.25 and 0.61, respectively. The arrow in both cases
points to the scores of the evaluated examples, both with negative labels. Despite the lower score for the bottom
task, we may believe that it is more anomalous, because it is
more extreme within the scores for the same task.
However, we want to output an anomaly score, which is comparable among the different tasks
so we can set a unified threshold when the system is deployed in practice.
Another reason for comparable scores is that we can have, for instance, 2 models each alerting that a certain medication was omitted.
Nevertheless, omitting one of the medications can be more severe than the other (eg.\,~antibiotics vs.\,vitamins).
To achieve the score comparability, we 
propose a simple approach, where we take the minimum and the maximum score obtained for the training set
and scale all scores for the same task linearly so that the score after scaling ranges from 0 to 1.

\begin{figure}
\begin{center}
\includegraphics[width=0.6\columnwidth,clip, viewport=45 308 589 463]{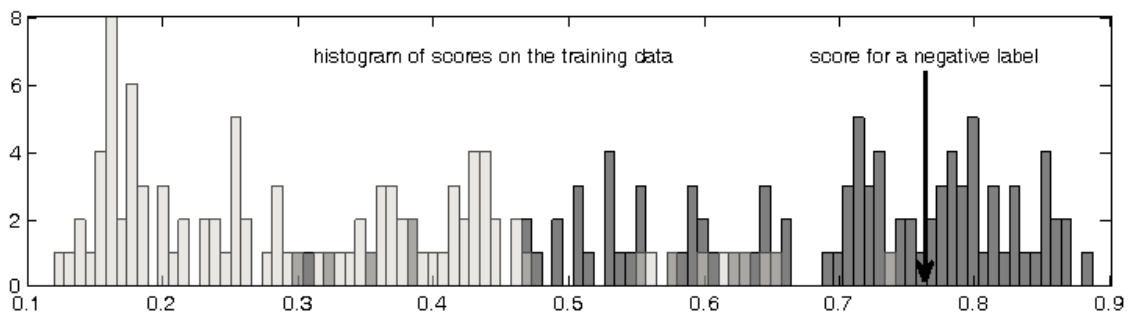}
\includegraphics[width=0.6\columnwidth,clip, viewport=58 316 558 471]{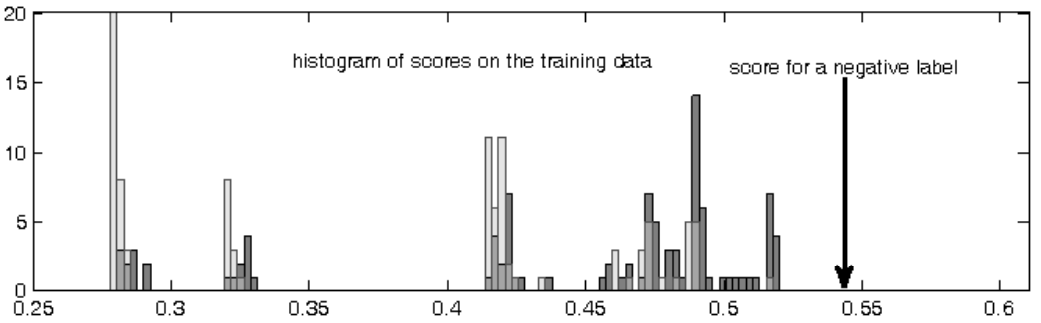}
\caption{Histogram of anomaly scores for 2 different tasks}
\label{fig:scaling2}%
\end{center}
\end{figure}

In Figure~\ref{fig:samplesize}, we fix $\gamma_g = 1$ and vary the number of examples we sample from the training set to construct the similarity graph, and also compare it to the weighted $k$-NN. The error bars show the variances over 10 runs.
Notice that both of the methods are not too sensitive to the graph size.
This is due to the multiplicity adjustment for the backbone graph (Section~\ref{sec:ConditionalAnomalyDetectionWithSoftHarmonicFunctions}).
Since we use the same graph both for SoftHAD and the weighted $k$-NN, we anticipate that we are able to outperform
the weighted $k$-NN due to the label propagation over the data manifold and not only within the immediate neighborhood.
In Figure~\ref{fig:gg}, we compare SoftHAD to the CAD using SVM with an RBF kernel for different regularization settings.
We sampled 200 examples to construct a graph (or train an SVM) and varied the $\gamma_g$ regularizer (or cost $c$ for SVM).
We outperform the SVM approach over the range of regularizers.
The AUC for the one-class SVM with an RBF was consistently below 55\%, so we do not show it in the figure.
We also compared the two methods with scaling adjustment for this multi-task problem (Figure~\ref{fig:gg}).
The scaling of anomaly scores improved the performance of both methods and makes the methods less sensitive to the regularization settings.

\begin{figure}
\begin{center}
\includegraphics[width=0.8\columnwidth,clip, viewport=52 265 539 522]{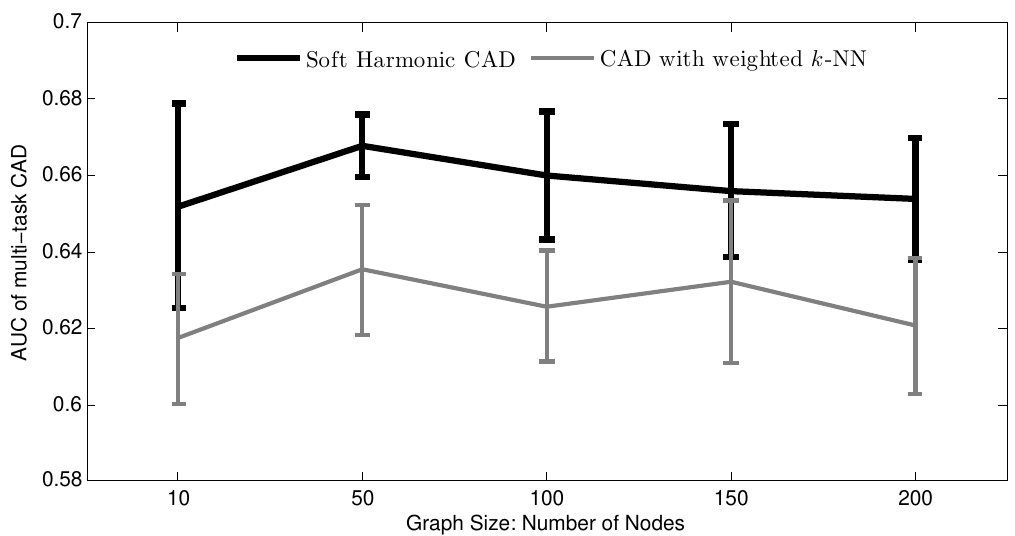}
\caption{Medical Dataset: Varying graph size}%
\label{fig:samplesize}%
\end{center}
\end{figure}

\begin{figure}
\begin{center}
\includegraphics[width=0.8\columnwidth, clip, viewport=39 251 565 530]{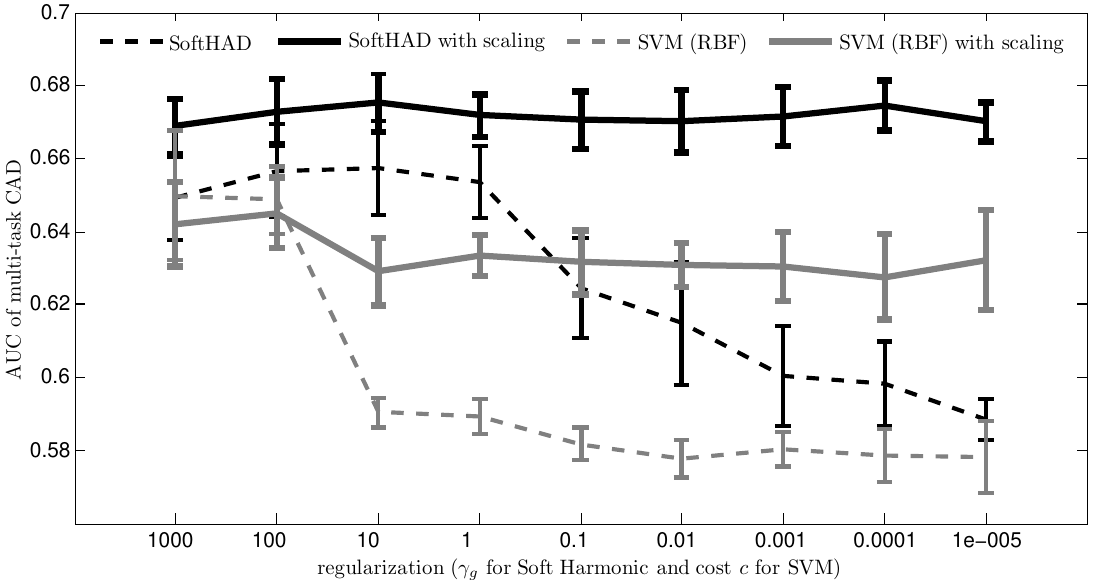}
\caption{Medical Dataset: Varying regularization}
\label{fig:gg}%
\end{center}
\end{figure}

\subsubsection{Conclusions}
\label{sec:ExpertsConclusions}

In the evaluations with human experts on the real-world data,
we showed we can indeed learn clinically useful alerts.  
The results reported here support that this is a promising methodology for raising clinically useful alerts.
Moreover, we showed that with label propagation on a data similarity graph 
built from patient records, we can significantly outperform previously proposed SVM-based anomaly detection
in detecting conditional anomalies.

%% file: datasets_cad.tex
\subsection{Synthetic Datasets}
\label{sec:SyntheticDatasets}

We use two synthetic datasets for the evaluation of conditional anomaly detection methods
where we know or can compute the true conditional anomaly score. 

\subsubsection{Core dataset}
\label{sec:CoreDataset}

Inspired by \cite{papadimitriou2003cross-outlier}, we generate a synthetic \emph{Core}
dataset, which consists of two overlapping squares from two uniform distributions.
We extend this dataset with two tiny squares  (Figure \ref{fig:core}, top left).
These 2 tiny squares may be considered anomalous, but not conditionally anomalous.
The goal is to detect 12 conditional anomalies that are located in the middle square
(Figure~\ref{fig:core}, top middle). 
We also use this dataset to demonstrate the challenges for 
conditional anomaly detectors, namely fringe and isolated points.

\subsubsection{Mixtures of gaussians}
\label{sec:MixturesOfGaussians}

We generated three synthetic datasets (D1, D2, and D3) with known underlying distributions
that let us compute the true anomaly scores.

\begin{figure}%
\centering
\includegraphics[width=0.8\columnwidth, clip, viewport=28 291 591 506]{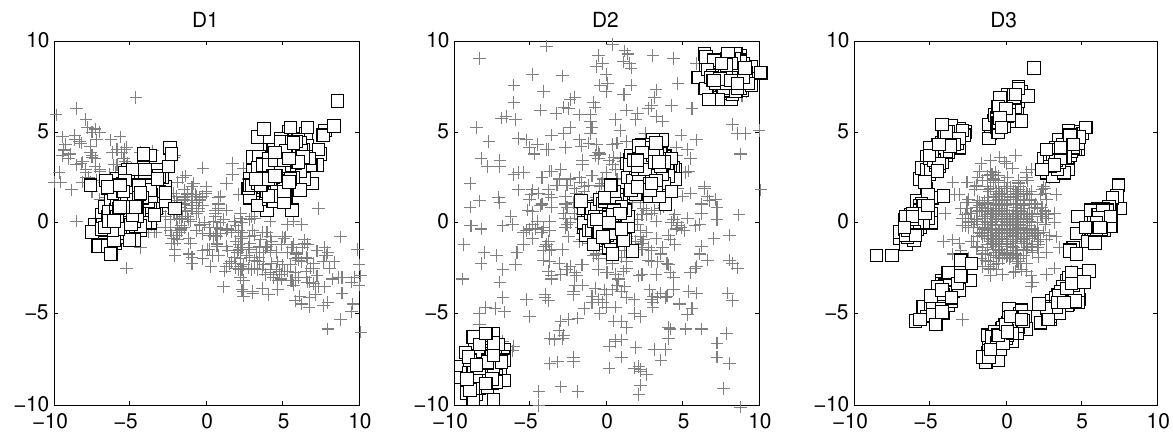} 
\caption{The three synthetic datasets with known underlying distributions}%
\label{fig:kdd_synthetic_datasets}%
\end{figure}

We show the three datasets we used in our experiments in
Figure~\ref{fig:kdd_synthetic_datasets}. Each dataset consists of an equal number 
of samples from the class $+1$ and class $-1$. The class
densities we use to generate these datasets 
are modeled with mixtures of multivariate Gaussians and 
vary in locations, shapes, and mutual overlaps. 

\subsection{Post-surgical cardiac patients (PCP)}
\label{sec:MARS}

For the evaluation of our conditional anomaly detection methods on the real world medical data,
we use the post-surgical cardiac patients (PCP) dataset. PCP is a database of de-identified records for 4486 post-surgical cardiac patients treated at one of the University of Pittsburgh Medical Center (UPMC) teaching hospitals. The entries in the database were populated from data from the 
MARS\footnote{MARS stands for Medical Archival System, and it is a medical record system 
that has been storing clinical and financial information from UMPC since 1980.}
system, which serves as an archive for much of the data collected at UPMC. The records for individual patients include discharge records, demographics, progress notes, all labs and tests (including standard and all special tests), two medication databases, microbiology labs, EKG, radiology and special procedures reports, and a financial charges database. The data in the PCP database were cleaned, cross-mapped, and  stored in a local MySQL database with protected access. The cohort of the patient data we
use in this dissertation consists of 4486 patients that underwent cardiac surgery 
from 2002 to 2007.
The database is very heterogeneous  and has many variables in different formats. 
It has also a fair amount of missing data. 

\begin{figure}
\begin{center}
\includegraphics[width=0.6\columnwidth, clip, viewport=18 32 709 409]{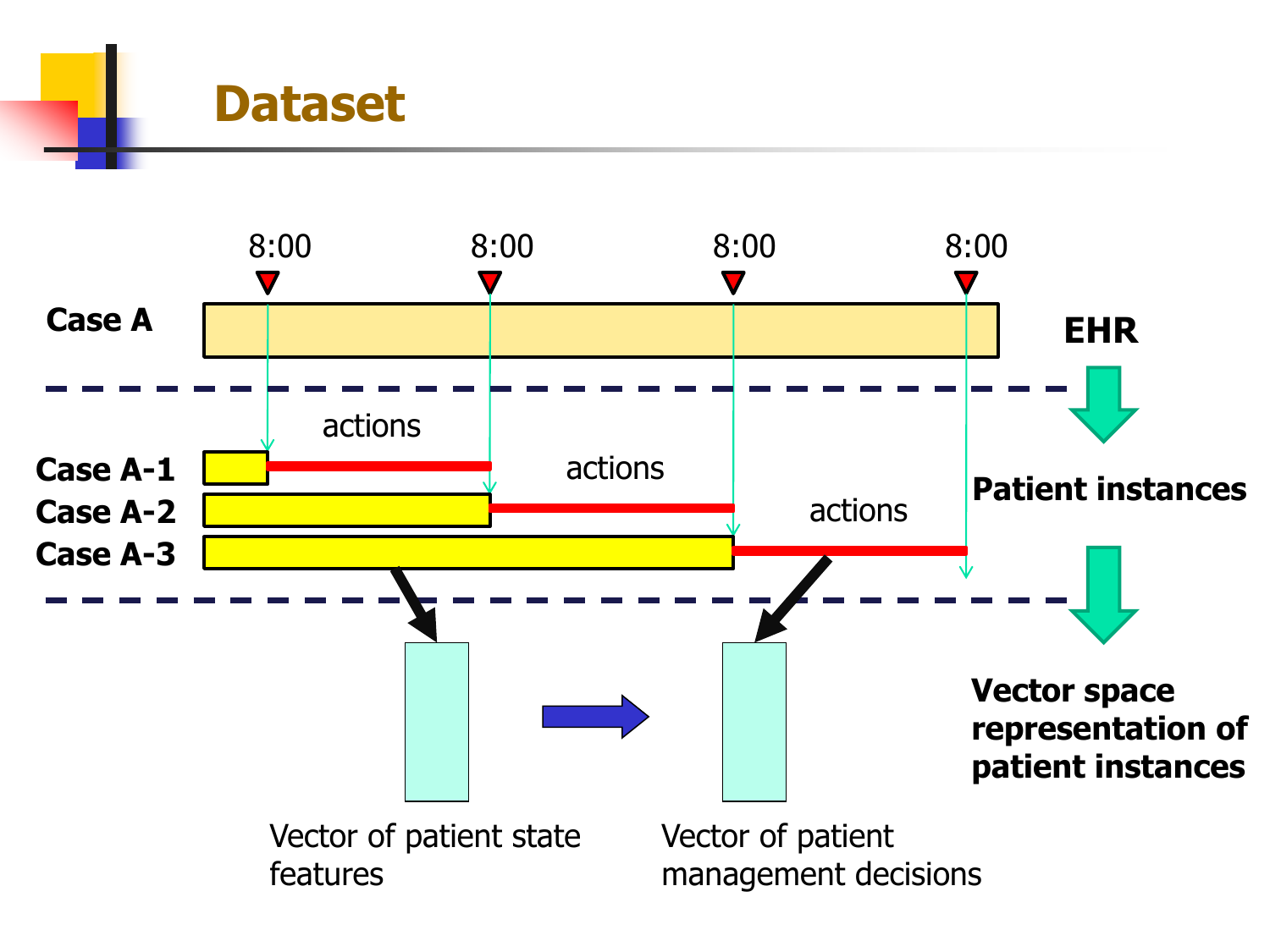}
\caption{Processing of data in the electronic health record}
\label{fig:pcp_segmentation}%
\end{center}
\end{figure}

The EHRs were first divided into two groups:  a training set that included 2646 patients, and a 
test set that included 1840 patients.  We use the time-stamped 
data in each EHR to segment the record at 8:00am every day to obtain multiple 
patient case instances, as illustrated in Figure \ref{fig:pcp_segmentation}:
1) segmentation of an EHR into multiple patient-state/decision instances, and 2) transformation of these instances into a vector space representation of patient states and their follow-up decisions. 
The segmentation led to 51,492 patient-state instances, 
such that 30,828 were used for training the model, and 20,664 were used in the evaluation.  


To represent a patient state we adopt a vector space representation that is suitable for machine learning approaches.  In this representation a patient state is represented by a set of features characterizing the patient at a specific point in time and their corresponding feature values. Features represent and summarize the information in the medical record such as last blood glucose measurement, last glucose trend, or the time the patient is on heparin.  

The features used in our experiment were generated from a time series associated with different clinical variables, such as blood glucose measurement, platelet measurement, and Amiodarone medication.  The clinical variables used in this study were from the following five sources:

\begin{enumerate}
\item 	Laboratory tests (LABs)
\item 	Medications (MEDs)
\item 	Visit features/demographics
\item 	Procedures 
\item 	Heart support devices			
\end{enumerate}

\noindent Altogether, our dataset consists of 9,223 different features.
We now briefly describe the features generated for clinical variables in each of these categories.

\subsubsection{Visit/Demographic Features}
\label{sec:Demographics}

We only have 3 features in this category: age, sex and race. These are static and the same for every time point we generate.

\subsubsection{Lab features}
\label{sec:Labs}

For the categorical labs, for example the ones with POS/NEG results, we use the following features:  last value, second to last value, first value, time since the last order, is the order pending, is the value known, and is the trend known. For the labs with continuous or ordinal values we use a richer set of features, including features as difference between the last two values, the slope of last 2 values, and their percentage drop/increase. We use the same kind of features for the following pairs of lab values: (last value, first value), (last value, nadir value), and (last value, horizon value). Nadir and horizon values are the lab values with the smallest and the greatest value recorded up to that point. Figure~\ref{fig:pcp_features} illustrates a subset of features generated for labs with continuous values. The total number of features generated for such a lab is 28. 
Some of the features here that can be derived from Figure~\ref{fig:pcp_features} are:

\begin{figure}
\begin{center}
\includegraphics[width=0.66\columnwidth, viewport=220 230 414 329]{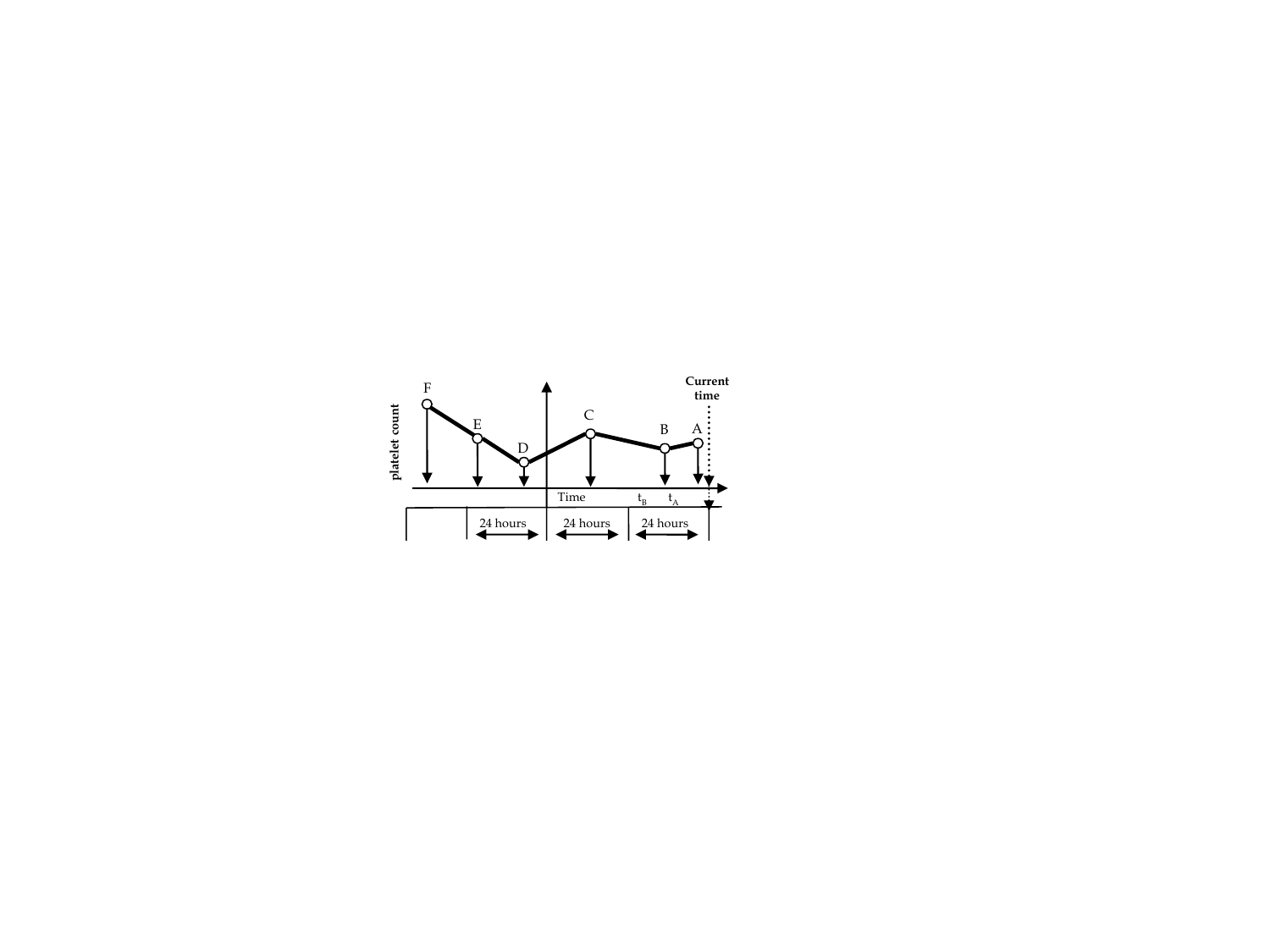}
\caption{Examples of temporal features for continuous lab values}
\label{fig:pcp_features}%
\end{center}
\end{figure}

\begin{itemize}	
\item Last value: A
\item Last value difference = B-A
\item Last percentage change = (B-A)/B
\item Last slope = (B-A) / (tB-tA)
\item Nadir = D
\item Nadir difference = A-D
\item Nadir percentage difference = (A-D)/D
\item Baseline = F
\item Drop from baseline = F-A
\item Percentage drop from baseline = (F-A)/F
\item 24 hour average = (A+B)/2
\end{itemize}

\eat{
 `acc` - this is the ascension number
 `rdate` and `rtime` - it is the date and time the specimen was received in the labs
 `cdate` and `ctime` - Is it date and time when the lab was collected?
 `range` - the normal values for this test
 `flag` - indicates whether the test is out of range, unless it was expected it to be.
 `status` - if this is a pending or final result, some test can give tentative results earlier 
}
\subsubsection{Medication features}
\label{sec:Medications}

For each medication we used four features: 1) an indicator if the patient is currently on the medication, 2) the time since the patient was first put on that medication, 3) the time since the patient was last on that medication, and 4) the time since last change in the order of the medication.

\subsubsection{Procedure features}
\label{sec:Procedures}

The procedure features capture the information about procedures, such as \textit{Heart valve repair}, that were performed either in operating room (OR) or at the bedside. In our data we distinguish 36 different procedures that are performed on cardiac patients. We record four features per procedure: 1) the time since the procedure was done the last time 2) the time since the procedure was done the first time 3) an indicator of whether the procedure was done in the last 24 hours and 4) an indicator of if the procedure was done. 

\subsubsection{Heart support device features}
\label{sec:HeartSupportDeviceFeatures}

Finally, we describe the status of 4 different heart support devices: an extra-corporeal membrane oxygenation (ECMO), a balloon counter pulsation, a pacemaker, and other heart assist devices. For each of them we record a single feature which describes whether the device is currently used to support the patient's heart function. 
  
\subsubsection{Orders/labels}
\label{sec:OrdersLabels}

Labels in this case correspond to patient-management decisions. In addition to feature generation, every patient-state example in the dataset that was generated by the above segmentation process was linked to lab order decisions and medication decisions that were made for the patient within next 24 hours. Patient management decisions considered were:

\begin{itemize}
	\item lab order decisions with (true/false) values reflecting whether the lab was ordered within the next 24 hours or not
	\item medication decisions with (true/false) values reflecting if the patient was given a medication within the next 24 hours or not. 	
\end{itemize}

\noindent A total of 335 lab order decisions and 407 medication decisions were recorded and linked to every patient-state example in the dataset.

%

%% file: kdd_syn_table.tex
\begin{table}[htbp]
  \centering
      \begin{tabular}{rrrr}
    \addlinespace
    \toprule
        &  Dataset \textbf{D1}   & Dataset \textbf{D2}   &  Dataset \textbf{ D3} \\
    \midrule           
    \emph{SVM RBF}  & 58.4\% (7.4) & 49.3\% (2.1) & 51.7\% (1.9) \\
    \emph{1cSVM RBF}   & 51.5\% (0.8) & 47.4\% (0.6) & 59.1\% (0.6) \\
    \emph{SoftHAD}      & \textbf{82.8\%} (1.3) & 63.9\% (2.3) & \textbf{63.5\%} (3.3) \\
    \emph{weighted $k$-NN}        & 64.3\% (2.2) & 45.6\% (1.6) & 62.5\% (1.5) \\
    \emph{$\lambda$-RWCAD} & 64.7\% (0.8) & \textbf{68.9\%} (1.1) & \textbf{67.4\%} (1.9) \\                   
    \bottomrule
    \end{tabular}%
		\caption{Mean anomaly AUROC and variance on three synthetic datasets} 		
		\label{tab:syn}%
\end{table}%

%% file: chap_future.tex
\chapter{Discussion}
\label{chap:FutureWork}

We have presented several algorithms for semi-supervised learning and conditional anomaly detection.
The algorithms are based on label propagation on a similarity graph built from 
examples in a dataset. Label propagation on graphs is polynomial, but still a computationally expensive
method. Therefore, we focused on the approximation approaches for the cases with large datasets
and when the data arrive in a stream. The main contributions of this dissertation to the field of machine learning 
are summarized below.

\begin{itemize}
	\item We presented one of the first works on \emph{online semi-supervised learning}. Despite a very natural scenario, this 
	setting has not been extensively studied in the past. To our best knowledge this is the first work on online semi-supervised 
	learning that comes with theoretical guarantees. Moreover, we built a real-time system that works on noisy real-world data.
	
	\item We introduced a label propagation method for conditional anomaly detection and applied it 
	to compute the anomaly score for the class labels. We presented a general framework where the 
	discriminative models need to regularized to decrease the effect caused by isolated and fringe points in the data. 		
		
	\item We presented a new semi-supervised learning algorithm based on max-margin graph cuts, which in 
		some classes of learning functions can perform better than the manifold regularization approach.
				
	\item We introduced a joint learning of the backbone graph and the label propagation and 
		show its relationship to the elastic nets. This is one 
	  of the first works, besides~\cite{zhu2005harmonic}, that relates propagated labels and 
		cluster centers. 
	
\end{itemize}

\bigskip

\noindent We also made contributions to the area of health informatics:

\begin{itemize}
	\item The existing error detection systems deployed in hospitals are built entirely by human experts. Although these systems are time-consuming and costly to build, they typically do not cover all the specialties of medical care. The statistical anomaly detection approach for error detection, proposed and studied in this work, relies solely on data that are extracted from existing patient record repositories and little or no expert input is required.  This reduces the cost of the approach and its deployment.
	Our most important finding is that the alert systems 
	can be learned from the past patient data instead of
	creating rule-based alert systems that require expensive human time to tune 
	and are currently used in hospitals.				
	
	\item We proposed a non-parametric method that can discover anomalies in clinical actions.
	The common use cases are: 1) discovery of an omitted order of a lab test 
	2) commission of a drug that has interactions with previously taken drugs
	3) controlling overspending: a detection of expensive actions that were not necessary, when 
	the resources could have been used better.		
	
	\item We conducted an extensive study with the human evaluation of the alerts 
	on the real patient records, showing  that the higher anomaly scores corresponded to the higher severity of the alerts. 
	
\end{itemize}

\noindent There are, however, some assumptions and limitation of our methods:

\begin{itemize}
	\item We assume that the data can be modeled with pair-wise similarities between the nodes and that such a model is meaningful.
	\item The similarity function between the graph nodes needs to be given or learned. 
	\item Our methods are expected to perform well when the manifold assumption holds.
	\item In the approximation settings, when we create a summary graph (both online and in a large scale setting),
	we assume that we can model the data well with a reduced number of nodes. 	
\end{itemize}

\noindent  Moreover, electronic health records (EHRs) are a necessary requirement for the successful deployment of the 
conditional anomaly methods described here. With an increasing number of medical groups adopting EHR systems \cite{gans2005medical}, more people will benefit from reduced medical errors. We imagine that the inclusion of our method into the existing EHR systems will 
require no extra time from physicians. Our anomaly detection framework serves more as background monitoring system
that raises alerts only when the confidence of an anomaly is high. Since our conditional 
anomaly methods produce a soft score reflecting the confidence, the threshold for alerting could be adjusted.
Nevertheless, the statistical anomalies that our methods produce may not need to always
correspond to useful alerts. For instance, an omission of a routine lab test or administration of a vitamin may 
be a significant statistical anomaly, but might not be worthy of physician's attention.

\noindent We now outline some related open questions and research opportunities. 

\begin{itemize}
	\item \emph{Structured Anomaly Detection}  

	In this dissertation we applied our conditional anomaly detection method to discover unusual clinical actions.
	Although, we did it separately for each action, these actions are not independent. 
	For example, a clinician usually prescribes a set of drugs such that:

	\begin{itemize}
		\item drugs with the same effect do not tend to be given at the same time.
		\item drugs with the opposite effect do not tend to be given at the same time.
		\item drugs with negative interactions do not tend to be given at the same time.
\end{itemize}

Therefore, we can form groups of drugs from which at most one is administered at the same time.
This additional information could be given a priori or learned from the data. 

	\item \emph{Graph Parametrization}: Despite the research in this area, the graph construction is still 
	not well understood. There are some rules of thumb, such as $\log (n)$ for the number of neighbors, 
	but a problem-specific calibration is usually needed. In particular, the clinical data
	could benefit from the similarity measures (kernels) that would measure the similarity of the 
	conditions from the electronic health data. 
	
	\item \emph{Multi-manifold Learning} 
	In our multi-manifold learning approach, we decomposed the graph and kept updating each of the components 
	independently in parallel. There can be some benefit in accuracy if we allow the components to exchange some  information.
	
	\item \emph{Concept Drift} In this dissertation we were concerned with adapting to the distribution in a  short-term period.
	The problem of concept drift is concerned with long-term changes, such as when a face of a person changes as he or she grows
	older or when medical practices change. One possible extension of our methods can be the online graph-based
	learning with forgetting the history. For example, we can delete the graph nodes which were
	added a long time ago and do not change the current prediction much if they are removed.
	
\end{itemize}

\noindent We expect that future research will address these questions.  
We hope that online semi-supervised learning will become more studied and used to address machine learning problems. 
We believe that our conditional anomaly detection methods will prevent some of the adverse outcomes, especially in medicine.

\bigskip

\noindent This work was supported by the NIH grants R21 LM009102-01A1, 
R01 1R01LM010019-01A1 and by the Mellon Foundation.